\documentclass[11pt, a4paper,leqno]{amsart}
\usepackage{amsmath,amsthm,amscd,amssymb,amsfonts, amsbsy}
\usepackage{latexsym}
\usepackage{txfonts}
\usepackage{exscale}
\usepackage{enumitem}
\setlist[itemize]{leftmargin=2em}
\usepackage{mathrsfs}
\usepackage{graphicx}

\usepackage{tikz-cd}

\usepackage{hyperref}
\newcommand{\doi}[1]{\href{https://doi.org/#1}{doi: #1}}

\hypersetup{
  colorlinks=true,
  linkcolor=blue,
  citecolor=blue,
  urlcolor=blue,
  pdftitle={A Commutator Framework for Selective Spectral Alignment in Deep Neural Networks},
  pdfauthor={Kaj Nystr\"om},
  pdfsubject={Finite-width spectral alignment in nonlinear neural networks},
  pdfkeywords={deep neural networks, feature learning, spectral alignment, matrix commutators, AGOP, NFM}
}

\usepackage{pgf}

\calclayout
\allowdisplaybreaks

\newtheorem{theorem}{Theorem}[section]
\newtheorem{proposition}{Proposition}[section]
\newtheorem{lemma}{Lemma}[section]
\newtheorem{corollary}{Corollary}[section]
\theoremstyle{definition}

\newtheorem{remark}{Remark}[section]

\numberwithin{equation}{section}
\def\barint{\,\Xint -} 
\def\bariint{\barint_{} \kern-.4em \barint}
\def\bariiint{\bariint_{} \kern-.4em \barint}

\newcommand{\beq}{\begin{equation}}
\newcommand{\bea}[1]{\begin{array}{#1} }
\newcommand{\eeq}{ \end{equation}}
\newcommand{\ea}{ \end{array}}

\def \d {{\delta}}

\def\mean#1{\mathchoice%
          {\mathop{\kern 0.2em\vrule width 0.6em height 0.69678ex depth -0.58065ex
                  \kern -0.8em \intop}\nolimits_{\kern -0.4em#1}}%
          {\mathop{\kern 0.1em\vrule width 0.5em height 0.69678ex depth -0.60387ex
                  \kern -0.6em \intop}\nolimits_{#1}}%
          {\mathop{\kern 0.1em\vrule width 0.5em height 0.69678ex
              depth -0.60387ex
                  \kern -0.6em \intop}\nolimits_{#1}}%
          {\mathop{\kern 0.1em\vrule width 0.5em height 0.69678ex depth -0.60387ex
                  \kern -0.6em \intop}\nolimits_{#1}}}

\def\vintslides_#1{\mathchoice%
          {\mathop{\kern 0.1em\vrule width 0.5em height 0.697ex depth -0.581ex
                  \kern -0.6em \intop}\nolimits_{\kern -0.4em#1}}%
          {\mathop{\kern 0.1em\vrule width 0.3em height 0.697ex depth -0.604ex
                  \kern -0.4em \intop}\nolimits_{#1}}%
          {\mathop{\kern 0.1em\vrule width 0.3em height 0.697ex depth -0.604ex
                  \kern -0.4em \intop}\nolimits_{#1}}%
          {\mathop{\kern 0.1em\vrule width 0.3em height 0.697ex depth -0.604ex
                  \kern -0.4em \intop}\nolimits_{#1}}}

\newcommand{\aveint}[2]{\mathchoice%
          {\mathop{\kern 0.2em\vrule width 0.6em height 0.69678ex depth -0.58065ex
                  \kern -0.8em \intop}\nolimits_{\kern -0.45em#1}^{#2}}%
          {\mathop{\kern 0.1em\vrule width 0.5em height 0.69678ex depth -0.60387ex
                  \kern -0.6em \intop}\nolimits_{#1}^{#2}}%
          {\mathop{\kern 0.1em\vrule width 0.5em height 0.69678ex depth -0.60387ex
                  \kern -0.6em \intop}\nolimits_{#1}^{#2}}%
          {\mathop{\kern 0.1em\vrule width 0.5em height 0.69678ex depth -0.60387ex
                  \kern -0.6em \intop}\nolimits_{#1}^{#2}}}

\def\eqn#1$$#2$${\begin{equation}\label#1#2\end{equation}}
\def\charfn_#1{{\raise1.2pt\hbox{$\chi
_{\kern-1pt\lower3pt\hbox{{$\scriptstyle#1$}}}$}}}

\def\qq1{q_*}
\def\q2{q_{**}}

\newdimen\vintbar
\def\vint{-\kern-\vintbar\int}

\def\0{\boldsymbol 0}

\newtoks\by
\newtoks\paper
\newtoks\book
\newtoks\jour
\newtoks\yr
\newtoks\pages
\newtoks\vol
\newtoks\publ

\def\name[#1, #2]{#1 #2}
\def\ota{{\hbox{\bf ???}}}
\def\cLear{\by=\ota\paper=\ota\book=\ota\jour=\ota\yr=\ota
\pages=\ota\vol=\ota\publ=\ota}
\def\endpaper{\the\by, \textit{\the\paper},
{\the\jour} \textbf{\the\vol} (\the\yr), \the\pages.\cLear}
\def\endbook{\the\by, \textit{\the\book},
\the\publ, \the\yr.\cLear}
\def\endpap{\the\by, \textit{\the\paper}, \the\jour.\cLear}
\def\endproc{\the\by, \textit{\the\paper}, \the\book, \the\publ,
\the\yr, \the\pages.\cLear}

\renewcommand{\d}{\, \mathrm{d}} 

\begin{document}

\title[Selective Spectral Alignment in Deep Neural Networks]
{A Commutator Framework for Selective Spectral\\ Alignment in Deep Neural Networks}
\author{Kaj Nystr\"om}
\address{Department of Mathematics, Uppsala University, Box 480\\
SE-751 06 Uppsala, Sweden}
\email{kaj.nystrom@math.uu.se}
\date{\today}

\begin{abstract}
We develop a finite-width geometric framework describing how learned feature
geometries are organized, transported, and selectively aligned in deep neural
networks. Incompatibility among weight-generated covariance, gates, and backward
sensitivities is quantified through three families of commutators: between
gates and covariance, between sensitivities and covariance, and between
average gradient outer products (AGOPs) and neural feature matrices (NFMs).

An exact layerwise identity decomposes the sensitivity-covariance commutator
into four sources: downstream transport, adjacent-layer imbalance, pointwise
sensitivity fluctuations, and nonlinear gate-covariance interactions. The
AGOP-NFM commutator is a singular-value-weighted transport of the internal
commutator, explaining why observed feature-side alignment alone does not
determine the internal geometry from which it emerges.

Buffered localized energies resolve mixing between separated covariance
subspaces. We establish spectral-gap, projector-evolution, and stabilization
estimates, and formulate conditional Lyapunov principles that yield decay
under explicit geometric error-bound or intrinsic-damping assumptions. These
criteria do not follow from gradient flow alone and clarify why risk reduction
need not imply commutator collapse.

Analytic examples and numerical experiments exhibit factorization of spectral
and activation geometry, transient growth, and cancellation among nonzero sources. In
tested finite-time regimes, cancellation dominated by a negative
transport-imbalance interaction persists across depths, widths, and two
regression benchmarks. Spectral alignment therefore appears as a layer- and
scale-dependent compatibility phenomenon governed by transport, interaction,
cancellation, and possible damping, rather than a universal consequence of
training.
\end{abstract}

\maketitle



    \setcounter{equation}{0} \setcounter{theorem}{0}

    \section{Introduction}

Understanding how neural networks learn statistical features from data remains
a central problem in supervised learning. A growing body of work suggests that
training produces geometric and spectral structure linking backward
sensitivity to the forward representation. Average Gradient Outer Products
(AGOPs) encode the directions to which the network output is sensitive,
whereas Neural Feature Matrices (NFMs) encode the geometry generated by the
learned weights. Their observed correspondence motivates the \emph{Neural
Feature Ansatz} (NFA), which in its strongest form relates a layerwise NFM to a
positive spectral power of the corresponding AGOP
\cite{Radhakrishnan2024}. Section~\ref{sec:related-work} gives a detailed
comparison with this and related approaches.

Deep linear networks provide an exact reference case: under balanced gradient
flow, the NFA holds with a spectral exponent determined by depth. For
nonlinear networks, however, exact power-law relations can fail even when the
target function is represented exactly \cite{TansleyMassartCartis2025}. A
prescribed relation between eigenvalues may therefore be too rigid to capture
the common geometric content. In this paper we instead study the weaker compatibility
condition
\begin{equation}
\label{eq:intro-feature-commutator}
[\mathcal S_\ell(t),\mathrm{NFM}_\ell(t)]=0,
\end{equation}
where \(\mathcal S_\ell\) is the layerwise AGOP and
\(\mathrm{NFM}_\ell=W_\ell^\top W_\ell\). Since both operators are real
symmetric, \eqref{eq:intro-feature-commutator} is equivalent to simultaneous
orthogonal diagonalizability. Every spectral functional relation implies
commutativity, but the converse imposes no relation between the corresponding
eigenvalues.

The purpose of this paper is to identify the internal mechanisms governing
this spectral compatibility in finite-width nonlinear networks. The question
is not whether training universally drives all relevant commutators to zero,
but how incompatibility is generated, transported, localized, and cancelled
across layers, and which additional hypotheses can produce genuine decay. The
theory is developed primarily for bias-free, fully connected, scalar-output
networks trained by population gradient flow. The algebraic identities also
apply to empirical averages, and vector outputs are accommodated by replacing
scalar gradient outer products with Jacobian Gram matrices.

The first observation is that the AGOP-NFM commutator is not a closed
geometric object. Let
\begin{equation}
\label{eq:intro-covariance-operators}
P_\ell=W_\ell W_\ell^\top,
\qquad
\mathrm{NFM}_\ell=W_\ell^\top W_\ell
\end{equation}
be the covariance operators on the internal and feature sides of the same
weight matrix. Introducing the population backward-sensitivity operator
\(Q_\ell\), one has
\begin{equation}
\label{eq:intro-common-transport}
\mathcal S_\ell=W_\ell^\top Q_\ell W_\ell
\end{equation}
and the exact transport identity
\begin{equation}
\label{eq:intro-transport-identity}
[\mathcal S_\ell,\mathrm{NFM}_\ell]
=
W_\ell^\top[Q_\ell,P_\ell]W_\ell.
\end{equation}
Thus the feature-side commutator is a congruence image of the internal
sensitivity-covariance commutator. Small internal incompatibility implies
feature-side alignment, subject to the scale of \(W_\ell\), whereas the
converse may fail because transport suppresses components in small singular
directions and annihilates covariance-null directions.

The second observation is that \([Q_\ell,P_\ell]\) satisfies an exact
backward recursion. The nonlinear activation geometry is represented by the
gate
\begin{equation}
\label{eq:intro-gates}
D_\ell(x,t)
=
\operatorname{diag}\bigl(\sigma'(z_\ell(x,t))\bigr),
\end{equation}
whose compatibility with the covariance geometry is measured by
\([D_\ell(x,t),P_\ell(t)]\). At each hidden layer the recursion has the
schematic form
\begin{equation}
\label{eq:intro-four-source-recursion}
[Q_\ell,P_\ell]
=
T_\ell^{\mathrm{tr}}
+T_\ell^{\mathrm{imb}}
+T_\ell^{\mathrm{fluc}}
+T_\ell^{\mathrm{gate}}.
\end{equation}
The four terms arise, respectively, from downstream transport,
adjacent-layer imbalance, pointwise sensitivity fluctuation, and the nonlinear
gate-covariance interaction. Here

\[
\Delta_\ell=P_\ell-W_{\ell+1}^\top W_{\ell+1},
\qquad
R_{\ell+1}=S_{\ell+1}-\mathcal S_{\ell+1}.
\]
They may reinforce or cancel one another. The gate commutator is therefore an intrinsic nonlinear
source, but it is neither the only source of internal incompatibility nor a
complete explanation of a small total commutator.

Equations~\eqref{eq:intro-transport-identity} and
\eqref{eq:intro-four-source-recursion} organize the three alignment levels
into the hierarchy
\begin{equation}
\label{eq:intro-commutator-hierarchy}
[D_\ell,P_\ell]
\quad\longrightarrow\quad
[Q_\ell,P_\ell]
\quad\longrightarrow\quad
[\mathcal S_\ell,\mathrm{NFM}_\ell].
\end{equation}
The first commutator records pointwise incompatibility between activation and
covariance geometry and enters as a primitive nonlinear source; the second
records the internal sensitivity mixing produced by all four mechanisms after
backward transport; and the third is its singular-value-weighted feature-side
image. The arrows indicate propagation and transport, not strict causality,
monotone decay, or a universal ordering of the corresponding energies.

Global commutator norms aggregate interactions across the entire covariance
spectrum. They do not reveal where in the spectrum the misalignment is
concentrated, nor do they distinguish coupling across well-separated spectral
blocks from mixing\footnote{Here and below, mixing between orthogonal spectral
subspaces \(E\) and \(F\) means operator-theoretic coupling: the corresponding
off-diagonal block \(\Pi_EA\Pi_F\) of the relevant operator \(A\) is nonzero.}
inside a nearly degenerate cluster. Buffered localized energies provide this
information at a prescribed blockwise spectral resolution by isolating the
direct coupling between selected covariance subspaces.

To make this distinction precise, write $W_\ell:\mathbb R^{m_\ell}\longrightarrow\mathbb R^{m_{\ell+1}}$ and use the spectral decomposition
\[
P_\ell
=
\sum_{i=1}^{m_{\ell+1}}
\lambda_{\ell,i}u_{\ell,i}u_{\ell,i}^\top,
\qquad
\lambda_{\ell,1}\geq\cdots\geq
\lambda_{\ell,m_{\ell+1}}\geq0.
\]
For \(1\leq q\leq m_{\ell+1}\), let
\[
\Pi_\ell^{(q)}
=
\sum_{i=1}^q u_{\ell,i}u_{\ell,i}^\top,
\qquad
\Pi_\ell^{(q),\perp}
=
I_{m_{\ell+1}}-\Pi_\ell^{(q)},
\]
and define the boundary gap
\[
\gamma_\ell^{(q)}
=
\lambda_{\ell,q}-\lambda_{\ell,q+1},
\qquad
1\leq q<m_{\ell+1}.
\]
When \(\gamma_\ell^{(q)}>0\), the projector
\(\Pi_\ell^{(q)}\) is intrinsically determined by \(P_\ell\), independently
of the choice of eigenbasis inside its spectral subspaces.

Fix \(k\geq1\) and a buffer width \(\beta\geq0\), with
\(k+\beta<m_{\ell+1}\), and assume
\[
\gamma_\ell^{(k)}>0,
\qquad
\gamma_\ell^{(k+\beta)}>0.
\]
The buffered localized energies measure the direct coupling between the
leading covariance subspace
\(\operatorname{Ran}\Pi_\ell^{(k)}\) and the complement of
\(\operatorname{Ran}\Pi_\ell^{(k+\beta)}\), while leaving mixing through the
intermediate modes
\[
u_{\ell,k+1},\ldots,u_{\ell,k+\beta}
\]
unrestricted when \(\beta>0\). When the buffer is chosen to contain a nearly
degenerate boundary cluster, rotations inside that cluster are therefore not
interpreted as misalignment across the separated spectral blocks. The
separation relevant to the retained coupling is the buffered gap
\[
\gamma_\ell^{(k,\beta)}
=
\lambda_{\ell,k}-\lambda_{\ell,k+\beta+1}.
\]

These localized energies are natural defects for selective alignment because
they are precisely the squared norms of the off-diagonal blocks connecting
the two selected spectral regions. Indeed, for a symmetric operator
\(A_\ell\) on \(\mathbb R^{m_{\ell+1}}\),
\[
\left\|
\Pi_\ell^{(k)}A_\ell
\Pi_\ell^{(k+\beta),\perp}
\right\|_F^2
=
\sum_{\substack{1\leq i\leq k\\
j>k+\beta}}
\left|
u_{\ell,i}^\top A_\ell u_{\ell,j}
\right|^2.
\]
This quantity vanishes exactly when \(A_\ell\) produces no direct coupling
between the leading \(k\)-dimensional covariance subspace and the remote
complement beyond the buffer. It is invariant under orthogonal changes of
basis within the two spectral subspaces and therefore depends on the
subspaces rather than on the particular eigenvectors used to represent
them. Moreover, the eigenbasis expansion of the global commutator gives
\[
\|[A_\ell,P_\ell]\|_F^2
\geq
2\bigl(\gamma_\ell^{(k,\beta)}\bigr)^2
\left\|
\Pi_\ell^{(k)}A_\ell
\Pi_\ell^{(k+\beta),\perp}
\right\|_F^2.
\]
Applying this estimate pointwise and integrating for
\(A_\ell=D_\ell(x)\), and applying it directly for \(A_\ell=Q_\ell\), shows
that the global gate and internal commutator energies dominate their
corresponding localized energies by the common factor
\(2(\gamma_\ell^{(k,\beta)})^2\). The same conclusion holds at the feature
level whenever the cutoff is admissible on both sides of \(W_\ell\).
Vanishing of these localized energies expresses selective alignment across
the buffered interface and is weaker than global commutativity.

To compare the three localized levels, we use on the feature side the right
singular vectors of \(W_\ell\), completed through \(\ker W_\ell\), so that
their nonzero spectral indices agree with those of \(P_\ell\). Let $r_\ell=\operatorname{rank}(W_\ell)$. A common cutoff is admissible when
\[
1\leq k\leq r_\ell,
\qquad
k\leq k+\beta\leq m_\ell,
\qquad
k+\beta<m_{\ell+1}.
\]
If \(\mathcal G_{\mathcal S,\ell}^{(k,\beta)}\) denotes the corresponding
feature-side localized energy, the transport identity gives
\begin{equation}
\mathcal G_{\mathcal S,\ell}^{(k,\beta)}
=
\sum_{\substack{1\leq i\leq k\\
k+\beta<j\leq r_\ell}}
\lambda_{\ell,i}\lambda_{\ell,j}
\left|
u_{\ell,i}^\top Q_\ell u_{\ell,j}
\right|^2.
\label{eq:intro-feature-localized-weighting}
\end{equation}
Thus feature-side alignment is a spectrally filtered, and potentially
non-faithful, image of internal sensitivity alignment: mixing into
low-variance directions is suppressed, whereas mixing into the null space of
the covariance is invisible.

The preceding localization describes the geometry at a fixed time. Along
training, the dynamical problem has two coupled components. The covariance
subspaces move, so the evolution of their projectors, boundary gaps, and
separated spectral clusters must be controlled. Simultaneously, the
activation operators evolve relative to these moving subspaces. For smooth
activations, the gates can be differentiated along the training trajectory.
For ReLU networks, by contrast, the gates change discontinuously when
preactivations cross zero, and their finite-time variation is controlled by
symmetric differences of the data regions on which the corresponding neurons
are active.

This leads to a central distinction between \emph{stabilization} and
\emph{alignment}. Under the appropriate gap and regularity assumptions,
finite variation of the covariance and activation geometries can force a
localized energy to approach a limit, but it neither identifies that limit
nor forces it to vanish. Genuine decay requires an additional compatibility
mechanism, a geometric error bound coupling the defect to risk dissipation, or
intrinsic transverse damping. Likewise, localized gate alignment propagates
through \eqref{eq:intro-commutator-hierarchy} only when the imbalance,
sensitivity fluctuations, operator scales, and relevant spectral gaps are
also controlled. The Lyapunov theory below distinguishes risk-based transfer
under a geometric error bound from decay generated directly by intrinsic
damping; neither mechanism follows from gradient flow alone.

The numerical experiments support the need for these qualifications. At the
matched finite-time risk levels attained by the tested synthetic ensembles, a
reduction in relative risk by approximately three orders of magnitude is not
accompanied by decay of the normalized global gate commutator, and the leading
localized sensitivity defect may increase before stabilizing. Across the
tested widths, greater width is associated with smaller internal and
feature-side defects but not with a smaller gate defect. At interior layers
of the deeper networks, the Frobenius interaction between the transport and
imbalance sources is strongly negative, leaving the total internal
commutator much smaller than the combined magnitudes of its sources. This
source-level cancellation persists across the tested depths and widths and
on two scalar-regression benchmarks. A function-preserving rescaling of
adjacent layers also changes the source decomposition without changing the
initial function represented by the network. These finite-time observations
support selective, mechanism-dependent alignment rather than universal
commutator decay driven by risk minimization.

\subsection{Our principal contributions}

The principal contributions of the paper are the following.
\begin{enumerate}[labelindent=0.75em,leftmargin=*]

\item
We derive the exact transport identity
\eqref{eq:intro-transport-identity} and the recursive four-source decomposition
\eqref{eq:intro-four-source-recursion} of the internal
sensitivity-covariance commutator. Together with the associated energy
identities, these formulas make downstream transport, adjacent-layer
imbalance, pointwise sensitivity fluctuations, nonlinear gate-covariance
interactions, and the resulting reinforcement and cancellation explicit at
finite width.

\item
We develop a common localization theory relative to the evolving spectrum of
\(P_\ell\). Buffered energies isolate direct coupling between separated
spectral blocks while leaving mixing through intermediate modes unrestricted.
Spectral-gap bounds relate these localized defects to the global commutator
energies, projector and cluster estimates control the selected covariance
subspaces, and singular-value weights quantify the transport of internal
mixing to feature space.

\item
We treat smooth and ReLU activations separately. For smooth activations, we
derive differential estimates for the localized gate energies. For ReLU
networks, where the gates change discontinuously, we formulate their
finite-time evolution through activation regions and gate-switching estimates
based on symmetric differences of these regions.

\item
We establish conditional stabilization, propagation, and Lyapunov principles.
Under finite variation, gap, and regularity hypotheses, selected localized
energies converge to finite limits. Quantitative decay is obtained under
explicit compatibility, geometric error-bound, or intrinsic damping
assumptions. We also identify the imbalance, fluctuation, gap, operator-scale,
and boundedness conditions needed to propagate alignment through the
commutator hierarchy.

\item
We analyze exactly solvable finite-width models exhibiting
spectral-activation factorization, separation between the three alignment
levels, exact cancellation among nonzero sources, and transient growth along
a constrained population-gradient-flow trajectory. These examples explain
why none of the individual commutator energies is automatically a Lyapunov
function.

\item
We introduce matched-risk numerical diagnostics for global, localized, and
source-level observables. The experiments confirm the exact identities to
numerical precision, distinguish the three alignment levels, reveal persistent
transport-imbalance cancellation, and isolate the parametrization dependence
of the source geometry through a function-preserving rescaling.

\end{enumerate}

Together, these results show why a small AGOP-NFM commutator may reflect
internal damping, singular-value filtering, structured source cancellation,
or a combination of these mechanisms. The hierarchy and its buffered
localization distinguish these possibilities at finite width.

\subsection{Organization of the paper}

The remainder of the paper is organized as follows.
Section~\ref{sec:related-work} places the commutator framework in the context
of the Neural Feature Ansatz, feature-alignment mechanisms, deep-linear
balancedness, symmetry-induced conservation laws, and related weight-Gram
dynamics.
Section~\ref{sec:Pre} introduces the network, sensitivity, covariance,
imbalance, fluctuation, and gradient-flow notation.
Section~\ref{sec:commutatorhierarchy} derives the transport identity, the
four-source internal recursion, its energy and norm estimates, and the
evolution of imbalance and sensitivity fluctuations.
Section~\ref{sec:SpectralGeometry} develops buffered localization, covariance
subspace dynamics, and transport to feature-side geometry.
Section~\ref{sec:ActivationGeometry} studies smooth activation dynamics and
ReLU switching.
Section~\ref{sec:AlignmentDynamics} distinguishes stabilization from decay
and gives conditional propagation results, while
Section~\ref{sec:lyapunov-alignment} develops Lyapunov principles for selected
buffered sensitivity defects.
Section~\ref{Analytic} presents the exactly solvable models and their
cancellation and transient-growth mechanisms.
Section~\ref{Num} develops and applies the finite-width numerical diagnostics.
Finally, Section~\ref{sec:Conclusion} summarizes the framework and briefly
indicates several directions for further work.
    \section{Relation to existing work}
\label{sec:related-work}

The present work is most closely related to the recent literature on the
Average Gradient Outer Product (AGOP) and the Neural Feature Ansatz (NFA).
The NFA was introduced in \cite{Radhakrishnan2024} as a mechanism for feature
learning, motivated by an empirical correspondence between layerwise AGOPs and
weight Gram matrices. Subsequent work connected AGOP-based feature learning to
deep neural collapse \cite{BeagleholeSukenikMondelliBelkin2024} and decomposed
the observed NFM-AGOP correlation through the singular geometry of weight
updates and preactivation tangent features
\cite{BeagleholeMitliagkasAgarwala2024}. Related forms of dynamically emerging
feature geometry include the alignment of neural tangent features with a small
number of task-relevant directions \cite{BaratinEtAl2021} and the Canonical
Representation Hypothesis of \cite{ZiyinChuangGalantiPoggio2025}, which posits
mutual alignment relations among representations, weights, and neuron
gradients. The FACT framework \cite{BoixAdseraMallinarSimonBelkin2026} provides
a complementary convergence-based perspective by deriving feature relations
from first-order optimality conditions.

Our point of departure is that an exact spectral functional relation between
the AGOP and the NFM is stronger than the eigenspace compatibility required
for spectral alignment. We therefore study the weaker condition stated in \eqref{eq:intro-feature-commutator} which retains common spectral subspaces without prescribing a relation between
the corresponding eigenvalues. The purpose is not to propose a competing
feature ansatz, but to determine how this weaker compatibility is generated,
transported, or obstructed by finite-width nonlinear dynamics. The commutator
framework accordingly separates the feature-side observable from the internal
sensitivity, covariance, activation, imbalance, and fluctuation geometries
that produce it.

Deep linear networks provide the principal exact comparison. Under gradient
flow, their adjacent-layer Gram differences are conserved; such balancedness
identities play a central role, for example, in
\cite{AroraCohenGolowichHu2019}. For balanced initialization they also lead to
exact depth-dependent NFA relations. A related implicit-regularization result
was established in \cite{JiTelgarsky2019}: for linearly separable
classification, the weight matrices become asymptotically rank one and their
adjacent singular directions align. The direct connection between
deep-linear balancedness and the NFA was developed in
\cite{TansleyMassartCartis2025}, which obtained the exponent \(1/L\) for an
\(L\)-layer linear network, proved an asymptotic relation for unbalanced
initialization with weight decay, and exhibited nonlinear counterexamples.

For bias-free homogeneous nonlinear networks, the full adjacent-layer matrix
law need not persist, although neuronwise differences between incoming and
outgoing squared norms remain invariant under gradient flow
\cite{DuHuLee2018}. Architectural symmetries and their associated conservation
laws have been studied more generally in
\cite{KuninEtAl2021,MarcotteGribonvalPeyre2023}; this literature also clarifies
their dependence on the continuous-time dynamics and on symmetry-breaking
perturbations. In the present framework, the adjacent-layer matrix imbalance
is not eliminated by assumption. It appears instead as one of the four sources
in the recursive sensitivity-covariance commutator identity. The
function-preserving rescaling experiment correspondingly changes the values of
the conserved neuronwise imbalance quantities while leaving the represented
initial function unchanged.

The closest dynamical comparisons are
\cite{BeagleholeMitliagkasAgarwala2024} and
\cite{ChaBeagleholeRadhakrishnanLee2026}. The former relates the NFM-AGOP
correlation to alignment between weight updates and preactivation tangent
features. The latter derives a Feature Learning Equation in which the weight
Gram matrix links parameter updates to a virtual evolution of hidden features.
Our analysis addresses a different level of this geometry: it derives an exact
layerwise transport identity and decomposes the internal
sensitivity-covariance commutator into downstream transport, adjacent-layer
imbalance, pointwise sensitivity fluctuations, and nonlinear gate-covariance
interactions. These identities hold along finite-width trajectories, are not
restricted to critical points, and require neither balanced initialization nor
commutator decay. They therefore permit transient growth and cancellation
among nonzero sources. The buffered localized energies further distinguish
mixing across separated covariance subspaces from mixing inside nearly
degenerate spectral clusters, leading to a selective rather than universal
notion of spectral alignment.

The present results are finite-width identities and estimates. They are
complementary to feature-learning limits in mean-field and tensor-program
parametrizations~\cite{NguyenPham,YangHu}, which address asymptotic training
dynamics as layer widths diverge.

Xiao, Sui, and Ruan~\cite{XiaoSuiRuan2026} recently proposed a flag-variety
framework for deep-network alignment. Their principal objects are the
relative positions of adjacent weight-generated subspaces, with emphasis on
intersection data invariant under reparametrization, flag geometry, and alignment
induced by ridge regularization; nonlinear activations appear there through a
commutator obstruction to exact basis alignment. The present framework is
complementary. Its basic objects are data-dependent gates, backward
sensitivities, and covariance operators, and its central conclusions are the
exact four-source recursion, cancellation among those sources, and the
AGOP-NFM transport identity. Thus the two approaches use related
commutator language but address different geometries and different dynamical
questions.

Finally, a separate future direction concerns probabilistic gate surrogates
and random-matrix universality. The nonasymptotic universality and sharp
concentration principles of \cite{BrailovskayaVanHandel2024} apply to sums of
independent random matrices. They do not apply directly to the data-dependent,
dynamically coupled gates of a trained network, but they suggest tools for
asking which commutator phenomena persist across probabilistic gate models.
No such universality principle is assumed in the deterministic hierarchy
developed here.

    \section{Preliminaries}\label{sec:Pre}

We consider scalar-valued regression on
\begin{equation}\label{eq14}
\mathcal X\times\mathcal Y=\mathbb R^d\times\mathbb R,
\qquad x\sim\mu,
\qquad y=f_\ast(x),
\end{equation}
where \(\mu\) is a probability measure on \(\mathbb R^d\) and
\(f_\ast:\mathbb R^d\to\mathbb R\) is the target function. Given a loss
\(\mathcal L:\mathbb R\times\mathbb R\to[0,\infty)\), differentiable in its
second argument, the population risk is
\begin{equation}\label{Pre2}
\mathcal R(W)
=
\int_{\mathbb R^d}
\mathcal L\bigl(f_\ast(x),F(x,W)\bigr)\,\d\mu(x).
\end{equation}
The empirical setting is obtained by replacing \(\mu\) with
\(\mu_N=N^{-1}\sum_{i=1}^N\delta_{x_i}\), so that
\begin{equation}\label{Pre:empirical-risk}
\mathcal R_N(W)
:=
\int_{\mathbb R^d}
\mathcal L\bigl(f_\ast(x),F(x,W)\bigr)\,\d\mu_N(x)
=
\frac1N\sum_{i=1}^N
\mathcal L\bigl(f_\ast(x_i),F(x_i,W)\bigr).
\end{equation}
All operator identities below then have empirical analogues with expectations
replaced by empirical averages.
The restriction to scalar-valued outputs simplifies the notation only;
vector-valued outputs are discussed in Remark~\ref{rem:vector-output}.

For a matrix \(A\), we write \(\|A\|_{\mathrm{op}}\) and \(\|A\|_F\) for its
operator and Frobenius norms, respectively, so that
\begin{equation}\label{Pre:norms}
\|A\|_{\mathrm{op}}
\leq
\|A\|_F
\leq
\sqrt{\operatorname{rank}(A)}\,\|A\|_{\mathrm{op}}.
\end{equation}
For square matrices of the same size, \([A,B]=AB-BA\). Unless stated
otherwise, dependence on \(x\), \(W\), and, along gradient flow, \(t\), is
suppressed. Gradients with respect to vector variables are understood as
column vectors.

\subsection{Feedforward networks}

Let \(F(\cdot,W):\mathbb R^d\to\mathbb R\) be a feedforward network of depth
\(L\), without bias parameters, with
\begin{equation}\label{NetworkDefinition}
W=(W_0,\ldots,W_{L-1}),
\qquad
W_\ell\in\mathbb R^{m_{\ell+1}\times m_\ell},
\quad \ell=0,\ldots,L-1,
\end{equation}
where \(m_0=d\) and \(m_L=1\). The omission of biases preserves the
homogeneous operator structure and avoids additional affine terms. Given
\(a_0=x\), define
\begin{align}
z_\ell&=W_\ell a_\ell,
&a_{\ell+1}&=\sigma(z_\ell),
&&\ell=0,\ldots,L-2,
\label{eq15}\\
z_{L-1}&=W_{L-1}a_{L-1},
&a_L&=z_{L-1}=F(x,W),
\label{eq15+}
\end{align}
where \(\sigma:\mathbb R\to\mathbb R\) acts componentwise. We assume that
\(\sigma\) is locally Lipschitz and impose additional regularity only when
the activation geometry is differentiated. ReLU is treated separately from
the smooth-activation theory. We do not include width-dependent
normalizations in \(W_\ell\); fixed layerwise normalizations can be absorbed
into the weights and associated covariance operators without changing the
identities below.

\subsection{Backpropagation and Jacobian transport}

For \(\ell=0,\ldots,L-2\), define the gate and adjoint layer Jacobian by
\begin{equation}\label{Pre4}
D_\ell(x,W)
=
\operatorname{diag}\bigl(\sigma'(z_\ell(x,W))\bigr),
\qquad
J_\ell(x,W)
=
W_\ell^\top D_\ell(x,W).
\end{equation}
Thus \(J_\ell\) is the adjoint of the differential
\(D_\ell W_\ell\) of the map \(a_\ell\mapsto a_{\ell+1}\). At the linear
output layer, set \(J_{L-1}=W_{L-1}^\top\). At points where \(\sigma\) is not
differentiable, we use a fixed measurable selection of \(\sigma'\); for ReLU,
\(\sigma'(s)=\mathbf 1_{\{s>0\}}\). Define the output and loss sensitivities by
\begin{equation}\label{Pre5+}
\widetilde\delta_\ell
=
\nabla_{z_\ell}F,
\qquad
\delta_\ell
=
\nabla_{z_\ell}\mathcal L
=
\partial_F\mathcal L\bigl(f_\ast(x),F(x,W)\bigr)
\widetilde\delta_\ell,
\qquad
\ell=0,\ldots,L-1.
\end{equation}
Since the output layer is linear,
\begin{equation}\label{Pre8}
\widetilde\delta_{L-1}=1,
\qquad
\widetilde\delta_\ell
=
D_\ell W_{\ell+1}^\top\widetilde\delta_{\ell+1},
\qquad
\ell=L-2,\ldots,0.
\end{equation}
Consequently,
\begin{equation}\label{Pre10}
\nabla_{W_\ell}\mathcal L
=
\delta_\ell a_\ell^\top,
\qquad
\nabla_{W_\ell}F
=
\widetilde\delta_\ell a_\ell^\top,
\qquad
\ell=0,\ldots,L-1,
\end{equation}
and repeated application of the chain rule gives
\begin{align}
\nabla_{a_\ell}\mathcal L
&=
\partial_F\mathcal L\bigl(f_\ast(x),F(x,W)\bigr)
J_\ell J_{\ell+1}\cdots J_{L-1},
\label{Pre11}\\
\nabla_{a_\ell}F
&=
J_\ell J_{\ell+1}\cdots J_{L-1},
\qquad
\ell=0,\ldots,L-1.
\label{Pre12}
\end{align}

\subsection{Sensitivity and covariance operators}

For a differentiable scalar function \(H\) of a vector \(v\), the gradient
outer product is \(\nabla_vH\nabla_vH^\top\), and its average over the data
distribution is the corresponding Average Gradient Outer Product (AGOP).
At the activation level, define the pointwise and population operators
\begin{equation}\label{Pre14}
S_\ell
=
\nabla_{a_\ell}F\nabla_{a_\ell}F^\top,
\qquad
\mathcal S_\ell
=
\mathbb E_\mu[S_\ell],
\qquad
\ell=0,\ldots,L.
\end{equation}
Thus \(S_\ell\) is positive semidefinite and has rank at most one, whereas
\(\mathcal S_\ell\) aggregates the pointwise sensitivity directions. Since
\(F=a_L\) and the output layer is linear,
\begin{align}
S_L&=1,\quad S_{L-1}=W_{L-1}^\top W_{L-1},\quad S_\ell
=
W_\ell^\top D_\ell S_{\ell+1}D_\ell W_\ell,\quad\ell=0,\ldots,L-2.
\label{Pre16}
\end{align}

To distinguish pointwise from averaged backward sensitivity, define the
pointwise backward sensitivity operator by
\begin{equation}\label{Pre18}
\mathsf Q_{L-1}:=I_{m_L}=1,
\qquad
\mathsf Q_\ell
:=
D_\ell S_{\ell+1}D_\ell,
\qquad
\ell=0,\ldots,L-2,
\end{equation}
and define
\begin{equation}\label{Pre18+}
Q_\ell
:=
\mathbb E_\mu[\mathsf Q_\ell],
\qquad
\ell=0,\ldots,L-1.
\end{equation}
Thus \(\mathsf Q_\ell\) depends on \(x\), while \(Q_\ell\) is its population
average and has no \(x\)-dependence. Equivalently,
\(S_\ell=W_\ell^\top\mathsf Q_\ell W_\ell\); hence averaging
\eqref{Pre16} yields
\begin{equation}\label{Pre17}
\mathcal S_\ell
=
W_\ell^\top Q_\ell W_\ell,
\qquad
\ell=0,\ldots,L-1.
\end{equation}
The forward covariance operator and Neural Feature Matrix associated with
\(W_\ell\) are
\begin{equation}\label{Pre21}
P_\ell
=
W_\ell W_\ell^\top,
\qquad
\mathrm{NFM}_\ell
=
W_\ell^\top W_\ell,
\qquad
\ell=0,\ldots,L-1.
\end{equation}
Thus \(P_\ell\in\mathbb R^{m_{\ell+1}\times m_{\ell+1}}\) and
\(\mathrm{NFM}_\ell\in\mathbb R^{m_\ell\times m_\ell}\). They have the same
nonzero eigenvalues and describe the outgoing and incoming covariance
geometries of \(W_\ell\), respectively. Together, \eqref{Pre17} and
\eqref{Pre21} exhibit \(\mathcal S_\ell\) and \(\mathrm{NFM}_\ell\) as the
congruence images under \(W_\ell\) of \(Q_\ell\) and
\(I_{m_{\ell+1}}\), respectively; this is the common transport
representation underlying the commutator hierarchy. At the final hidden layer,
\(\mathcal S_{L-1}=S_{L-1}=\mathrm{NFM}_{L-1}\).

The alignment diagnostics studied below occur at three levels:
\begin{equation}\label{Pre:three-levels}
[D_\ell(x),P_\ell],
\qquad
[Q_\ell,P_\ell],
\qquad
[\mathcal S_\ell,\mathrm{NFM}_\ell].
\end{equation}
The first is defined for \(\ell=0,\ldots,L-2\), and the latter two for
\(\ell=0,\ldots,L-1\). They measure, respectively, the compatibility of
\(D_\ell\) with \(P_\ell\), of \(Q_\ell\) with \(P_\ell\), and of
\(\mathcal S_\ell\) with \(\mathrm{NFM}_\ell\).
No decay of these commutators is assumed: training may create, reduce,
transport, or cancel the corresponding incompatibilities. The purpose of the
subsequent hierarchy is to relate the three levels and identify the mechanisms
responsible for their behavior.

The loss-weighted population AGOP satisfies
\begin{equation}\label{Pre:loss-agop}
\mathcal S_\ell^{\mathcal L}(W)
:=
\mathbb E_\mu
\bigl[
\nabla_{a_\ell}\mathcal L\,
\nabla_{a_\ell}\mathcal L^\top
\bigr]
=
\mathbb E_\mu
\bigl[
(\partial_F\mathcal L)^2S_\ell
\bigr].
\end{equation}

\begin{remark}\label{rem:vector-output}
For \(F:\mathbb R^d\to\mathbb R^M\), let
\(\mathcal J_\ell=\partial F/\partial a_\ell\in
\mathbb R^{M\times m_\ell}\) and define
\begin{equation}\label{Pre:vector-agop}
S_\ell
=
\mathcal J_\ell^\top\mathcal J_\ell
=
\sum_{r=1}^M
\nabla_{a_\ell}F_r\nabla_{a_\ell}F_r^\top.
\end{equation}
With \(m_L=M\), \(S_L=I_M\), and
\(\mathsf Q_{L-1}=Q_{L-1}=I_M\), the transport identities and hidden-space
commutator relations retain the same form.
\end{remark}

\subsection{Spectral projectors}

Since \(P_\ell\) is symmetric and positive semidefinite, write
\begin{align}
P_\ell
&=
U_\ell\Lambda_\ell U_\ell^\top,
\qquad
U_\ell
=
[u_{\ell,1},\ldots,u_{\ell,m_{\ell+1}}]
\in O(m_{\ell+1}),
\label{Pre22-}\\
P_\ell
&=
\sum_{i=1}^{m_{\ell+1}}
\lambda_{\ell,i}u_{\ell,i}u_{\ell,i}^\top,
\qquad
\lambda_{\ell,1}\geq\cdots\geq
\lambda_{\ell,m_{\ell+1}}\geq0,
\label{Pre23+}
\end{align}
where
\(\Lambda_\ell
=
\operatorname{diag}(\lambda_{\ell,1},\ldots,\lambda_{\ell,m_{\ell+1}})\).
For \(1\leq k<m_{\ell+1}\), define
\begin{equation}\label{Pre:spectral-projector}
\Pi_\ell^{(k)}
=
\sum_{i=1}^k u_{\ell,i}u_{\ell,i}^\top,
\qquad
\Pi_\ell^{(k),\perp}
=
I_{m_{\ell+1}}-\Pi_\ell^{(k)},
\qquad
\gamma_\ell^{(k)}
=
\lambda_{\ell,k}-\lambda_{\ell,k+1}.
\end{equation}
We also use the endpoint conventions
\(\Pi_\ell^{(0)}=0\) and
\(\Pi_\ell^{(m_{\ell+1})}=I_{m_{\ell+1}}\).
If \(\gamma_\ell^{(k)}>0\), then \(\Pi_\ell^{(k)}\) is uniquely determined
by \(P_\ell\); if \(\gamma_\ell^{(k)}=0\), it depends on the choice of
eigenbasis across the corresponding degenerate spectral cluster. These
projectors and gaps provide the coordinates used below to localize
activation-covariance and sensitivity-covariance incompatibility.

\subsection{Imbalance and fluctuation operators}

For \(\ell=0,\ldots,L-2\), define the adjacent-layer imbalance by
\begin{equation}\label{Pre23apa}
\Delta_\ell
:=
P_\ell-\mathrm{NFM}_{\ell+1}
=
W_\ell W_\ell^\top-W_{\ell+1}^\top W_{\ell+1}.
\end{equation}
Thus \(\Delta_\ell\) measures the difference between the outgoing covariance
of layer \(\ell\) and the incoming covariance of layer \(\ell+1\). In deep
linear networks it is conserved under gradient flow, whereas in nonlinear
networks it enters as a source in the internal commutator recursion. The
commutator \([P_\ell,\mathrm{NFM}_{\ell+1}]\) measures only the spectral
incompatibility of the adjacent covariance geometries; hence
\(\Delta_\ell=0\) implies
\([P_\ell,\mathrm{NFM}_{\ell+1}]=0\), but not conversely.

Define the centered sensitivity fluctuation by
\begin{equation}\label{Pre:fluctuation}
R_\ell(x,W)
:=
S_\ell(x,W)-\mathcal S_\ell(W),
\qquad
\mathbb E_\mu[R_\ell]=0,
\qquad
\ell=0,\ldots,L.
\end{equation}
In particular, \(R_{L-1}=R_L=0\).
The imbalance \(\Delta_\ell\), fluctuation \(R_{\ell+1}\), and the gate
commutator \([D_\ell,P_\ell]\) generate the three local source terms in the
recursive decomposition of \([Q_\ell,P_\ell]\); the fourth term transports
the downstream commutator. This decomposition is developed in
Section~\ref{sec:commutatorhierarchy}.

\subsection{Gradient flow}

We consider the population gradient flow
\begin{equation}\label{Pre24}
\dot W_\ell
=
-\nabla_{W_\ell}\mathcal R(W)
=
-\int_{\mathbb R^d}
\delta_\ell(x,W)a_\ell(x,W)^\top\,\d\mu(x),
\qquad
\ell=0,\ldots,L-1.
\end{equation}

\medskip
\noindent
\textbf{Standing trajectory convention.}
Unless a stronger hypothesis is stated, a training trajectory \(W(t)\) is
locally absolutely continuous and
satisfies the selected gradient-flow equation \eqref{Pre24} for almost every
\(t\). At a ReLU kink we use throughout the fixed selection
\(\sigma'(s)=\mathbf 1_{\{s>0\}}\). Algebraic identities involving
\(D_\ell,Q_\ell,P_\ell\), and the three commutators are pointwise identities
at every time at which these operators are defined. Differential identities
for \(W_\ell,P_\ell,Q_\ell\), or their spectral projectors are interpreted
almost everywhere on intervals on which the relevant operator-valued maps are
absolutely continuous and the stated gaps remain open. We never differentiate
a ReLU gate in time: its global variation is treated by the finite-increment
estimates of Section~\ref{sec:ActivationGeometry}. Whenever continuity of
\(D_\ell(\cdot,t)\) in \(\mathrm L^2(\mu;F)\) is invoked, we explicitly use
the no-interface-mass condition
\(\mu\{x:z_{\ell,r}(x,t)=0\}=0\) at the limiting time. These conventions also
apply to empirical gradient flow, with expectations replaced by finite
averages.

Along such a trajectory, the chain rule for locally absolutely continuous
curves gives, for almost every \(t\),
\begin{equation}\label{Pre26}
\frac{\d}{\d t}\mathcal R(W(t))
=
-\sum_{\ell=0}^{L-1}
\|\nabla_{W_\ell}\mathcal R(W(t))\|_F^2.
\end{equation}
Using \eqref{Pre10}, the dissipation identity can equivalently be written as
\begin{align}
\frac{\d}{\d t}\mathcal R(W(t))
&=
-\sum_{\ell=0}^{L-1}
\int_{\mathbb R^d}\int_{\mathbb R^d}
K_\ell(x,x',W(t))\,\d\mu(x)\,\d\mu(x'),
\label{Pre27}\\
K_\ell(x,x',W)
&:=
\langle a_\ell(x,W),a_\ell(x',W)\rangle
\langle\delta_\ell(x,W),\delta_\ell(x',W)\rangle.
\label{Pre27apa}
\end{align}
Thus gradient-flow dissipation couples forward activation correlations with
backward loss-sensitivity correlations.

\section{The internal commutator hierarchy}\label{sec:commutatorhierarchy}

We now relate the three alignment levels introduced in
\eqref{Pre:three-levels}. The feature-side commutator is the most aggregated
quantity and follows directly from the factorization
\(\mathcal S_\ell=W_\ell^\top Q_\ell W_\ell\). Its internal preimage
\([Q_\ell,P_\ell]\), however, contains the mechanisms that the feature-side
observable alone cannot distinguish. We first make this compression explicit
and then derive the exact recursion satisfied by the internal commutators. No
decay statement is implicit in either identity.

\begin{lemma}\label{lem:commutator-transfer}
Fix \(t\geq0\) and \(\ell\in\{0,\ldots,L-1\}\), and set
\[
C_\ell^{\mathrm{feat}}:=[\mathcal S_\ell,\mathrm{NFM}_\ell],
\qquad
C_\ell^{\mathrm{int}}:=[Q_\ell,P_\ell].
\]
Then
\begin{equation}\label{eq:commutator-congruence}
C_\ell^{\mathrm{feat}}
=W_\ell^\top C_\ell^{\mathrm{int}}W_\ell
\end{equation}
and
\begin{equation}\label{eq:comm-upper}
\|C_\ell^{\mathrm{feat}}\|_F
\leq
\|W_\ell\|_{\mathrm{op}}^2\|C_\ell^{\mathrm{int}}\|_F.
\end{equation}
If \(W_\ell\) has full row rank, and \(\sigma_{\min}(W_\ell)\) denotes
its smallest singular value, then
\begin{equation}\label{eq:comm-two-sided}
\sigma_{\min}(W_\ell)^2\|C_\ell^{\mathrm{int}}\|_F
\leq
\|C_\ell^{\mathrm{feat}}\|_F
\leq
\sigma_{\max}(W_\ell)^2\|C_\ell^{\mathrm{int}}\|_F,
\end{equation}
and hence \(C_\ell^{\mathrm{feat}}=0\) if and only if
\(C_\ell^{\mathrm{int}}=0\).
\end{lemma}

\begin{proof}
Suppressing \(\ell\) and \(t\), the identities
\(\mathcal S=W^\top QW\), \(\mathrm{NFM}=W^\top W\), and \(P=WW^\top\)
give
\[
[\mathcal S,\mathrm{NFM}]
=W^\top(QP-PQ)W
=W^\top[Q,P]W.
\]
The upper bound follows from
\(\|ABC\|_F\leq\|A\|_{\mathrm{op}}\|B\|_F
\|C\|_{\mathrm{op}}\). If \(W\) has full row rank, its Moore-Penrose
inverse satisfies \(WW^\dagger=I\), so
\([Q,P]=(W^\dagger)^\top
[\mathcal S,\mathrm{NFM}]W^\dagger\). Since
\(\|W^\dagger\|_{\mathrm{op}}=\sigma_{\min}(W)^{-1}\), this also gives the
lower bound.
\end{proof}

The transport in \eqref{eq:commutator-congruence} is a congruence, not a
similarity. More precisely, let
\[
W_\ell
=
\sum_{i=1}^{r_\ell}
\sigma_{\ell,i}u_{\ell,i}v_{\ell,i}^\top,
\qquad
\sigma_{\ell,i}>0,
\]
be a thin singular-value decomposition. Since
\(
P_\ell u_{\ell,i}
=
\sigma_{\ell,i}^2u_{\ell,i}
\),
\begin{align}
\|C_\ell^{\mathrm{feat}}\|_F^2
&=
\sum_{i,j=1}^{r_\ell}
\sigma_{\ell,i}^2\sigma_{\ell,j}^2
\left|
u_{\ell,i}^\top
C_\ell^{\mathrm{int}}
u_{\ell,j}
\right|^2=
\sum_{i,j=1}^{r_\ell}
\sigma_{\ell,i}^2\sigma_{\ell,j}^2
\bigl(
\sigma_{\ell,j}^2-\sigma_{\ell,i}^2
\bigr)^2
\left|
u_{\ell,i}^\top Q_\ell u_{\ell,j}
\right|^2.
\label{eq:comm-svd}
\end{align}
The second identity displays two distinct spectral filters. The gap factor
\(
|\sigma_{\ell,j}^2-\sigma_{\ell,i}^2|
\)
is already present in the internal commutator
\(
C_\ell^{\mathrm{int}}=[Q_\ell,P_\ell]
\),
whereas transport to feature space introduces the additional factor
\(\sigma_{\ell,i}\sigma_{\ell,j}\).

Consequently, \(C_\ell^{\mathrm{int}}\), together with \(W_\ell\),
determines \(C_\ell^{\mathrm{feat}}\), but the converse generally fails.
Indeed, the congruence discards every component of
\(C_\ell^{\mathrm{int}}\) involving
\(\ker W_\ell^\top\) and attenuates the components associated with small
singular values. Even when \(W_\ell\) has full row rank and the congruence can
be inverted, the reconstruction becomes unstable as the smallest singular
value approaches zero. A small feature-side commutator may therefore reflect
a small internal commutator, additional suppression under singular-value
transport, loss of covariance-null directions, or a combination of these
effects.

In this relative sense, \(C_\ell^{\mathrm{int}}\) is the more informative
object: it retains the full internal commutator before feature-side
compression and is the quantity governed by the four-source recursive
decomposition. This does not mean that it resolves every form of internal
mixing. Since the commutator itself weights
\(u_{\ell,i}^\top Q_\ell u_{\ell,j}\) by the covariance gap, it remains
insensitive to mixing inside an exactly degenerate eigenspace and weakly
sensitive inside a nearly degenerate cluster. The buffered localized energies
provide the finer spectral resolution needed in that regime. At the terminal layer, \(Q_{L-1}=I_{m_L}\), and hence both commutators vanish.

We next resolve \(C_\ell^{\mathrm{int}}\) into one transported contribution
and three local sources.

\begin{theorem}\label{thm:internal-commutator-hierarchy}
The terminal internal commutator satisfies
\[
C_{L-1}^{\mathrm{int}}=[Q_{L-1},P_{L-1}]=0.
\]
For \(\ell=0,\ldots,L-2\), define
\begin{align}
T_\ell^{\mathrm{tr}}
&:=
\mathbb E_\mu
\bigl[
D_\ell W_{\ell+1}^\top C_{\ell+1}^{\mathrm{int}}
W_{\ell+1}D_\ell
\bigr],
\label{eq:transport-component}\\
T_\ell^{\mathrm{imb}}
&:=
\mathbb E_\mu
\bigl[
D_\ell[\mathcal S_{\ell+1},\Delta_\ell]D_\ell
\bigr],
\label{eq:imbalance-component}\\
T_\ell^{\mathrm{fluc}}
&:=
\mathbb E_\mu
\bigl[
D_\ell[R_{\ell+1},P_\ell]D_\ell
\bigr],
\label{eq:fluctuation-component}\\
T_\ell^{\mathrm{gate}}
&:=
\mathbb E_\mu
\bigl[
[D_\ell,P_\ell]S_{\ell+1}D_\ell
+D_\ell S_{\ell+1}[D_\ell,P_\ell]
\bigr].
\label{eq:gate-component}
\end{align}
Then each component is skew-symmetric and
\begin{equation}\label{PrimitiveFormulaCorrectintro}
C_\ell^{\mathrm{int}}
=T_\ell^{\mathrm{tr}}
+T_\ell^{\mathrm{imb}}
+T_\ell^{\mathrm{fluc}}
+T_\ell^{\mathrm{gate}}.
\end{equation}
\end{theorem}

\begin{proof}
Since \(Q_\ell=\mathbb E_\mu[\mathsf Q_\ell]\) and
\(\mathsf Q_\ell=D_\ell S_{\ell+1}D_\ell\),
\begin{align}
C_\ell^{\mathrm{int}}
&=
\mathbb E_\mu
\bigl[
D_\ell[S_{\ell+1},P_\ell]D_\ell
+[D_\ell,P_\ell]S_{\ell+1}D_\ell
+D_\ell S_{\ell+1}[D_\ell,P_\ell]
\bigr].
\label{eq:firstsplit}
\end{align}
Using \(S_{\ell+1}=\mathcal S_{\ell+1}+R_{\ell+1}\) and
\(P_\ell=\mathrm{NFM}_{\ell+1}+\Delta_\ell\), we obtain
\begin{align*}
[S_{\ell+1},P_\ell]
&=
[\mathcal S_{\ell+1},\mathrm{NFM}_{\ell+1}]
+[\mathcal S_{\ell+1},\Delta_\ell]
+[R_{\ell+1},P_\ell]\\
&=
W_{\ell+1}^\top C_{\ell+1}^{\mathrm{int}}W_{\ell+1}
+[\mathcal S_{\ell+1},\Delta_\ell]
+[R_{\ell+1},P_\ell].
\end{align*}
Substitution into \eqref{eq:firstsplit} proves
\eqref{PrimitiveFormulaCorrectintro}. The first three components are
congruences of skew-symmetric matrices. For the last one, symmetry of
\(D_\ell\) and \(S_{\ell+1}\) gives
\[
\bigl(
[D_\ell,P_\ell]S_{\ell+1}D_\ell
+D_\ell S_{\ell+1}[D_\ell,P_\ell]
\bigr)^\top
=-
[D_\ell,P_\ell]S_{\ell+1}D_\ell
-D_\ell S_{\ell+1}[D_\ell,P_\ell].
\]
Hence all four components are skew-symmetric.
\end{proof}

The term \(T_\ell^{\mathrm{tr}}\) transports downstream internal
incompatibility. The other terms are local: \(T_\ell^{\mathrm{imb}}\) measures
the interaction with adjacent-layer imbalance,
\(T_\ell^{\mathrm{fluc}}\) measures variation of the pointwise sensitivity
around its population average, and \(T_\ell^{\mathrm{gate}}\) is the nonlinear
gate-covariance contribution. The four components need not reinforce one
another in Frobenius space and may undergo substantial cancellation; none is
asserted to decay under training.

\begin{proposition}\label{prop:component-energy-decomposition}
For
\[
\mathcal E_\ell^\nu:=\|T_\ell^\nu\|_F^2,
\qquad
\mathcal I_\ell^{\nu,\eta}
:=2\langle T_\ell^\nu,T_\ell^\eta\rangle_F,
\qquad
\nu,\eta\in
\{\mathrm{tr},\mathrm{imb},\mathrm{fluc},\mathrm{gate}\},
\]
one has
\begin{equation}\label{eq:component-energy-decomposition}
\|C_\ell^{\mathrm{int}}\|_F^2
=
\sum_\nu\mathcal E_\ell^\nu
+
\sum_{\nu<\eta}\mathcal I_\ell^{\nu,\eta}.
\end{equation}
If \(\sum_\nu\|T_\ell^\nu\|_F>0\), the aggregate cancellation ratio
\begin{equation}\label{eq:aggregate-cancellation-ratio}
\chi_\ell
:=
\frac{\|C_\ell^{\mathrm{int}}\|_F}
{\sum_\nu\|T_\ell^\nu\|_F}
\end{equation}
satisfies \(0\leq\chi_\ell\leq1\). Small \(\chi_\ell\) indicates strong
matrix-valued cancellation among the source components. If the denominator in
\eqref{eq:aggregate-cancellation-ratio} vanishes, then every component and
\(C_\ell^{\mathrm{int}}\) vanish, and \(\chi_\ell\) is left undefined.
\end{proposition}

\begin{proof}
Both statements follow by expanding the squared Frobenius norm in
\eqref{PrimitiveFormulaCorrectintro} and applying the triangle inequality.
\end{proof}

Thus a small internal commutator has two distinct explanations: the
components may be individually small, or large components may cancel.
The following norm estimate controls the first mechanism but necessarily
discards the second.

Define the gate-covariance energy
\begin{equation}\label{eq:gate-covariance-energy}
\mathcal G_\ell(t)
:=
\mathbb E_\mu\|[D_\ell(t),P_\ell(t)]\|_F^2,
\qquad \ell=0,\ldots,L-2.
\end{equation}

\begin{lemma}\label{lem:recursive-estimate}
For \(\ell=0,\ldots,L-2\), set
\begin{align*}
d_\ell
&:=\mathbb E_\mu\|D_\ell\|_{\mathrm{op}}^2,
&
v_\ell
&:=\mathbb E_\mu
\bigl[\|D_\ell\|_{\mathrm{op}}^2\|R_{\ell+1}\|_F\bigr],
&
s_\ell
&:=\mathbb E_\mu
\bigl[\|S_{\ell+1}\|_{\mathrm{op}}^2
\|D_\ell\|_{\mathrm{op}}^2\bigr].
\end{align*}
Then
\begin{align}
\|C_\ell^{\mathrm{int}}\|_F^2
&\leq
4\|W_{\ell+1}\|_{\mathrm{op}}^4d_\ell^2
\|C_{\ell+1}^{\mathrm{int}}\|_F^2
+16d_\ell^2\|\Delta_\ell\|_{\mathrm{op}}^2
\|\mathcal S_{\ell+1}\|_F^2
+16\|W_\ell\|_{\mathrm{op}}^4v_\ell^2
+16\mathcal G_\ell s_\ell.
\label{esta+}
\end{align}
\end{lemma}

\begin{proof}
The four components satisfy
\begin{align*}
\|T_\ell^{\mathrm{tr}}\|_F
&\leq
\|W_{\ell+1}\|_{\mathrm{op}}^2d_\ell
\|C_{\ell+1}^{\mathrm{int}}\|_F,\qquad\|T_\ell^{\mathrm{imb}}\|_F
\leq
2d_\ell\|\Delta_\ell\|_{\mathrm{op}}
\|\mathcal S_{\ell+1}\|_F,\\
\|T_\ell^{\mathrm{fluc}}\|_F
&\leq
2\|W_\ell\|_{\mathrm{op}}^2v_\ell,\qquad
\|T_\ell^{\mathrm{gate}}\|_F
\leq
2\mathbb E_\mu
\bigl[
\|[D_\ell,P_\ell]\|_F
\|S_{\ell+1}\|_{\mathrm{op}}
\|D_\ell\|_{\mathrm{op}}
\bigr].
\end{align*}
Here we used
\(\|[A,B]\|_F\leq2\|A\|_{\mathrm{op}}\|B\|_F\) and
\(\|P_\ell\|_{\mathrm{op}}=\|W_\ell\|_{\mathrm{op}}^2\).
Applying Cauchy-Schwarz to the last expectation and
\((a+b+c+d)^2\leq4(a^2+b^2+c^2+d^2)\) proves \eqref{esta+}.
\end{proof}

\begin{corollary}\label{cor:ReLU-commutator-estimate}
If \(\sigma\) is ReLU, then
\begin{align}
\|C_\ell^{\mathrm{int}}\|_F^2
&\leq
4\|W_{\ell+1}\|_{\mathrm{op}}^4
\|C_{\ell+1}^{\mathrm{int}}\|_F^2
+16\|\Delta_\ell\|_{\mathrm{op}}^2
\|\mathcal S_{\ell+1}\|_F^2
+16\|W_\ell\|_{\mathrm{op}}^4
\mathbb E_\mu\|R_{\ell+1}\|_F^2
+16\mathcal G_\ell
\mathbb E_\mu\|S_{\ell+1}\|_{\mathrm{op}}^2.
\label{ReLUCommutatorEstimate}
\end{align}
\end{corollary}

\begin{proof}
For ReLU, \(\|D_\ell\|_{\mathrm{op}}\leq1\). Hence
\(d_\ell\leq1\),
\(v_\ell^2\leq\mathbb E_\mu\|R_{\ell+1}\|_F^2\), and
\(s_\ell\leq\mathbb E_\mu\|S_{\ell+1}\|_{\mathrm{op}}^2\).
The conclusion follows from Lemma~\ref{lem:recursive-estimate}.
\end{proof}

At the first nontrivial internal level,
\[
C_{L-1}^{\mathrm{int}}=0,
\qquad
S_{L-1}=\mathcal S_{L-1}=W_{L-1}^\top W_{L-1},
\qquad
R_{L-1}=0.
\]
Consequently,
\begin{equation}\label{eq:penultimate-exact-recursion}
C_{L-2}^{\mathrm{int}}
=T_{L-2}^{\mathrm{imb}}+T_{L-2}^{\mathrm{gate}},
\end{equation}
and
\begin{equation}\label{eq:penultimate-energy-decomposition}
\|C_{L-2}^{\mathrm{int}}\|_F^2
=
\mathcal E_{L-2}^{\mathrm{imb}}
+\mathcal E_{L-2}^{\mathrm{gate}}
+\mathcal I_{L-2}^{\mathrm{imb},\mathrm{gate}}.
\end{equation}
For ReLU,
\begin{equation}\label{eq:penultimate-ReLU-estimate}
\|C_{L-2}^{\mathrm{int}}\|_F^2
\leq
16\|\Delta_{L-2}\|_{\mathrm{op}}^2\|S_{L-1}\|_F^2
+16\mathcal G_{L-2}\|S_{L-1}\|_{\mathrm{op}}^2.
\end{equation}
Thus the hierarchy is initiated by imbalance and gate geometry, after which
the resulting internal incompatibility is transported toward the input and
supplemented by all three local sources. This is not an implication chain
from \([D_\ell,P_\ell]\) to \(C_\ell^{\mathrm{int}}\) to
\(C_\ell^{\mathrm{feat}}\): the gate commutator is only one local source of
\(C_\ell^{\mathrm{int}}\), while
\(C_\ell^{\mathrm{feat}}=W_\ell^\top C_\ell^{\mathrm{int}}W_\ell\) is a
weighted compression of the combined internal geometry.

\begin{remark}
\label{rem:choice-source-decomposition}
The four-source decomposition in
\eqref{PrimitiveFormulaCorrectintro} is exact but not algebraically unique.
Its form is chosen to separate the mechanisms relevant to the commutator
hierarchy. The transport term is the only component containing the downstream
commutator \(C_{\ell+1}^{\mathrm{int}}\); the imbalance term isolates the
interaction of the population sensitivity \(\mathcal S_{\ell+1}\) with the
adjacent-layer mismatch \(\Delta_\ell\); the fluctuation term is the part of
the central commutator contribution
\(D_\ell[S_{\ell+1},P_\ell]D_\ell\) generated by the centered field
\(R_{\ell+1}\); and the gate term contains the product-rule contributions
generated by \([D_\ell,P_\ell]\), weighted by the full pointwise sensitivity
\(S_{\ell+1}\). Thus the labels describe the algebraic origin of the terms
rather than mutually exclusive dependence on the underlying quantities. This
convention keeps every component skew-symmetric and provides the source
normalization used in the energy identities and numerical diagnostics below. For a coarser description of propagation, the three local sources can be
combined. Define
\[
\mathcal T_\ell(A)
:=
\mathbb E_\mu
\bigl[
D_\ell W_{\ell+1}^\top A
W_{\ell+1}D_\ell
\bigr]
\]
and
\[
T_\ell^{\mathrm{loc}}
:=
T_\ell^{\mathrm{imb}}
+
T_\ell^{\mathrm{fluc}}
+
T_\ell^{\mathrm{gate}}.
\]
Then \eqref{PrimitiveFormulaCorrectintro} becomes
\[
C_\ell^{\mathrm{int}}
=
\mathcal T_\ell
\bigl(
C_{\ell+1}^{\mathrm{int}}
\bigr)
+
T_\ell^{\mathrm{loc}}.
\]
Since \(C_{L-1}^{\mathrm{int}}=0\), iteration gives
\[
C_\ell^{\mathrm{int}}
=
\sum_{j=\ell}^{L-2}
\bigl(
\mathcal T_\ell\circ\cdots\circ\mathcal T_{j-1}
\bigr)
\bigl(
T_j^{\mathrm{loc}}
\bigr),
\]
where the composition is interpreted as the identity when \(j=\ell\).
Thus every internal incompatibility is generated locally at the same or a
deeper layer and, when generated deeper, is transported backward through the
intervening layers. Finer decompositions are possible. For example,
\[
[R_{\ell+1},P_\ell]
=
[R_{\ell+1},\mathrm{NFM}_{\ell+1}]
+
[R_{\ell+1},\Delta_\ell]
\]
splits the fluctuation source into a contribution relative to the downstream
NFM and a mixed fluctuation-imbalance contribution. The gate source could
similarly be divided by writing
\(S_{\ell+1}=\mathcal S_{\ell+1}+R_{\ell+1}\). Such refinements introduce
additional source energies and pairwise interactions without changing the
underlying identity. We retain the four-source formulation because it separates the mechanisms
studied here without introducing additional mixed source terms and pairwise
interactions. Accordingly, all source-level energies and cancellation
diagnostics below refer to this fixed convention.
\end{remark}
\subsection{Evolution of the adjacent-layer imbalance}

We next record how one local source evolves under gradient flow. Set
\[
G_j(t):=\nabla_{W_j}\mathcal R(W(t))
\]
and
\begin{align}
F_\ell
&:=
-G_\ell W_\ell^\top-W_\ell G_\ell^\top
+G_{\ell+1}^\top W_{\ell+1}
+W_{\ell+1}^\top G_{\ell+1}.
\label{eq:imbalance-forcing}
\end{align}

\begin{proposition}\label{prop:ImbalanceEvolution}
Under the unregularized gradient flow \(\dot W_j=-G_j\),
\begin{equation}\label{ImbalanceFlow}
\dot\Delta_\ell=F_\ell.
\end{equation}
Moreover,
\begin{align}
\|F_\ell\|_{\mathrm{op}}
&\leq
2\|W_\ell\|_{\mathrm{op}}\|G_\ell\|_{\mathrm{op}}
+2\|W_{\ell+1}\|_{\mathrm{op}}
\|G_{\ell+1}\|_{\mathrm{op}}.
\label{ImbalanceEstimate}
\end{align}
Under the regularized flow \(\dot W_j=-G_j-\lambda W_j\), \(\lambda>0\),
with \(G_j=\nabla_{W_j}\mathcal R\) as above,
\begin{equation}\label{RegularizedImbalanceFlow}
\dot\Delta_\ell+2\lambda\Delta_\ell=F_\ell.
\end{equation}
\end{proposition}

\begin{proof}
Differentiate
\(\Delta_\ell=W_\ell W_\ell^\top-W_{\ell+1}^\top W_{\ell+1}\)
and substitute the corresponding gradient flow. The norm estimate follows
from the triangle inequality and submultiplicativity.
\end{proof}
\begin{remark}
\label{rem:deep-linear-balancedness}
For a deep linear network and any differentiable loss depending on the weights
only through the end-to-end product, the chain rule gives
\[
G_\ell W_\ell^\top
=
W_{\ell+1}^\top G_{\ell+1},
\qquad
W_\ell G_\ell^\top
=
G_{\ell+1}^\top W_{\ell+1}.
\]
Consequently, \(F_\ell=0\). Under unregularized gradient flow, this yields $\Delta_\ell(t)=\Delta_\ell(0)$, whereas under quadratic weight decay with the same coefficient \(\lambda>0\)
at every layer, $\Delta_\ell(t)
=
e^{-2\lambda t}\Delta_\ell(0)$. These are the classical matrix balancedness identities; see, for example,
\cite{AroraCohenGolowichHu2019}. In nonlinear networks, the gates generally
destroy the full matrix identities, and \(F_\ell\) measures the corresponding
failure of the deep-linear conservation law. For bias-free networks with
positively homogeneous activations, however, a diagonal neuronwise
conservation law survives, as explained in
Remark~\ref{rem:neuronwise-balancedness} below.
\end{remark}

\begin{corollary}\label{cor:ImbalanceLongTime}
Suppose that
\begin{equation}\label{eq:imbalance-integrability}
\int_0^\infty
\bigl(
\|W_\ell(t)\|_{\mathrm{op}}\|G_\ell(t)\|_{\mathrm{op}}
+\|W_{\ell+1}(t)\|_{\mathrm{op}}
\|G_{\ell+1}(t)\|_{\mathrm{op}}
\bigr)\,\d t<\infty.
\end{equation}
Then, for the unregularized flow, \(\Delta_\ell(t)\) converges to a symmetric
matrix \(\Delta_\ell^\infty\), and
\begin{align}
\|\Delta_\ell^\infty-\Delta_\ell(t)\|_{\mathrm{op}}
&\leq
2\int_t^\infty
\bigl(
\|W_\ell(s)\|_{\mathrm{op}}\|G_\ell(s)\|_{\mathrm{op}}
+\|W_{\ell+1}(s)\|_{\mathrm{op}}
\|G_{\ell+1}(s)\|_{\mathrm{op}}
\bigr)\,\d s.
\label{eq:imbalance-tail-estimate}
\end{align}
For the regularized flow,
\begin{align}
\|\Delta_\ell(t)\|_{\mathrm{op}}
&\leq
e^{-2\lambda t}\|\Delta_\ell(0)\|_{\mathrm{op}}+
2\int_0^t e^{-2\lambda(t-s)}
\bigl(
\|W_\ell(s)\|_{\mathrm{op}}\|G_\ell(s)\|_{\mathrm{op}}
+\|W_{\ell+1}(s)\|_{\mathrm{op}}
\|G_{\ell+1}(s)\|_{\mathrm{op}}
\bigr)\,\d s,
\label{RegularizedImbalanceIntegralEstimate}
\end{align}
and \(\Delta_\ell(t)\to0\) as \(t\to\infty\).
\end{corollary}

\begin{proof}
In the unregularized case, \eqref{ImbalanceFlow}-\eqref{ImbalanceEstimate}
make \(\dot\Delta_\ell\) integrable in operator norm; the Cauchy criterion
and integration over \((t,\infty)\) give the conclusion. In the regularized
case, variation of constants in \eqref{RegularizedImbalanceFlow} and
\eqref{ImbalanceEstimate} give the displayed estimate. The convolution of an
\(\mathrm{L}^1(0,\infty)\) forcing with \(e^{-2\lambda t}\) converges to zero.
\end{proof}

Risk dissipation controls the time integral of squared gradient norms, but it
does not by itself imply \eqref{eq:imbalance-integrability}, which contains an
additional weight factor and is an \(\mathrm{L}^1\)-in-time condition.
Accordingly,
unregularized gradient flow may stabilize the imbalance at a nonzero limit;
the result above does not assert full matrix balancing.

\begin{remark}
\label{rem:neuronwise-balancedness}
Assume that the network is bias-free and that the activation is positively
homogeneous of degree one, as for the ReLU activation. The corresponding
neuronwise rescaling invariance implies that, along the unregularized gradient
flow,
\[
\frac{\d}{\d t}
\left(
\|W_\ell[i,:](t)\|_2^2
-
\|W_{\ell+1}[:,i](t)\|_2^2
\right)
=0.
\]
Since
\[
(\Delta_\ell)_{ii}
=
\|W_\ell[i,:]\|_2^2
-
\|W_{\ell+1}[:,i]\|_2^2,
\]
it follows that
\[
\operatorname{Diag}\Delta_\ell(t)
=
\operatorname{Diag}\Delta_\ell(0);
\]
see \cite{DuHuLee2018}. In particular,
\[
\|W_\ell(t)\|_F^2
-
\|W_{\ell+1}(t)\|_F^2
=
\operatorname{tr}\Delta_\ell(0).
\]
Thus, writing
\[
\Delta_\ell(t)
=
\operatorname{Diag}\Delta_\ell(0)
+
\operatorname{Off}\Delta_\ell(t),
\]
the diagonal part is fixed by the initialization, whereas the off-diagonal
part may be generated and modified by the nonlinear dynamics. This contrasts
with the deep linear setting, in which the full matrix \(\Delta_\ell\) is
conserved. The distinction is directly relevant to the imbalance source in the
commutator recursion, since
\[
[S_{\ell+1},\Delta_\ell]
=
[S_{\ell+1},\operatorname{Diag}\Delta_\ell(0)]
+
[S_{\ell+1},\operatorname{Off}\Delta_\ell].
\]
Consequently, neuronwise balanced initialization,
\(\operatorname{Diag}\Delta_\ell(0)=0\), eliminates the conserved diagonal
contribution but does not prevent the nonlinear dynamics from generating an
off-diagonal imbalance source.
Moreover, scalar multiples of the identity commute with every matrix, so
only the anisotropic part
\[
\Delta_\ell^\circ
=
\Delta_\ell
-
\frac{\operatorname{tr}\Delta_\ell}{m_{\ell+1}}I
\]
contributes to this commutator. Full matrix balancedness is therefore neither
preserved nor required for the commutator hierarchy; what matters is the part
of the imbalance that is spectrally incompatible with \(S_{\ell+1}\).
\end{remark}

\subsection{Sensitivity fluctuations across the data distribution}

The fluctuation source is governed by
\begin{equation}\label{eq:fluctuation-recall}
R_{\ell+1}(x,t)
=S_{\ell+1}(x,t)-\mathcal S_{\ell+1}(t),
\qquad
\mathbb E_\mu[R_{\ell+1}]=0,
\end{equation}
and its natural size is
\begin{equation}\label{eq:fluctuation-variance}
\mathcal V_{\ell+1}(t)
:=\mathbb E_\mu\|R_{\ell+1}(t)\|_F^2.
\end{equation}

\begin{proposition}\label{prop:fluctuation-identities}
Let \(x,x'\) be independent with law \(\mu\). Then
\begin{align}
\mathcal V_{\ell+1}
&=
\mathbb E_\mu\|S_{\ell+1}\|_F^2
-\|\mathcal S_{\ell+1}\|_F^2\\
&=
\frac12\mathbb E_{\mu\otimes\mu}
\|S_{\ell+1}(x)-S_{\ell+1}(x')\|_F^2.
\label{eq:pairwise-fluctuation}
\end{align}
If \(g_{\ell+1}=\nabla_{a_{\ell+1}}F\), so that
\(S_{\ell+1}=g_{\ell+1}g_{\ell+1}^\top\), then
\begin{align}
\mathcal V_{\ell+1}
&=
\mathbb E_\mu\|g_{\ell+1}\|^4
-\left\|
\mathbb E_\mu[g_{\ell+1}g_{\ell+1}^\top]
\right\|_F^2
\label{eq:rank-one-fluctuation}\\
&=
\frac12\mathbb E_{\mu\otimes\mu}
\bigl[
\|g_{\ell+1}(x)\|^4+\|g_{\ell+1}(x')\|^4
-2\langle g_{\ell+1}(x),g_{\ell+1}(x')\rangle^2
\bigr].
\label{eq:pairwise-sensitivity-fluctuation}
\end{align}
In particular, \(\mathcal V_{\ell+1}=0\) if and only if
\(S_{\ell+1}(x)=\mathcal S_{\ell+1}\) for \(\mu\)-almost every \(x\).
\end{proposition}

\begin{proof}
The first identity follows by expanding
\(\mathbb E_\mu\|S_{\ell+1}-\mathcal S_{\ell+1}\|_F^2\); the second is the
independent-copy variance identity in the Frobenius space. The final two
follow from
\(\|gg^\top\|_F^2=\|g\|^4\) and
\(\langle gg^\top,hh^\top\rangle_F=\langle g,h\rangle^2\).
\end{proof}

The fluctuation component satisfies
\begin{align}
\|T_\ell^{\mathrm{fluc}}\|_F^2
&\leq
4\|W_\ell\|_{\mathrm{op}}^4
\left(
\mathbb E_\mu
\bigl[
\|D_\ell\|_{\mathrm{op}}^2\|R_{\ell+1}\|_F
\bigr]
\right)^2\leq
4\|W_\ell\|_{\mathrm{op}}^4
\mathbb E_\mu\|D_\ell\|_{\mathrm{op}}^4
\,\mathcal V_{\ell+1}.
\label{FluctuationReduction}
\end{align}
For ReLU gates,
\begin{equation}\label{eq:ReLU-fluctuation-component}
\|T_\ell^{\mathrm{fluc}}\|_F^2
\leq
4\|W_\ell\|_{\mathrm{op}}^4\mathcal V_{\ell+1}.
\end{equation}
At the final hidden layer,
\(S_{L-1}=\mathcal S_{L-1}\) and \(\mathcal V_{L-1}=0\), which explains the
absence of a fluctuation term in \eqref{eq:penultimate-exact-recursion}. At
earlier layers, \(\mathcal V_{\ell+1}\) measures variation generated by the
downstream activation pattern. It has no closed monotonicity law under
gradient flow in general; its smallness is an additional hypothesis or an
empirical observation, not a consequence of the hierarchy.

\begin{remark}\label{rem:Poincare-fluctuation-control}
Suppose that \(\mu\) satisfies the \(\mathrm{L}^2(\mu)\) Poincar\'e inequality
\[
\|f-\mathbb E_\mu[f]\|_{\mathrm{L}^2(\mu)}^2
\leq
C_{\mathrm P}
\|\nabla f\|_{\mathrm{L}^2(\mu)}^2
\]
for every sufficiently regular scalar function \(f\). Assume that
\(x\mapsto S_{\ell+1}(x,t)\) belongs entrywise to
\(\mathrm{H}^1(\mu)\). Applying the scalar Poincar\'e inequality to the
matrix entries and summing gives
\begin{equation}
\mathcal V_{\ell+1}(t)
\leq
C_{\mathrm P}\,
\mathbb E_\mu
\left\|
\nabla_x S_{\ell+1}(x,t)
\right\|_F^2,
\label{eq:Poincare-fluctuation}
\end{equation}
where the norm on the right denotes the Hilbert-Schmidt norm over both the
matrix and input-variable indices. If
\(S_{\ell+1}=g_{\ell+1}g_{\ell+1}^\top\), then
\[
\partial_{x_j}S_{\ell+1}
=
(\partial_{x_j}g_{\ell+1})g_{\ell+1}^\top
+
g_{\ell+1}(\partial_{x_j}g_{\ell+1})^\top,
\]
and consequently
\[
\left\|
\nabla_x S_{\ell+1}
\right\|_F^2
\leq
4\|g_{\ell+1}\|^2
\left\|
\nabla_x g_{\ell+1}
\right\|_F^2.
\]
Thus,
\begin{equation}
\mathcal V_{\ell+1}(t)
\leq
4C_{\mathrm P}\,
\mathbb E_\mu
\left[
\|g_{\ell+1}(x,t)\|^2
\left\|
\nabla_x g_{\ell+1}(x,t)
\right\|_F^2
\right].
\label{eq:Poincare-sensitivity-fluctuation}
\end{equation}
For smooth activations, this reduces the estimation of the fluctuation source
to the spatial regularity of the backward sensitivity. For ReLU networks,
however, \(g_{\ell+1}\) is generally piecewise smooth and may jump across
downstream activation boundaries. It therefore need not belong to
\(\mathrm{H}^1(\mu)\), and its almost-everywhere derivative does not detect
the jumps between activation regions. Hence
\eqref{eq:Poincare-fluctuation} cannot be applied naively in the ReLU setting.
There the pairwise identity \eqref{eq:pairwise-fluctuation} is more robust:
it suggests estimating \(\mathcal V_{\ell+1}\) through the frequency and size
of changes in the downstream activation pattern. In particular,
\[
\mathcal V_{\ell+1}
\leq
\frac12\,
\mathbb E_{\mu\otimes\mu}
\left[
\bigl(
\|g_{\ell+1}(x)\|+\|g_{\ell+1}(x')\|
\bigr)^2
\|g_{\ell+1}(x)-g_{\ell+1}(x')\|^2
\right].
\]
These estimates provide possible sufficient conditions for small
fluctuations, but they do not imply that
\(\mathcal V_{\ell+1}(t)\) is small or decreasing along gradient flow.
\end{remark}

\section{Spectral localization and covariance geometry}
\label{sec:SpectralGeometry}

The three commutators in \eqref{Pre:three-levels} measure different forms of
spectral compatibility. The gate and internal commutators act on the covariance
space of \(P_\ell\) and may therefore be localized relative to the same
spectral decomposition, whereas the feature commutator is their congruence
image under \(W_\ell\). We develop this common localization, quantify the
motion of the covariance subspaces, and identify the singular-value weights
introduced by transport to the feature side. These results describe spectral
resolution and stability; they do not imply decay of any commutator.

\subsection{Common and buffered spectral localization}

Fix a layer \(\ell\), suppress the time variable, and use the spectral
decomposition \eqref{Pre22-}-\eqref{Pre23+}. The following elementary identity
is the basis of the localization theory.

\begin{proposition}\label{prop:common-spectral-localization}
Let \(A_\ell\) be symmetric on \(\mathbb R^{m_{\ell+1}}\). Then
\begin{equation}\label{eq:common-commutator-expansion}
\|[A_\ell,P_\ell]\|_F^2
=
\sum_{i,j=1}^{m_{\ell+1}}
(\lambda_{\ell,i}-\lambda_{\ell,j})^2
|u_{\ell,i}^\top A_\ell u_{\ell,j}|^2
=
2\sum_{i<j}
(\lambda_{\ell,i}-\lambda_{\ell,j})^2
|u_{\ell,i}^\top A_\ell u_{\ell,j}|^2.
\end{equation}
If \(I_{\mathsf a},I_{\mathsf b}\) are disjoint spectral index sets,
\(
\Pi_{\ell,\mathsf a}=\sum_{i\in I_{\mathsf a}}
u_{\ell,i}u_{\ell,i}^\top
\),
\(
\Pi_{\ell,\mathsf b}=\sum_{j\in I_{\mathsf b}}
u_{\ell,j}u_{\ell,j}^\top
\), and
\begin{equation}\label{eq:cluster-distance}
\delta_\ell(\mathsf a,\mathsf b)
:=
\min_{i\in I_{\mathsf a},\,j\in I_{\mathsf b}}
|\lambda_{\ell,i}-\lambda_{\ell,j}|,
\end{equation}
then
\begin{equation}\label{eq:common-cluster-bound}
\|[A_\ell,P_\ell]\|_F^2
\geq
2\delta_\ell(\mathsf a,\mathsf b)^2
\|\Pi_{\ell,\mathsf a}A_\ell\Pi_{\ell,\mathsf b}\|_F^2.
\end{equation}
\end{proposition}

\begin{proof}
Set \(\widetilde A_\ell=U_\ell^\top A_\ell U_\ell\). In the eigenbasis of
\(P_\ell\),
\[
\bigl(U_\ell^\top[A_\ell,P_\ell]U_\ell\bigr)_{ij}
=(\lambda_{\ell,j}-\lambda_{\ell,i})(\widetilde A_\ell)_{ij}.
\]
Orthogonal invariance of the Frobenius norm gives the first identity, and
symmetry gives the second. Restricting the sum to
\(I_{\mathsf a}\times I_{\mathsf b}\) and its transpose yields
\eqref{eq:common-cluster-bound}.
\end{proof}

Thus the commutator weights the mixing coefficient
\(|u_{\ell,i}^\top A_\ell u_{\ell,j}|\) by the spectral separation of the
corresponding covariance directions. It weakly detects mixing within a nearly
degenerate cluster and strongly detects mixing across separated clusters.

Fix integers \(1\leq k\leq k+\beta<m_{\ell+1}\), \(\beta\geq0\), and assume
\(\gamma_\ell^{(k)},\gamma_\ell^{(k+\beta)}>0\). Define the buffered gap
\begin{equation}\label{eq:buffered-gap}
\gamma_\ell^{(k,\beta)}
:=
\lambda_{\ell,k}-\lambda_{\ell,k+\beta+1}.
\end{equation}
The buffer removes the intermediate directions
\(u_{\ell,k+1},\ldots,u_{\ell,k+\beta}\) and retains only mixing from the
leading \(k\)-dimensional subspace into the complement of the leading
\(k+\beta\) directions. Set
\begin{align}
\mathcal G_{D,\ell}^{(k,\beta)}(t)
&:=
\int_{\mathbb R^d}
\|\Pi_\ell^{(k)}(t)D_\ell(x,t)
\Pi_\ell^{(k+\beta),\perp}(t)\|_F^2\,\d\mu(x),
\label{eq:localized-gate-energy}\\
\mathcal G_{Q,\ell}^{(k,\beta)}(t)
&:=
\|\Pi_\ell^{(k)}(t)Q_\ell(t)
\Pi_\ell^{(k+\beta),\perp}(t)\|_F^2.
\label{eq:localized-sensitivity-energy}
\end{align}
For \(\beta=0\), we omit \(\beta\) from the superscript. In addition to the
global gate energy \(\mathcal G_\ell\) from
\eqref{eq:gate-covariance-energy}, define
\begin{equation}\label{eq:global-sensitivity-energy}
\mathcal G_{Q,\ell}(t)
:=
\|[Q_\ell(t),P_\ell(t)]\|_F^2
=\|C_\ell^{\mathrm{int}}(t)\|_F^2.
\end{equation}
Proposition~\ref{prop:common-spectral-localization}, applied pointwise to
\(D_\ell\) and directly to \(Q_\ell\), gives
\begin{align}
\mathcal G_\ell(t)
&\geq
2\bigl(\gamma_\ell^{(k,\beta)}(t)\bigr)^2
\mathcal G_{D,\ell}^{(k,\beta)}(t),
\label{eq:localized-gate-gap-bound}\\
\mathcal G_{Q,\ell}(t)
&\geq
2\bigl(\gamma_\ell^{(k,\beta)}(t)\bigr)^2
\mathcal G_{Q,\ell}^{(k,\beta)}(t).
\label{eq:localized-sensitivity-gap-bound}
\end{align}
The first localized energy measures the pointwise nonlinear mixing that enters
the internal recursion as a local source; the second measures the accumulated
incompatibility between backward sensitivity and covariance geometry. For
\(\beta=0\), their vanishing is equivalent, respectively, to invariance of
\(\operatorname{Ran}\Pi_\ell^{(k)}\) under \(D_\ell(x,t)\) for
\(\mu\)-almost every \(x\), and under \(Q_\ell(t)\). For \(\beta>0\),
vanishing excludes only direct coupling to the remote complement; mixing
through the buffered directions may remain.

The same definitions apply to any two disjoint spectral clusters. This is the
natural formulation near repeated eigenvalues: the cluster projector is
intrinsic even when the eigenvectors inside the cluster are not, and
\eqref{eq:common-cluster-bound} depends only on the separation between the
clusters.

For later use, let \(A_\ell(x,t)\) be symmetric and define
\begin{align*}
\mathcal E_{A,\ell}^{(k,\beta)}
&:=
\int\|\Pi_\ell^{(k)}A_\ell
\Pi_\ell^{(k+\beta),\perp}\|_F^2\,\d\mu,\\
\mathcal C_{A,\ell}^{(k,\beta)}
&:=
\int\|[\Pi_\ell^{(k)}A_\ell
\Pi_\ell^{(k+\beta),\perp},P_\ell]\|_F^2\,\d\mu.
\end{align*}
The eigenbasis expansion gives the exact equivalence
\begin{equation}\label{eq:weighted-localized-equivalence}
\bigl(\gamma_\ell^{(k,\beta)}\bigr)^2
\mathcal E_{A,\ell}^{(k,\beta)}
\leq
\mathcal C_{A,\ell}^{(k,\beta)}
\leq
(\lambda_{\ell,1}-\lambda_{\ell,m_{\ell+1}})^2
\mathcal E_{A,\ell}^{(k,\beta)}.
\end{equation}
Thus \(\mathcal E_{A,\ell}^{(k,\beta)}\) measures unweighted mixing across
the buffered interface, whereas \(\mathcal C_{A,\ell}^{(k,\beta)}\) records
its spectrally weighted contribution to the commutator.

\subsection{Motion and stability of covariance subspaces}

Write \(G_\ell=\nabla_{W_\ell}\mathcal R(W)\). We consider both
\(\dot W_\ell=-G_\ell\) and
\(\dot W_\ell=-G_\ell-\lambda W_\ell\), \(\lambda\geq0\).

\begin{proposition}\label{prop:covariance-evolution-spectral}
Under unregularized gradient flow,
\begin{equation}\label{eq:covariance-evolution-spectral}
\dot P_\ell
=-G_\ell W_\ell^\top-W_\ell G_\ell^\top,
\qquad
\|\dot P_\ell\|_{\mathrm{op}}
\leq2\|W_\ell\|_{\mathrm{op}}\|G_\ell\|_{\mathrm{op}}.
\end{equation}
Under gradient flow with weight decay,
\begin{equation}\label{eq:regularized-covariance-evolution}
\dot P_\ell+2\lambda P_\ell
=-G_\ell W_\ell^\top-W_\ell G_\ell^\top=:H_\ell,
\qquad
\|H_\ell\|_{\mathrm{op}}
\leq2\|W_\ell\|_{\mathrm{op}}\|G_\ell\|_{\mathrm{op}}.
\end{equation}
\end{proposition}

\begin{proof}
Differentiate \(P_\ell=W_\ell W_\ell^\top\), insert the corresponding
gradient-flow equation, and use submultiplicativity.
\end{proof}

Assume that \(P_\ell(t)\) is continuously differentiable on an interval
\(I\) and that
\[
\gamma_\ell^{(k)}(t)
=
\lambda_{\ell,k}(t)-\lambda_{\ell,k+1}(t)>0,
\qquad t\in I.
\]
Fix \(t_0\in I\). By continuity of the spectrum, there exist a neighborhood
\(J\subset I\) of \(t_0\) and a positively oriented simple closed contour
\(\Gamma_\ell^{(k)}\) that lies in the resolvent set of \(P_\ell(t)\) for
every \(t\in J\), encloses
\[
\lambda_{\ell,1}(t),\ldots,\lambda_{\ell,k}(t),
\]
and excludes the remaining eigenvalues. The contour can therefore be held
fixed for \(t\in J\), and the leading spectral projector admits the Riesz
representation; see, for example, \cite{Kato1995}.
\begin{align}
\Pi_\ell^{(k)}(t)
&=
\frac{1}{2\pi i}
\oint_{\Gamma_\ell^{(k)}}
(zI-P_\ell(t))^{-1}\,\d z.
\label{eq:Riesz-projector}
\end{align}
Because \(P_\ell(t)\) is real symmetric, this Riesz projector coincides with
the orthogonal projector onto the leading \(k\)-dimensional covariance
subspace. The construction depends only on the separation between the leading
and trailing spectral clusters; no simplicity assumption is required for the
eigenvalues inside either cluster.

For \(z\in\Gamma_\ell^{(k)}\), differentiation of the resolvent gives
\[
\frac{\d}{\d t}
(zI-P_\ell(t))^{-1}
=
(zI-P_\ell(t))^{-1}
\dot P_\ell(t)
(zI-P_\ell(t))^{-1}.
\]
Differentiating under the contour integral therefore yields
\begin{align}
\dot\Pi_\ell^{(k)}(t)
&=
\frac{1}{2\pi i}
\oint_{\Gamma_\ell^{(k)}}
(zI-P_\ell(t))^{-1}
\dot P_\ell(t)
(zI-P_\ell(t))^{-1}\,\d z.
\label{eq:Riesz-projector-derivative}
\end{align}
In particular, \(\Pi_\ell^{(k)}(t)\) remains continuously differentiable as
long as the boundary gap remains open. Differentiating
\((\Pi_\ell^{(k)})^2=\Pi_\ell^{(k)}\) also gives
\[
\Pi_\ell^{(k)}
\dot\Pi_\ell^{(k)}
\Pi_\ell^{(k)}
=
0,
\qquad
\Pi_\ell^{(k),\perp}
\dot\Pi_\ell^{(k)}
\Pi_\ell^{(k),\perp}
=
0,
\]
and hence
\[
\dot\Pi_\ell^{(k)}
=
\Pi_\ell^{(k),\perp}
\dot\Pi_\ell^{(k)}
\Pi_\ell^{(k)}
+
\Pi_\ell^{(k)}
\dot\Pi_\ell^{(k)}
\Pi_\ell^{(k),\perp}.
\]
Thus the projector derivative records only motion between the leading
subspace and its complement; rotations of a chosen eigenbasis within either
spectral cluster do not affect it. The two resolvent factors in
\eqref{eq:Riesz-projector-derivative} also explain the inverse dependence of
the projector velocity on the separating spectral gap.

\begin{proposition}\label{prop:projector-velocity}
There is a universal constant \(C>0\) such that
\begin{equation}\label{eq:projector-velocity-general}
\|\dot\Pi_\ell^{(k)}(t)\|_{\mathrm{op}}
\leq
\frac{C}{\gamma_\ell^{(k)}(t)}
\|\dot P_\ell(t)\|_{\mathrm{op}}
\end{equation}
whenever \(P_\ell\) is differentiable and the gap remains open. Under either
gradient flow,
\begin{equation}\label{eq:projector-velocity-gradient}
\|\dot\Pi_\ell^{(k)}(t)\|_{\mathrm{op}}
\leq
\frac{2C}{\gamma_\ell^{(k)}(t)}
\|W_\ell(t)\|_{\mathrm{op}}\|G_\ell(t)\|_{\mathrm{op}}.
\end{equation}
\end{proposition}

\begin{proof}
The Riesz formula and the resolvent bound for a separated invariant subspace
give \eqref{eq:projector-velocity-general}; this is the differential form of
the standard invariant-subspace perturbation estimate
\cite{DavisKahan1970,Kato1995}. Equation
\eqref{eq:covariance-evolution-spectral} then gives the unregularized case of
\eqref{eq:projector-velocity-gradient}. Under weight decay, the radial term
\(-2\lambda P_\ell\) commutes with \(P_\ell\) and changes its eigenvalues but
not its eigenspaces, so only \(H_\ell\) contributes to the projector velocity.
The bound follows from \eqref{eq:regularized-covariance-evolution}.
\end{proof}

Consequently,
\begin{equation}\label{eq:projector-stability}
\|\Pi_\ell^{(k)}(t)-\Pi_\ell^{(k)}(s)\|_{\mathrm{op}}
\leq
2C\int_s^t
\frac{\|W_\ell(\tau)\|_{\mathrm{op}}
\|G_\ell(\tau)\|_{\mathrm{op}}}
{\gamma_\ell^{(k)}(\tau)}\,\d\tau.
\end{equation}
In particular, \(\Pi_\ell^{(k)}(t)\) converges in operator norm if
\begin{equation}\label{eq:projector-integrability}
\int_0^\infty
\frac{\|W_\ell(t)\|_{\mathrm{op}}
\|G_\ell(t)\|_{\mathrm{op}}}
{\gamma_\ell^{(k)}(t)}\,\d t<\infty.
\end{equation}

The inverse gap in these estimates makes spectral separation a dynamical
condition rather than a fixed background assumption.

\begin{proposition}\label{prop:gap-variation}
If \(P_\ell\) is continuously differentiable, then, for \(0\leq s\leq t\),
\begin{equation}\label{eq:gap-variation-general}
|\gamma_\ell^{(k)}(t)-\gamma_\ell^{(k)}(s)|
\leq
2\|P_\ell(t)-P_\ell(s)\|_{\mathrm{op}}
\leq
2\int_s^t\|\dot P_\ell(\tau)\|_{\mathrm{op}}\,\d\tau.
\end{equation}
Under unregularized gradient flow,
\begin{equation}\label{eq:gap-variation-gradient}
\gamma_\ell^{(k)}(t)
\geq
\gamma_\ell^{(k)}(s)
-4\int_s^t
\|W_\ell(\tau)\|_{\mathrm{op}}
\|G_\ell(\tau)\|_{\mathrm{op}}\,\d\tau.
\end{equation}
The same estimate holds for the buffered gap
\(\gamma_\ell^{(k,\beta)}\). Under weight decay it holds after replacing
\(P_\ell\) and its gaps by
\(\overline P_\ell(t)=e^{2\lambda t}P_\ell(t)\) and
\(\overline\gamma_\ell(t)=e^{2\lambda t}\gamma_\ell(t)\), with the
integrand on the right multiplied by \(e^{2\lambda\tau}\).
\end{proposition}

\begin{proof}
Weyl's inequality bounds the variation of each ordered eigenvalue by
\(\|P_\ell(t)-P_\ell(s)\|_{\mathrm{op}}\); subtracting two eigenvalues gives
\eqref{eq:gap-variation-general}. Equation
\eqref{eq:covariance-evolution-spectral} yields
\eqref{eq:gap-variation-gradient}. For weight decay,
\(\dot{\overline P}_\ell=e^{2\lambda t}H_\ell\), and the same argument
applies.
\end{proof}

In particular, a gap present at time \(s\) remains open on \([s,t]\) if the
corresponding right-hand variation in \eqref{eq:gap-variation-gradient} is
strictly smaller than its value at \(s\). No simplicity assumption is needed
inside either separated spectral cluster.

The projector motion also enters the localized energies. For a matrix-valued
function \(A\), write
\[
\|A\|_{\mathrm L^2(\mu;F)}
:=
\left(\int\|A(x)\|_F^2\,\d\mu(x)\right)^{1/2}.
\]

\begin{proposition}\label{prop:localized-energy-evolution}
Let \(A_\ell(\cdot,t),\partial_tA_\ell(\cdot,t)\in
\mathrm L^2(\mu;F)\), and set
\[
\mathcal E_{A,\ell}^{(k,\beta)}(t)
:=
\int\|\Pi_\ell^{(k)}(t)A_\ell(x,t)
\Pi_\ell^{(k+\beta),\perp}(t)\|_F^2\,\d\mu(x).
\]
Whenever the two projectors are differentiable,
\begin{align}
\left|\frac{\d}{\d t}\mathcal E_{A,\ell}^{(k,\beta)}(t)\right|
&\leq
2\sqrt{\mathcal E_{A,\ell}^{(k,\beta)}(t)}
\Bigl(
\|\partial_tA_\ell(\cdot,t)\|_{\mathrm L^2(\mu;F)}+
\bigl(\|\dot\Pi_\ell^{(k)}(t)\|_{\mathrm{op}}
+\|\dot\Pi_\ell^{(k+\beta)}(t)\|_{\mathrm{op}}\bigr)
\|A_\ell(\cdot,t)\|_{\mathrm L^2(\mu;F)}
\Bigr).
\label{eq:localized-energy-evolution}
\end{align}
\end{proposition}

\begin{proof}
For
\(B_\ell=\Pi_\ell^{(k)}A_\ell
\Pi_\ell^{(k+\beta),\perp}\),
\[
\partial_tB_\ell
=
\dot\Pi_\ell^{(k)}A_\ell\Pi_\ell^{(k+\beta),\perp}
+\Pi_\ell^{(k)}\partial_tA_\ell\Pi_\ell^{(k+\beta),\perp}
-\Pi_\ell^{(k)}A_\ell\dot\Pi_\ell^{(k+\beta)}.
\]
Differentiate \(\|B_\ell\|_{\mathrm L^2(\mu;F)}^2\) and apply
Cauchy-Schwarz and submultiplicativity.
\end{proof}

For \(A_\ell=Q_\ell\), the integral is omitted and the
\(\mathrm L^2(\mu;F)\)-norm is replaced by the Frobenius norm. Combining
\eqref{eq:localized-energy-evolution} with
\eqref{eq:projector-velocity-gradient} displays explicitly the dependence on
the two boundary gaps. The estimate is two-sided in sign and gives stability,
not monotonicity. For smooth activations it applies to \(A_\ell=D_\ell\) under
the stated regularity assumptions. A ReLU gate changes discontinuously when
an activation boundary is crossed, so its evolution must instead be treated
through switching estimates or finite-time increments.

\subsection{Transport to feature-side geometry}

Write a thin singular-value decomposition
\[
W_\ell
=
\sum_{i=1}^{r_\ell}\sqrt{\lambda_{\ell,i}}
u_{\ell,i}v_{\ell,i}^\top,
\qquad
\lambda_{\ell,1}\geq\cdots\geq\lambda_{\ell,r_\ell}>0,
\]
where \(r_\ell=\operatorname{rank}(W_\ell)\). The nonzero eigenvalues of
\(P_\ell\) and \(\mathrm{NFM}_\ell\) are the same. Extend
\(v_{\ell,1},\ldots,v_{\ell,r_\ell}\) to an orthonormal eigenbasis of
\(\mathrm{NFM}_\ell\), with the remaining vectors spanning
\(\ker W_\ell\), and set
\[
\Pi_{\ell,N}^{(j)}
:=
\sum_{i=1}^j v_{\ell,i}v_{\ell,i}^\top,
\qquad
\Pi_{\ell,N}^{(0)}=0,
\qquad
\Pi_{\ell,N}^{(m_\ell)}=I_{m_\ell}.
\]
For \(1\leq k\leq r_\ell\) and
\(k\leq k+\beta\leq m_\ell\), define
\begin{equation}\label{eq:feature-localized-energy}
\mathcal G_{\mathcal S,\ell}^{(k,\beta)}
:=
\|\Pi_{\ell,N}^{(k)}\mathcal S_\ell
\Pi_{\ell,N}^{(k+\beta),\perp}\|_F^2.
\end{equation}
When the three localization levels are compared at the same cutoff, we also
require \(k+\beta<m_{\ell+1}\). The extension through the nullspace resolves
the endpoint convention without assigning a positive singular value to a
null direction.

\begin{proposition}\label{prop:feature-side-localization}
For \(1\leq i,j\leq r_\ell\),
\begin{equation}\label{eq:feature-matrix-element}
v_{\ell,i}^\top\mathcal S_\ell v_{\ell,j}
=
\sqrt{\lambda_{\ell,i}\lambda_{\ell,j}}
\,u_{\ell,i}^\top Q_\ell u_{\ell,j}.
\end{equation}
Consequently, for the preceding range of \((k,\beta)\),
\begin{align}
\mathcal G_{\mathcal S,\ell}^{(k,\beta)}
&=
\sum_{\substack{1\leq i\leq k\\k+\beta<j\leq r_\ell}}
\lambda_{\ell,i}\lambda_{\ell,j}
|u_{\ell,i}^\top Q_\ell u_{\ell,j}|^2,
\label{eq:feature-localized-weighting}\\
\|[\mathcal S_\ell,\mathrm{NFM}_\ell]\|_F^2
&=
\sum_{i,j=1}^{r_\ell}
(\lambda_{\ell,i}-\lambda_{\ell,j})^2
\lambda_{\ell,i}\lambda_{\ell,j}
|u_{\ell,i}^\top Q_\ell u_{\ell,j}|^2.
\label{eq:feature-commutator-weighting}
\end{align}
Moreover,
\begin{align}
&\|[\Pi_{\ell,N}^{(k)}\mathcal S_\ell
\Pi_{\ell,N}^{(k+\beta),\perp},\mathrm{NFM}_\ell]\|_F^2=
\sum_{\substack{1\leq i\leq k\\k+\beta<j\leq r_\ell}}
(\lambda_{\ell,i}-\lambda_{\ell,j})^2
\lambda_{\ell,i}\lambda_{\ell,j}
|u_{\ell,i}^\top Q_\ell u_{\ell,j}|^2.
\label{eq:feature-side-weighted-localization}
\end{align}
\end{proposition}

\begin{proof}
By \eqref{Pre17},
\(\mathcal S_\ell=W_\ell^\top Q_\ell W_\ell\), and
\(W_\ell v_{\ell,i}=\sqrt{\lambda_{\ell,i}}u_{\ell,i}\), which proves
\eqref{eq:feature-matrix-element}. Summation over the indicated blocks gives
the remaining identities.
\end{proof}

\begin{remark}
\label{rem:feature-localized-filtering}
Formula~\eqref{eq:feature-localized-weighting} explains why the
feature-side commutator may vanish more frequently than the corresponding
internal commutators. Indeed, each internal mixing entry
\(u_{\ell,i}^{\top}Q_\ell u_{\ell,j}\) is transported to the feature side
with the spectral weight
\(\lambda_{\ell,i}\lambda_{\ell,j}\). Thus the feature-side observable only
detects sensitivity mixing between directions that are visible through the
covariance operator \(P_\ell\). In particular, if
\(k+\beta\geq r_\ell\), then
\[
\operatorname{Ran}\Pi_{\ell,N}^{(k+\beta),\perp}
\subseteq\ker W_\ell,
\]
and hence
\[
\mathcal G_{\mathcal S,\ell}^{(k,\beta)}=0.
\]
If, in addition, \(k+\beta<m_{\ell+1}\), then at the corresponding
covariance-side cutoff one may still have
\[
\Pi_\ell^{(k)}
Q_\ell
\Pi_\ell^{(k+\beta),\perp}
\neq 0.
\]
Hence the feature-side commutator can indicate exact alignment while
nontrivial internal sensitivity mixing persists in covariance-null
directions. More generally, whenever the common cutoff is admissible,
\begin{equation}
\mathcal G_{\mathcal S,\ell}^{(k,\beta)}
\leq
\lambda_{\ell,1}\lambda_{\ell,k+\beta+1}
\mathcal G_{Q,\ell}^{(k,\beta)},
\label{eq:feature-localized-upper-bound}
\end{equation}
with the convention that the right-hand side is zero once
\(k+\beta\geq r_\ell\). Thus mixing into low-variance directions is strongly
suppressed. The feature-side commutator is therefore a spectrally filtered,
and potentially non-faithful, image of the internal sensitivity commutator.
This provides a structural explanation for the more frequent vanishing of
feature-side commutators observed in the numerical experiments: feature-side
alignment need not imply full internal spectral alignment.
\end{remark}

For a common cutoff satisfying \(k+\beta<m_{\ell+1}\),
Proposition~\ref{prop:common-spectral-localization}, now on the feature side,
also gives
\begin{equation}\label{eq:localized-feature-gap-bound}
\|[\mathcal S_\ell,\mathrm{NFM}_\ell]\|_F^2
\geq
2\bigl(\gamma_\ell^{(k,\beta)}\bigr)^2
\mathcal G_{\mathcal S,\ell}^{(k,\beta)}.
\end{equation}
The internal and feature-side localizations detect the same off-diagonal
entries of \(Q_\ell\), but the latter multiplies each entry by
\(\lambda_{\ell,i}\lambda_{\ell,j}\). Internal mixing involving small
singular directions may therefore remain visible in
\(\mathcal G_{Q,\ell}^{(k,\beta)}\) while being suppressed on the feature
side, and directions annihilated by \(W_\ell^\top\) are invisible there. This
is the localized form of the congruence identity
\eqref{eq:commutator-congruence}; under full row rank, the converse control is
quantified by \eqref{eq:comm-two-sided}.

\begin{remark}\label{rem:pointwise-feature-localization}
The pointwise feature energy
\[
\widetilde{\mathcal G}_{S,\ell}^{(k,\beta)}
:=
\int\|\Pi_{\ell,N}^{(k)}S_\ell(x,t)
\Pi_{\ell,N}^{(k+\beta),\perp}\|_F^2\,\d\mu(x)
\]
splits, by \(S_\ell=\mathcal S_\ell+R_\ell\) and
\(\mathbb E_\mu R_\ell=0\), as
\begin{align*}
\widetilde{\mathcal G}_{S,\ell}^{(k,\beta)}
&=
\mathcal G_{\mathcal S,\ell}^{(k,\beta)}
+\int\|\Pi_{\ell,N}^{(k)}R_\ell(x,t)
\Pi_{\ell,N}^{(k+\beta),\perp}\|_F^2\,\d\mu(x).
\end{align*}
Pointwise AGOP alignment is therefore stronger than population AGOP
alignment: it also requires control of the localized sensitivity
fluctuations.
\end{remark}

The three localized quantities thus resolve different parts of the same
geometry: \(\mathcal G_{D,\ell}^{(k,\beta)}\) records nonlinear gate mixing,
\(\mathcal G_{Q,\ell}^{(k,\beta)}\) records accumulated internal mixing, and
\(\mathcal G_{\mathcal S,\ell}^{(k,\beta)}\) records the part visible after
singular-value-weighted transport. They should be compared, not identified;
none is forced by the preceding identities to decrease during training.

\begin{remark}\label{rem:higher-order-commutators}
For \(q\geq1\), define
\(
\operatorname{ad}_{P_\ell}^{1}(A)=[A,P_\ell]
\) and
\(
\operatorname{ad}_{P_\ell}^{q+1}(A)
=[\operatorname{ad}_{P_\ell}^{q}(A),P_\ell]
\). Then
\[
\|\operatorname{ad}_{P_\ell}^{q}(A)\|_F^2
=
\sum_{i,j}|\lambda_{\ell,i}-\lambda_{\ell,j}|^{2q}
|u_{\ell,i}^\top A u_{\ell,j}|^2.
\]
Moreover, the congruence identity extends inductively to
\begin{equation}\label{eq:higher-order-transport-identity}
\operatorname{ad}_{\mathrm{NFM}_\ell}^{\,q}(\mathcal S_\ell)
=
W_\ell^\top\operatorname{ad}_{P_\ell}^{\,q}(Q_\ell)W_\ell.
\end{equation}
Higher order therefore changes the spectral weighting by emphasizing widely
separated modes; it does not create a scale-independent monotone hierarchy.
We do not develop its dynamics here.
\end{remark}

\section{Activation geometry and gate switching}
\label{sec:ActivationGeometry}

The pointwise gate \(D_\ell(x,t)\) and the averaged sensitivity
\(Q_\ell(t)\) are localized relative to the same covariance
\(P_\ell(t)\), but their dynamics are different. The latter is generated by
backward transport and population averaging, whereas
\[
D_\ell(x,t)=\operatorname{diag}\bigl(\sigma'(z_\ell(x,t))\bigr)
\]
is a pointwise nonlinear function of the preactivation. We isolate here the
activation-side dynamics needed below. For smooth activations the gate has a
classical velocity; for ReLU it evolves by switching between diagonal
projections. Neither description contains an intrinsic alignment mechanism.

\subsection{Smooth activation geometry}

Assume first that \(\sigma\in C^2(\mathbb R)\). Along a continuously
differentiable weight trajectory, differentiation of the forward recursion
gives, for \(\ell=0,\ldots,L-2\),
\begin{equation}\label{eq:activation-velocity-recursion}
\partial_ta_0=0,
\qquad
\partial_tz_\ell=\dot W_\ell a_\ell+W_\ell\partial_ta_\ell,
\qquad
\partial_ta_{\ell+1}=D_\ell\partial_tz_\ell.
\end{equation}
Thus gate motion at layer \(\ell\) depends on the local weight velocity and
on the activation dynamics transported from the preceding layers.

\begin{proposition}\label{prop:smooth-gate-evolution}
Let the weight trajectory be continuously differentiable on an interval
\(I\), and suppose that
\(M_\sigma:=\|\sigma''\|_{L^\infty(\mathbb R)}<\infty\). Then, for
\(\ell=0,\ldots,L-2\) and \(t\in I\),
\begin{equation}\label{eq:smooth-gate-evolution}
\partial_tD_\ell(x,t)
=
\operatorname{diag}\bigl(
\sigma''(z_\ell(x,t))\odot\partial_tz_\ell(x,t)
\bigr),
\end{equation}
and
\begin{align}
\|\partial_tD_\ell(x,t)\|_{\mathrm{op}}
&\leq M_\sigma\|\partial_tz_\ell(x,t)\|_\infty,
\label{eq:smooth-gate-op-bound}\\
\|\partial_tD_\ell(x,t)\|_F
&\leq M_\sigma\|\partial_tz_\ell(x,t)\|_2,
\label{eq:smooth-gate-F-bound}\\
\|\partial_tD_\ell(\cdot,t)\|_{\mathrm L^2(\mu;F)}
&\leq
M_\sigma
\|\partial_tz_\ell(\cdot,t)\|_{
\mathrm L^2(\mu;\mathbb R^{m_{\ell+1}})}.
\label{eq:smooth-gate-L2-bound}
\end{align}
\end{proposition}

\begin{proof}
The \(r\)-th diagonal entry of \(D_\ell\) is
\(\sigma'(z_{\ell,r})\), whose derivative is
\(\sigma''(z_{\ell,r})\partial_tz_{\ell,r}\). The pointwise norm bounds
follow from the corresponding \(\ell^\infty\)- and \(\ell^2\)-norms of the
diagonal, and integration gives \eqref{eq:smooth-gate-L2-bound}.
\end{proof}

These estimates control gate variation through preactivation variation; they
have no preferred sign and do not imply decay of a gate-covariance
commutator.

\subsection{ReLU regions and gate switching}

For ReLU, \(\sigma'(r)=\mathbf 1_{\{r>0\}}\). For
\(r=1,\ldots,m_{\ell+1}\), define
\begin{equation}\label{eq:relu-activation-regions}
A_{\ell,r}(t):=\{x:z_{\ell,r}(x,t)>0\},
\qquad
Z_{\ell,r}(t):=\{x:z_{\ell,r}(x,t)=0\}.
\end{equation}
Then
\begin{equation}\label{eq:relu-gate-regions}
D_\ell(x,t)
=
\operatorname{diag}\bigl(
\mathbf 1_{A_{\ell,1}(t)}(x),\ldots,
\mathbf 1_{A_{\ell,m_{\ell+1}}(t)}(x)
\bigr).
\end{equation}
If \(z_{\ell,r}(\cdot,t)\) is continuous, the activation interface
\(\partial A_{\ell,r}(t)\) is contained in \(Z_{\ell,r}(t)\), with equality
under the usual sign-change nondegeneracy.

The next identity expresses the gate geometry in the covariance eigenbasis.

\begin{proposition}\label{prop:activation-region-representation}
Let
\(P_\ell(t)=U_\ell(t)\Lambda_\ell(t)U_\ell(t)^\top\) and define
\(\widetilde D_\ell(x,t)
:=U_\ell(t)^\top D_\ell(x,t)U_\ell(t)\). Then
\begin{align}
(\widetilde D_\ell(x,t))_{ij}
&=
\sum_{r=1}^{m_{\ell+1}}
u_{\ell,ri}(t)u_{\ell,rj}(t)
\mathbf 1_{A_{\ell,r}(t)}(x),
\label{eq:rotated-relu-gate}\\
\mathbb E_\mu\!\left[|(\widetilde D_\ell(x,t))_{ij}|^2\right]
&=
\sum_{r,s=1}^{m_{\ell+1}}
u_{\ell,ri}(t)u_{\ell,rj}(t)u_{\ell,si}(t)u_{\ell,sj}(t)
\mu\bigl(A_{\ell,r}(t)\cap A_{\ell,s}(t)\bigr).
\label{eq:activation-intersection-formula}
\end{align}
\end{proposition}

\begin{proof}
Insert \eqref{eq:relu-gate-regions} into
\(U_\ell^\top D_\ell U_\ell\), then square the resulting scalar expression
and integrate.
\end{proof}

Together with \eqref{eq:common-commutator-expansion}, this shows that the
gate-covariance energy couples spectral separation in \(P_\ell\) to the
overlap geometry of the activation regions. Its time dependence is measured
without differentiating the ReLU gate.

\begin{proposition}\label{prop:relu-switching-distance}
For every \(s,t\),
\begin{equation}\label{eq:relu-switching-distance}
\|D_\ell(\cdot,t)-D_\ell(\cdot,s)\|_{
\mathrm L^2(\mu;F)}^2
=
\sum_{r=1}^{m_{\ell+1}}
\mu\bigl(A_{\ell,r}(t)\triangle A_{\ell,r}(s)\bigr).
\end{equation}
Moreover, if \(z_{\ell,r}(x,t_n)\to z_{\ell,r}(x,t)\) for
\(\mu\)-almost every \(x\) and
\begin{equation}\label{eq:no-mass-on-interface}
\mu\bigl(Z_{\ell,r}(t)\bigr)=0,
\qquad r=1,\ldots,m_{\ell+1},
\end{equation}
then \(D_\ell(\cdot,t_n)\to D_\ell(\cdot,t)\) in
\(\mathrm L^2(\mu;F)\).
\end{proposition}

\begin{proof}
For each \(r\),
\[
\bigl(
\mathbf 1_{A_{\ell,r}(t)}
-
\mathbf 1_{A_{\ell,r}(s)}
\bigr)^2
=
\mathbf 1_{A_{\ell,r}(t)\triangle A_{\ell,r}(s)},
\]
which gives \eqref{eq:relu-switching-distance}. Under
\eqref{eq:no-mass-on-interface}, the signs converge almost everywhere, and
the second assertion follows by dominated convergence.
\end{proof}

Thus the ReLU gate may be continuous in the matrix-valued
\(\mathrm L^2\)-space without being classically differentiable. If
\(z_{\ell,r}\) is \(C^1\) in \((x,t)\) and
\(\nabla_xz_{\ell,r}\neq0\) on the interface, its normal velocity in the
direction \(\nabla_xz_{\ell,r}/\|\nabla_xz_{\ell,r}\|_2\) is
\[
v_{\ell,r}^{\mathrm n}
=
-\frac{\partial_tz_{\ell,r}}
{\|\nabla_xz_{\ell,r}\|_2}.
\]
The symmetric-difference identity provides a robust global formulation that
does not require this regularity.

\subsection{Localized gate-energy variation}

Fix integers \(1\leq k\leq k+\beta<m_{\ell+1}\), \(\beta\geq0\), and set
\begin{equation}\label{eq:localized-gate-block-activation}
B_{D,\ell}^{(k,\beta)}(x,t)
:=
\Pi_\ell^{(k)}(t)D_\ell(x,t)
\Pi_\ell^{(k+\beta),\perp}(t).
\end{equation}
Then
\(\mathcal G_{D,\ell}^{(k,\beta)}
=\|B_{D,\ell}^{(k,\beta)}\|_{
\mathrm L^2(\mu;F)}^2\). For a smooth activation, the common evolution
estimate from Proposition~\ref{prop:localized-energy-evolution} gives the
following specialization.

\begin{corollary}\label{cor:smooth-localized-gate-energy}
Assume the hypotheses of Proposition~\ref{prop:smooth-gate-evolution}, with
\(D_\ell(\cdot,t)\in\mathrm L^2(\mu;F)\) and
\(\partial_tz_\ell(\cdot,t)\in
\mathrm L^2(\mu;\mathbb R^{m_{\ell+1}})\), and suppose that
\(\gamma_\ell^{(k)},\gamma_\ell^{(k+\beta)}>0\) on the interval under
consideration. Then
\begin{align}
\left|
\frac{\d}{\d t}\mathcal G_{D,\ell}^{(k,\beta)}(t)
\right|
&\leq
2\sqrt{\mathcal G_{D,\ell}^{(k,\beta)}(t)}
\Bigl[
M_\sigma
\|\partial_tz_\ell(\cdot,t)\|_{
\mathrm L^2(\mu;\mathbb R^{m_{\ell+1}})}
\notag\\
&\qquad+
\bigl(
\|\dot\Pi_\ell^{(k)}(t)\|_{\mathrm{op}}
+\|\dot\Pi_\ell^{(k+\beta)}(t)\|_{\mathrm{op}}
\bigr)
\|D_\ell(\cdot,t)\|_{
\mathrm L^2(\mu;F)}
\Bigr].
\label{eq:smooth-localized-gate-energy}
\end{align}
Under unregularized gradient flow or gradient flow with weight decay, the
projector contribution inside the brackets is bounded by
\begin{align}
&2C\|W_\ell(t)\|_{\mathrm{op}}\|G_\ell(t)\|_{\mathrm{op}}
\left(
\frac{1}{\gamma_\ell^{(k)}(t)}
+\frac{1}{\gamma_\ell^{(k+\beta)}(t)}
\right)
\|D_\ell(\cdot,t)\|_{\mathrm L^2(\mu;F)}.
\label{eq:smooth-projector-contribution}
\end{align}
\end{corollary}

\begin{proof}
Apply \eqref{eq:localized-energy-evolution} with \(A_\ell=D_\ell\), use
\eqref{eq:smooth-gate-L2-bound}, and then apply
\eqref{eq:projector-velocity-gradient} at the two cutoff indices.
\end{proof}

For ReLU, the corresponding control is finite-time.

\begin{proposition}\label{prop:relu-localized-energy-variation}
Assume that the relevant spectral projectors have been chosen at \(s\) and
\(t\). Then
\begin{align}
\left|
\sqrt{\mathcal G_{D,\ell}^{(k,\beta)}(t)}
-
\sqrt{\mathcal G_{D,\ell}^{(k,\beta)}(s)}
\right|
&\leq
\bigl(
\sum_{r=1}^{m_{\ell+1}}
\mu\bigl(A_{\ell,r}(t)\triangle A_{\ell,r}(s)\bigr)
\bigr)^{1/2}
\notag\\
&\quad+
\sqrt{m_{\ell+1}}
\Bigl(
\|\Pi_\ell^{(k)}(t)-\Pi_\ell^{(k)}(s)\|_{\mathrm{op}}+
\|\Pi_\ell^{(k+\beta)}(t)
-\Pi_\ell^{(k+\beta)}(s)\|_{\mathrm{op}}
\Bigr).
\label{eq:relu-localized-energy-variation}
\end{align}
If, in addition, the weights follow either gradient flow considered above and
the relevant gaps remain positive on \([s,t]\), the projector differences are
controlled by \eqref{eq:projector-stability}.
\end{proposition}

\begin{proof}
The reverse triangle inequality and a three-term telescoping decomposition
give
\begin{align*}
\bigl|
\|B_{D,\ell}^{(k,\beta)}(t)\|_{
\mathrm L^2(\mu;F)}
-
\|B_{D,\ell}^{(k,\beta)}(s)\|_{
\mathrm L^2(\mu;F)}
\bigr|
&\leq
\|D_\ell(t)-D_\ell(s)\|_{
\mathrm L^2(\mu;F)}
\\
&\quad+
\sqrt{m_{\ell+1}}
\|\Pi_\ell^{(k)}(t)-\Pi_\ell^{(k)}(s)\|_{\mathrm{op}}
\\
&\quad+
\sqrt{m_{\ell+1}}
\|\Pi_\ell^{(k+\beta)}(t)
-\Pi_\ell^{(k+\beta)}(s)\|_{\mathrm{op}},
\end{align*}
because
\(\|D_\ell(\cdot,\tau)\|_{\mathrm L^2(\mu;F)}^2
=\sum_r\mu(A_{\ell,r}(\tau))\leq m_{\ell+1}\).
Now use Proposition~\ref{prop:relu-switching-distance}.
\end{proof}

The smooth and ReLU estimates separate activation motion from covariance
subspace motion. They control variation, not its sign. Hence convergence of
the projectors and stabilization of the gates imply convergence of the
localized energies, but do not force their limits to vanish; alignment
requires an additional mechanism controlling the limiting cross blocks.

\section{Stabilization and conditional propagation of spectral alignment}
\label{sec:AlignmentDynamics}

The preceding sections introduced activation, internal sensitivity, and
feature-side alignment and identified the exact relations between them. These
relations do not form an unconditional decay chain. The gate commutator is one
local source in the recursion for the internal commutator, while the
feature-side commutator is a singular-value-weighted congruence image of the
combined internal geometry. We now distinguish stabilization of the localized
energies from decay to zero and give sufficient conditions under which
alignment propagates between the three levels.

Throughout the section, fix a hidden layer \(0\leq\ell\leq L-2\) and integers
\[
1\leq k\leq k+\beta<m_{\ell+1},
\qquad \beta\geq0.
\]
We use the buffered energies
\(\mathcal G_{D,\ell}^{(k,\beta)}\) and
\(\mathcal G_{Q,\ell}^{(k,\beta)}\) from
\eqref{eq:localized-gate-energy} and
\eqref{eq:localized-sensitivity-energy}; when \(\beta=0\), we omit \(\beta\)
from the superscript.

\subsection{Stabilization and limiting compatibility}

For smooth activations, Corollary~\ref{cor:smooth-localized-gate-energy}
implies
\begin{equation}\label{eq:sqrt-energy-stability}
\left|
\sqrt{\mathcal G_{D,\ell}^{(k,\beta)}(t)}
-
\sqrt{\mathcal G_{D,\ell}^{(k,\beta)}(s)}
\right|
\leq
\int_s^t b_{D,\ell}^{(k,\beta)}(\tau)\,\d\tau,
\end{equation}
where
\begin{align}
b_{D,\ell}^{(k,\beta)}
&:=
\|\partial_tD_\ell\|_{\mathrm L^2(\mu;F)}
+
\bigl(
\|\dot\Pi_\ell^{(k)}\|_{\mathrm{op}}
+
\|\dot\Pi_\ell^{(k+\beta)}\|_{\mathrm{op}}
\bigr)
\|D_\ell\|_{\mathrm L^2(\mu;F)}.
\label{eq:alignment-variation-coefficient}
\end{align}
Thus \(b_{D,\ell}^{(k,\beta)}\in\mathrm L^1(0,\infty)\) implies convergence
of the localized gate energy. For ReLU, the same conclusion follows from
Proposition~\ref{prop:relu-localized-energy-variation} if the two covariance
projectors converge and \(D_\ell(\cdot,t)\) is Cauchy in
\(\mathrm L^2(\mu;F)\). By
Proposition~\ref{prop:relu-switching-distance}, the latter condition is
equivalent to
\[
\lim_{T\to\infty}
\sup_{s,t\geq T}
\sum_{r=1}^{m_{\ell+1}}
\mu\bigl(A_{\ell,r}(t)\triangle A_{\ell,r}(s)\bigr)
=0.
\]
Likewise,
\(\mathcal G_{Q,\ell}^{(k,\beta)}(t)\) converges whenever \(Q_\ell(t)\) is
Cauchy in Frobenius norm and the two projectors are Cauchy in operator norm.
These assertions control late-time variation, not the value of the limiting
cross block.

\begin{theorem}\label{thm:limiting-spectral-compatibility}
Assume that
\[
\Pi_\ell^{(k)}(t)\longrightarrow\Pi_\ell^{(k),\infty},
\qquad
\Pi_\ell^{(k+\beta)}(t)
\longrightarrow\Pi_\ell^{(k+\beta),\infty}
\quad\text{in operator norm}.
\]
If
\(D_\ell(\cdot,t)\to D_\ell^\infty(\cdot)\) in
\(\mathrm L^2(\mu;F)\), then
\begin{align}
\mathcal G_{D,\ell}^{(k,\beta)}(t)
&\longrightarrow
\left\|
\Pi_\ell^{(k),\infty}D_\ell^\infty(\cdot)
\Pi_\ell^{(k+\beta),\infty,\perp}
\right\|_{\mathrm L^2(\mu;F)}^2.
\label{eq:limiting-gate-cross-block}
\end{align}
If \(Q_\ell(t)\to Q_\ell^\infty\) in Frobenius norm, then
\begin{align}
\mathcal G_{Q,\ell}^{(k,\beta)}(t)
&\longrightarrow
\left\|
\Pi_\ell^{(k),\infty}Q_\ell^\infty
\Pi_\ell^{(k+\beta),\infty,\perp}
\right\|_F^2.
\label{eq:limiting-sensitivity-cross-block}
\end{align}
Hence either localized energy converges to zero if and only if its limiting
buffered cross block vanishes. For \(\beta=0\), these conditions are
equivalent, respectively, to
\begin{equation}\label{eq:limiting-unbuffered-compatibility}
[D_\ell^\infty(x),\Pi_\ell^{(k),\infty}]=0
\quad\text{for \(\mu\)-a.e. \(x\)},
\qquad
[Q_\ell^\infty,\Pi_\ell^{(k),\infty}]=0.
\end{equation}
\end{theorem}

\begin{proof}
A three-term telescoping decomposition and continuity of matrix
multiplication give
\[
\Pi_\ell^{(k)}(t)D_\ell(\cdot,t)
\Pi_\ell^{(k+\beta),\perp}(t)
\longrightarrow
\Pi_\ell^{(k),\infty}D_\ell^\infty(\cdot)
\Pi_\ell^{(k+\beta),\infty,\perp}
\]
in \(\mathrm L^2(\mu;F)\); the argument for \(Q_\ell\) is identical. When
\(\beta=0\), symmetry shows that vanishing of one cross block is equivalent
to block diagonality relative to \(\Pi_\ell^{(k),\infty}\), hence to the
corresponding commutation relation.
\end{proof}

The theorem separates stabilization from alignment. Stabilization means that
the cross block approaches a limit; alignment is the additional assertion
that this limit vanishes. No convergence estimate above supplies that
compatibility by itself.

\subsection{A dissipation criterion for localized decay}

The following abstract criterion applies to either the gate or the averaged
sensitivity localization. It is stated separately from the network-specific
problem of proving the required coercivity estimate.

\begin{theorem}\label{thm:dissipation-driven-alignment}
Let \(\mathscr G:[0,\infty)\to[0,\infty)\) be one of
\(\mathcal G_{D,\ell}^{(k,\beta)}\) or
\(\mathcal G_{Q,\ell}^{(k,\beta)}\), and assume that \(\mathscr G\) is
uniformly continuous. Suppose that \(W(t)\) follows gradient flow,
\(\dot W=-\nabla\mathcal R(W)\), that \(\mathcal R\) is bounded below along
the trajectory, and set
\[
\|\nabla\mathcal R(W)\|_F^2
:=
\sum_{j=0}^{L-1}
\|\nabla_{W_j}\mathcal R(W)\|_F^2.
\]
Assume that, for some \(c>0\) and nonnegative
\(r\in\mathrm L^1(0,\infty)\),
\begin{equation}\label{eq:alignment-coercivity}
\mathscr G(t)\leq c\|\nabla\mathcal R(W(t))\|_F^2+r(t)
\end{equation}
for almost every \(t\geq0\). Then \(\mathscr G(t)\to0\) as
\(t\to\infty\).
\end{theorem}

\begin{proof}
The dissipation identity and the lower bound for the risk give
\[
\int_0^\infty\|\nabla\mathcal R(W(t))\|_F^2\,\d t<\infty.
\]
Integrating \eqref{eq:alignment-coercivity} yields
\(\mathscr G\in\mathrm L^1(0,\infty)\). A nonnegative, uniformly continuous,
integrable function on \([0,\infty)\) converges to zero, which proves the
claim.
\end{proof}

\begin{remark}\label{rem:alignment-coercivity-not-automatic}
For smooth activations, uniform bounds on
\(\|D_\ell(\cdot,t)\|_{\mathrm L^2(\mu;F)}\) and on the variation
coefficient in \eqref{eq:alignment-variation-coefficient} imply uniform
continuity of the localized gate energy; uniform cutoff gaps and the bounds
entering \eqref{eq:smooth-projector-contribution} provide one way to control
the projector part of that coefficient. For ReLU
networks, the corresponding uniform continuity may instead be verified using
the switching and projector moduli in
\eqref{eq:relu-localized-energy-variation}, while the sensitivity energy can
be treated directly through the deterministic operator \(Q_\ell(t)\). The
substantive assumption in
Theorem~\ref{thm:dissipation-driven-alignment} is therefore the coercive
estimate~\eqref{eq:alignment-coercivity}. In the ReLU case, boundedness of the
gate operators implies that
\(\mathcal G_{D,\ell}^{(k,\beta)}\) is uniformly bounded, but it does not imply
this estimate. Indeed, the loss gradient contains the gates only through
residual-weighted and backward-transported expressions and may be small
because of interpolation, cancellation, weak backward sensitivity, or small
activations, while the gate-covariance incompatibility remains nonzero. In
particular, a stationary point of the risk may satisfy
\[
\nabla\mathcal R(W)=0
\qquad\text{and}\qquad
\mathcal G_{D,\ell}^{(k,\beta)}>0.
\]
Thus \eqref{eq:alignment-coercivity} is a genuine geometric hypothesis
asserting that persistent mixing across the selected covariance blocks must
remain visible through loss dissipation, up to an integrable transient. It is
not a general consequence of gradient flow or of bounded gate operators and
may fail even when the risk converges. Accordingly,
Theorem~\ref{thm:dissipation-driven-alignment} provides a conditional
dissipation-driven mechanism rather than a universal alignment theorem.
\end{remark}

\subsection{From localized to global alignment}

The passage from all unbuffered localized energies to the global commutator
requires only bounded covariance scale; the converse additionally requires
spectral separation.

\begin{proposition}\label{prop:localized-to-global-comparison}
At a fixed time, let \(A_\ell(x)\) be symmetric and square-integrable,
possibly independent of \(x\), and set
\begin{align*}
\mathcal E_{A,\ell}^{(k)}
&:=
\int\|\Pi_\ell^{(k)}A_\ell
\Pi_\ell^{(k),\perp}\|_F^2\,\d\mu,\qquad \mathcal C_{A,\ell}
:=
\int\|[A_\ell,P_\ell]\|_F^2\,\d\mu,
\end{align*}
where the integrals are omitted when \(A_\ell\) is deterministic. If
\(n=m_{\ell+1}\) and
\(\underline\gamma_\ell:=\min_{1\leq k<n}\gamma_\ell^{(k)}>0\), then
\begin{equation}\label{eq:localized-global-two-sided}
2\underline\gamma_\ell^2
\sum_{k=1}^{n-1}\mathcal E_{A,\ell}^{(k)}
\leq
\mathcal C_{A,\ell}
\leq
2\|P_\ell\|_{\mathrm{op}}^2
\sum_{k=1}^{n-1}\mathcal E_{A,\ell}^{(k)}.
\end{equation}
\end{proposition}

\begin{proof}
In the eigenbasis of \(P_\ell\), writing
\(a_{ij}=u_{\ell,i}^\top A_\ell u_{\ell,j}\), one has
\begin{align*}
\sum_{k=1}^{n-1}\mathcal E_{A,\ell}^{(k)}
&=
\int\sum_{1\leq i<j\leq n}(j-i)|a_{ij}|^2\,\d\mu,
&
\mathcal C_{A,\ell}
&=
2\int\sum_{1\leq i<j\leq n}
(\lambda_{\ell,i}-\lambda_{\ell,j})^2|a_{ij}|^2\,\d\mu.
\end{align*}
Now use
\(|\lambda_{\ell,i}-\lambda_{\ell,j}|\leq\|P_\ell\|_{\mathrm{op}}\) and
\(\lambda_{\ell,i}-\lambda_{\ell,j}
\geq(j-i)\underline\gamma_\ell\), together with
\((j-i)^2\geq j-i\).
\end{proof}

For \(A_\ell=D_\ell\), the global quantity is \(\mathcal G_\ell\); for
\(A_\ell=Q_\ell\), it is \(\mathcal G_{Q,\ell}\). Hence decay of every
unbuffered localized energy implies global decay if
\(\|P_\ell(t)\|_{\mathrm{op}}\) remains bounded. Conversely, global decay
implies decay at every uniformly separated interface by
\eqref{eq:localized-gate-gap-bound} or
\eqref{eq:localized-sensitivity-gap-bound}. Without the scale and gap
hypotheses, the corresponding implications need not hold.

\subsection{Propagation through the internal hierarchy}

Gate alignment removes only the gate source in the internal recursion.
Transported incompatibility, adjacent-layer imbalance, and sensitivity
fluctuations remain and must be controlled separately.

\begin{proposition}\label{prop:propagation-of-internal-alignment}
Suppose that \(\sigma\) is ReLU, that
\[
\sup_{t\geq0}\|W_j(t)\|_{\mathrm{op}}<\infty,
\qquad j=0,\ldots,L-1,
\]
and that, for \(\ell=0,\ldots,L-2\),
\begin{align}
\sup_{t\geq0}\|\mathcal S_{\ell+1}(t)\|_F
&<\infty,\qquad\sup_{t\geq0}\mathbb E_\mu
\|S_{\ell+1}(t)\|_{\mathrm{op}}^2
<\infty.
\label{eq:propagation-sensitivity-bounds}
\end{align}
Assume further that, for the same indices,
\begin{equation}\label{eq:propagation-source-decay}
\|\Delta_\ell(t)\|_{\mathrm{op}}\longrightarrow0,
\qquad
\mathcal V_{\ell+1}(t)\longrightarrow0,
\qquad
\mathcal G_\ell(t)\longrightarrow0.
\end{equation}
Then
\begin{equation}\label{eq:internal-alignment-conclusion}
\mathcal G_{Q,\ell}(t)
=\|C_\ell^{\mathrm{int}}(t)\|_F^2
\longrightarrow0
\end{equation}
for every hidden layer \(\ell\).
\end{proposition}

\begin{proof}
At the terminal internal level,
\(C_{L-1}^{\mathrm{int}}=[I_{m_L},P_{L-1}]=0\). If
\(\mathcal G_{Q,\ell+1}(t)\to0\), the ReLU estimate
\eqref{ReLUCommutatorEstimate} gives
\begin{align*}
\mathcal G_{Q,\ell}
&\leq
4\|W_{\ell+1}\|_{\mathrm{op}}^4\mathcal G_{Q,\ell+1}
+16\|\Delta_\ell\|_{\mathrm{op}}^2
\|\mathcal S_{\ell+1}\|_F^2
+16\|W_\ell\|_{\mathrm{op}}^4\mathcal V_{\ell+1}
+
16\mathcal G_\ell
\mathbb E_\mu\|S_{\ell+1}\|_{\mathrm{op}}^2.
\end{align*}
Every term on the right tends to zero, and backward induction proves the
claim.
\end{proof}

The assumptions in \eqref{eq:propagation-source-decay} are independent
requirements, not consequences of gate alignment. In particular,
\(\mathcal V_{L-1}=0\), but fluctuation decay at earlier layers is a genuine
additional hypothesis; likewise, unregularized gradient flow may only make
\(\Delta_\ell(t)\) converge to a nonzero limit, as shown by
Corollary~\ref{cor:ImbalanceLongTime}. If a buffered gap remains uniformly
positive, \eqref{eq:localized-sensitivity-gap-bound} further gives
\begin{equation}\label{eq:propagated-buffered-sensitivity}
\mathcal G_{Q,\ell}^{(k,\beta)}(t)
\leq
\frac{\mathcal G_{Q,\ell}(t)}
{2(\gamma_\ell^{(k,\beta)}(t))^2}
\longrightarrow0.
\end{equation}

\subsection{Consequences for feature geometry}

\begin{corollary}\label{cor:feature-side-alignment}
If \(\mathcal G_{Q,\ell}(t)\to0\) and
\(\sup_{t\geq0}\|W_\ell(t)\|_{\mathrm{op}}<\infty\), then
\[
\|[\mathcal S_\ell(t),\mathrm{NFM}_\ell(t)]\|_F
\longrightarrow0.
\]
If, in addition, \(\mathcal S_\ell(t)\to\mathcal S_\ell^\infty\) and
\(\mathrm{NFM}_\ell(t)\to\mathrm{NFM}_\ell^\infty\), then
\([\mathcal S_\ell^\infty,\mathrm{NFM}_\ell^\infty]=0\).
\end{corollary}

\begin{proof}
By Lemma~\ref{lem:commutator-transfer},
\[
\|[\mathcal S_\ell,\mathrm{NFM}_\ell]\|_F
\leq
\|W_\ell\|_{\mathrm{op}}^2
\sqrt{\mathcal G_{Q,\ell}}.
\]
The first conclusion follows, and the second is immediate from continuity of
the commutator.
\end{proof}

The converse to the conclusion of Corollary \ref{cor:feature-side-alignment} need not hold: the feature-side commutator may be small because
internal directions are suppressed by small or clustered singular values or
lost through \(\ker W_\ell^\top\). Thus the section yields only conditional
propagation. Decay of all unbuffered localized gate energies gives global gate
alignment under a covariance-scale bound; global gate alignment gives
internal alignment only when imbalance and fluctuations also vanish; and
internal alignment gives feature-side alignment under bounded transport. None
of these hypotheses is asserted to hold universally along training
trajectories.

\section{Conditional coercivity and Lyapunov principles}
\label{sec:lyapunov-alignment}

The results in this section characterize sufficient mechanisms for decay; they
do not establish that these mechanisms arise generically from neural-network
gradient flow. The risk-dissipation identity governs only the scalar risk.
Even if one additionally assumes a Polyak--Lojasiewicz (PL)
inequality,\footnote{In the notation introduced below, the PL inequality is
the condition
\[
\mathscr D(t)\geq2\mu\mathscr E(t),
\qquad \mu>0,
\]
where \(\mathscr E\) is the excess risk and \(\mathscr D\) is the squared
gradient norm. Combined with the risk-dissipation identity
\(\dot{\mathscr E}=-\mathscr D\), it yields
\(\mathscr E(t)\leq e^{-2\mu t}\mathscr E(0)\). The PL condition does not
require convexity, but it controls the objective value rather than the
geometric alignment defects.}
the resulting exponential decay concerns the excess risk, not the localized
defects. We therefore consider two additional coercive relations. The first is
static: a geometric error bound
controls the selected defects by the excess risk and hence, through the PL
inequality, by the risk dissipation. The second is dynamic: the sensitivity
evolution itself damps selected cross blocks relative to the covariance
splitting. Both mechanisms are conditional. Their purpose is to identify
alternative geometric coercivity conditions under which selective alignment
can be deduced from optimization.

The static principle applies equally to the localized gate and sensitivity
defects,
\[
\mathcal G_{D,\ell}^{(k,\beta)}
\quad\text{and}\quad
\mathcal G_{Q,\ell}^{(k,\beta)}.
\]
The dynamic principle is formulated for the sensitivity energies. This
distinction is dynamical rather than geometric: on intervals of smooth
evolution, \(Q_\ell(t)\) admits an ordinary evolution equation, whereas the
ReLU gate \(D_\ell(x,t)\) changes through sample-dependent switching events.
Switching estimates may control variation of the gate energies, but do not by
themselves provide a sign-definite transverse term.

Let \(\mathcal R\) denote the population or empirical risk and set
\begin{equation}\label{eq:lyap-excess-dissipation}
\mathscr E(t)
:=
\mathcal R(W(t))-\mathcal R_*,
\qquad
\mathcal R_*:=\inf_W\mathcal R(W),
\qquad
\mathscr D(t)
:=
\sum_{j=0}^{L-1}
\|\nabla_{W_j}\mathcal R(W(t))\|_F^2.
\end{equation}
Along gradient flow,
\begin{equation}\label{eq:lyap-risk-dissipation}
\dot W_j=-\nabla_{W_j}\mathcal R(W),
\qquad
\dot{\mathscr E}(t)=-\mathscr D(t).
\end{equation}
Fix a finite set
\[
\mathcal I
\subset
\left\{
(\ell,k,\beta):
0\leq\ell\leq L-2,
\quad
1\leq k\leq k+\beta<m_{\ell+1},
\quad
\beta\in\mathbb N_0
\right\}.
\]
For \(i=(\ell,k,\beta)\in\mathcal I\), write
\begin{align}
\Pi_i(t)
&:=
\Pi_\ell^{(k)}(t),\quad \widehat\Pi_i(t)
:=
\Pi_\ell^{(k+\beta)}(t),\quad
\widehat\Pi_i^\perp(t)
:=
I-\widehat\Pi_i(t),
\label{eq:lyap-buffered-projectors}\\
B_i(t)
&:=
\Pi_i(t)Q_\ell(t)\widehat\Pi_i^\perp(t),
\label{eq:lyap-buffered-block}\\
q_i(t)
&:=
\|B_i(t)\|_F^2
=
\mathcal G_{Q,\ell}^{(k,\beta)}(t).
\label{eq:lyap-buffered-defect}
\end{align}
Here \(Q_\ell=\mathbb E_\mu[\mathsf Q_\ell]\) is the deterministic averaged
sensitivity operator. The quantity \(q_i\) measures direct transfer from the
leading \(k\) covariance directions to the complement of the leading
\(k+\beta\) directions, leaving the buffered modes unpenalized.

\begin{remark}
Since the projectors in \eqref{eq:lyap-buffered-projectors} are spectral
projectors of \(P_\ell\),
\[
[B_i,P_\ell]
=
\Pi_i[Q_\ell,P_\ell]\widehat\Pi_i^\perp.
\]
Thus the localized commutator energy
\[
\widehat q_i
:=
\left\|
\Pi_i[Q_\ell,P_\ell]\widehat\Pi_i^\perp
\right\|_F^2
\]
satisfies
\[
\gamma_i^2q_i
\leq
\widehat q_i
\leq
\Gamma_i^2q_i,
\qquad
\gamma_i
:=
\gamma_\ell^{(k,\beta)},
\qquad
\Gamma_i
:=
\lambda_{\ell,1}-\lambda_{\ell,m_{\ell+1}},
\]
where \(\gamma_i\) is the buffered spectral gap. Under a uniform positive
lower bound on \(\gamma_i\) and a uniform upper bound on \(\Gamma_i\), decay
of \(q_i\) is equivalent to decay of \(\widehat q_i\). We focus on \(q_i\)
because it measures transverse
mixing independently of spectral scale; differentiating \(\widehat q_i\)
would also introduce \([B_i,\dot P_\ell]\).
\end{remark}

\subsection{Variation and risk-based transfer}

We first specialize Proposition~\ref{prop:localized-energy-evolution} to the
selected sensitivity blocks.

\begin{proposition}\label{prop:lyap-modewise-production}
Fix \(i=(\ell,k,\beta)\in\mathcal I\), and suppose that \(Q_\ell\),
\(\Pi_i\), and \(\widehat\Pi_i\) are absolutely continuous on the interval
under consideration. Assume that
\begin{align}
\|Q_\ell(t)\|_F
&\leq M_i^Q,
\label{eq:lyap-Q-bound}\\
\|\dot\Pi_i(t)\|_{\mathrm{op}}
&\leq H_i^-\sqrt{\mathscr D(t)},\quad\|\dot{\widehat\Pi}_i(t)\|_{\mathrm{op}}
\leq H_i^+\sqrt{\mathscr D(t)},
\label{eq:lyap-projector-production}\\
\|\Pi_i(t)\dot Q_\ell(t)\widehat\Pi_i^\perp(t)\|_F
&\leq L_i^Q\sqrt{\mathscr D(t)}.
\label{eq:lyap-Q-production}
\end{align}
Then, almost everywhere,
\begin{equation}\label{eq:lyap-modewise-production}
|\dot q_i(t)|
\leq
c_i\sqrt{q_i(t)}\sqrt{\mathscr D(t)},
\qquad
c_i
:=
2\bigl(L_i^Q+M_i^Q(H_i^-+H_i^+)\bigr).
\end{equation}
\end{proposition}

\begin{proof}
Differentiation gives
\[
\dot B_i
=
\dot\Pi_iQ_\ell\widehat\Pi_i^\perp
+
\Pi_i\dot Q_\ell\widehat\Pi_i^\perp
-
\Pi_iQ_\ell\dot{\widehat\Pi}_i.
\]
The assumptions yield
\[
\|\dot B_i\|_F
\leq
\bigl(L_i^Q+M_i^Q(H_i^-+H_i^+)\bigr)
\sqrt{\mathscr D}.
\]
Since \(\dot q_i=2\langle B_i,\dot B_i\rangle_F\), the result follows.
\end{proof}

\begin{remark} Along the gradient flow \eqref{eq:lyap-risk-dissipation}, if
\(\|W_\ell\|_{\mathrm{op}}\leq M_\ell\) and the two cutoff gaps are
bounded below by \(\underline\gamma_i^->0\) and
\(\underline\gamma_i^+>0\), Proposition~\ref{prop:projector-velocity}
permits the choices
\begin{equation}\label{eq:lyap-explicit-projector-constants}
H_i^-
=
\frac{2CM_\ell}{\underline\gamma_i^-},
\qquad
H_i^+
=
\frac{2CM_\ell}{\underline\gamma_i^+},
\end{equation}
where \(C\) is the constant in
\eqref{eq:projector-velocity-gradient}; here
\(\|G_\ell\|_{\mathrm{op}}\leq\sqrt{\mathscr D}\). Thus the production estimate
deteriorates as either boundary gap closes. Its remaining substantive input is
the transverse derivative bound \eqref{eq:lyap-Q-production}. For ReLU, the
argument applies on intervals of absolute continuity; a global statement
requires a bounded-variation or switching remainder.
\end{remark}

The simplest risk-transfer principle through a geometric error bound is immediate.

\begin{theorem}
\label{thm:production-controlled-lyapunov}
Assume that, along the gradient-flow trajectory,
\begin{equation}\label{eq:lyap-PL}
\mathscr D(t)
\geq
2\mu\mathscr E(t),
\qquad
\mu>0,
\end{equation}
and, for every \(i\in\mathcal I\),
\begin{equation}\label{eq:lyap-geometric-error-bound}
q_i(t)
\leq
K_i\mathscr E(t),
\qquad
K_i>0.
\end{equation}
Then
\begin{equation}\label{eq:lyap-risk-transfer-decay}
\mathscr E(t)
\leq
e^{-2\mu t}\mathscr E(0),
\qquad
q_i(t)
\leq
K_i e^{-2\mu t}\mathscr E(0),
\quad
i\in\mathcal I.
\end{equation}
The same conclusion holds with \(q_i\) replaced by a localized gate defect
whenever the corresponding error bound is available.
\end{theorem}

\begin{proof}
Equations \eqref{eq:lyap-risk-dissipation} and \eqref{eq:lyap-PL} give
\(\dot{\mathscr E}\leq-2\mu\mathscr E\). The conclusion follows from
Gronwall's inequality and \eqref{eq:lyap-geometric-error-bound}.
\end{proof}

The alignment conclusion comes from the error bound, not from the production
estimate. The latter becomes useful when geometric variation and remainders
must be absorbed into a single functional. Indeed, suppose in addition that
the absolutely continuous defects satisfy, almost everywhere,
\begin{equation}\label{eq:lyap-production-with-remainder}
\dot q_i
\leq
c_i\sqrt{q_i}\sqrt{\mathscr D}+r_i,
\qquad
r_i\geq0,
\qquad
r_i\in\mathrm L^1_{\mathrm{loc}},
\end{equation}
and choose \(A>0\), \(\omega_i>0\), and \(\rho\in(0,1)\) so that
\begin{equation}\label{eq:lyap-risk-absorption-condition}
\sum_{i\in\mathcal I}
\omega_i c_i\sqrt{\frac{K_i}{2\mu}}
\leq
A\rho.
\end{equation}
Then, for
\[
\mathscr V_{\mathrm{risk}}
:=
A\mathscr E+
\sum_{i\in\mathcal I}\omega_iq_i,
\]
the PL inequality and \eqref{eq:lyap-geometric-error-bound} give
\begin{equation}\label{eq:lyap-strict-dissipation}
\dot{\mathscr V}_{\mathrm{risk}}
\leq
-A(1-\rho)\mathscr D
+
\sum_{i\in\mathcal I}\omega_ir_i.
\end{equation}
This estimate organizes production and remainders but does not strengthen the
decay already contained in \eqref{eq:lyap-risk-transfer-decay}.

The geometric error bound is restrictive. Let
\[
\mathcal M
:=
\{W:\mathcal R(W)=\mathcal R_*\}
\]
be a local minimizing set. Suppose that, in a tubular neighborhood \(U\) of a
regular component of \(\mathcal M\),
\begin{align}
\mathcal R(W)-\mathcal R_*
&\geq
c\,\operatorname{dist}(W,\mathcal M)^2,
\label{eq:lyap-quadratic-growth}\\
B_i(\overline W)
&=0
\quad
\text{for every }\overline W\in\mathcal M\cap U,
\label{eq:lyap-aligned-minimizers}\\
\|B_i(W)-B_i(\overline W)\|_F
&\leq
L_i\|W-\overline W\|
\quad
\text{for }W,\overline W\in U.
\label{eq:lyap-block-lipschitz}
\end{align}
Choosing a nearest \(\overline W\in\mathcal M\) gives
\begin{equation}\label{eq:lyap-error-bound-sufficient}
q_i(W)
\leq
L_i^2\operatorname{dist}(W,\mathcal M)^2
\leq
\frac{L_i^2}{c}
\bigl(\mathcal R(W)-\mathcal R_*\bigr).
\end{equation}
Persistent cutoff gaps and local regularity of \(Q_\ell(W)\) provide the
required regularity of \(B_i\). The strong condition is
\eqref{eq:lyap-aligned-minimizers}: the selected block must vanish on the
entire minimizing component under consideration.

For empirical squared loss, let
\[
r_N(W)
:=
\bigl(
F(x^{(1)},W)-y^{(1)},\ldots,
F(x^{(N)},W)-y^{(N)}
\bigr).
\]
Then
\[
\mathcal R_N(W)
=
\frac{1}{2N}\|r_N(W)\|_2^2,
\]
and the interpolation set is
\[
\mathcal M_N
:=
\{W:r_N(W)=0\}
=
\{W:\mathcal R_N(W)=0\}.
\]
Near a regular interpolating point, uniform nondegeneracy of the Jacobian
\(J_N=D_Wr_N\) in directions normal to \(\mathcal M_N\) yields local quadratic
growth. A uniform lower bound on the relevant positive spectrum of the
empirical Gram operator \(N^{-1}J_NJ_N^\top\) yields a local PL inequality.
Neither condition implies \(B_i=0\) on \(\mathcal M_N\). Risk-based alignment is
therefore available only near minimizing components on which interpolation
and the selected spectral compatibility condition coexist.

\subsection{Intrinsic transverse damping}
\label{subsec:lyap-intrinsic-damping}

The preceding mechanism explains alignment only when the defect is already
tied to the risk by \eqref{eq:lyap-geometric-error-bound}. We now assume
instead that the sensitivity evolution damps the selected cross block.

\begin{proposition}\label{prop:lyap-damped-buffered-defect}
Fix \(i=(\ell,k,\beta)\in\mathcal I\). In addition to
\eqref{eq:lyap-Q-bound} and \eqref{eq:lyap-projector-production}, suppose that
there exist \(\kappa_i>0\) and \(L_i^Q\geq0\) such that
\begin{equation}\label{eq:lyap-transverse-damping}
\left\|
\Pi_i\dot Q_\ell\widehat\Pi_i^\perp
+
\kappa_iB_i
\right\|_F
\leq
L_i^Q\sqrt{\mathscr D}.
\end{equation}
Then, almost everywhere,
\begin{equation}\label{eq:lyap-damped-defect}
\dot q_i
\leq
-2\kappa_iq_i
+
c_i\sqrt{q_i}\sqrt{\mathscr D},
\qquad
c_i
:=
2\bigl(L_i^Q+M_i^Q(H_i^-+H_i^+)\bigr).
\end{equation}
\end{proposition}

\begin{proof}
Set
\[
Z_i
:=
\Pi_i\dot Q_\ell\widehat\Pi_i^\perp+
\kappa_iB_i.
\]
Then
\[
\dot B_i
=
-\kappa_iB_i+Z_i
+
\dot\Pi_iQ_\ell\widehat\Pi_i^\perp
-
\Pi_iQ_\ell\dot{\widehat\Pi}_i.
\]
The assumptions bound \(\|\dot B_i+\kappa_iB_i\|_F\) by
\[
\bigl(L_i^Q+M_i^Q(H_i^-+H_i^+)\bigr)\sqrt{\mathscr D}.
\]
Taking the Frobenius inner product with \(2B_i\) proves the result.
\end{proof}

The assumption has a direct operator interpretation. Define
\[
\mathfrak T_i(Q)
:=
\Pi_iQ\widehat\Pi_i^\perp
+
\widehat\Pi_i^\perp Q\Pi_i.
\]
Since \(\operatorname{Ran}\Pi_i\subseteq
\operatorname{Ran}\widehat\Pi_i\),
\[
\Pi_i\mathfrak T_i(Q)\widehat\Pi_i^\perp
=
\Pi_iQ\widehat\Pi_i^\perp.
\]
Hence a decomposition
\begin{equation}\label{eq:lyap-transverse-operator-model}
\dot Q_\ell
=
-\kappa_i\mathfrak T_i(Q_\ell)
+
\mathcal F_i,
\end{equation}
together with
\begin{equation}\label{eq:lyap-effective-transverse-forcing}
\left\|
\Pi_i\mathcal F_i\widehat\Pi_i^\perp
\right\|_F
\leq
L_i^Q\sqrt{\mathscr D},
\end{equation}
implies \eqref{eq:lyap-transverse-damping}. The motion of the projectors is
then accounted for by the terms \(H_i^-\) and \(H_i^+\) in
\eqref{eq:lyap-damped-defect}. Equivalently,
\[
\dot B_i
=
-\kappa_iB_i+\mathcal R_i,
\]
where
\[
\mathcal R_i
:=
\Pi_i\mathcal F_i\widehat\Pi_i^\perp
+
\dot\Pi_iQ_\ell\widehat\Pi_i^\perp
-
\Pi_iQ_\ell\dot{\widehat\Pi}_i.
\]

When \(\beta=0\),
\[
\mathfrak T_i(Q)
=
\Pi_iQ\Pi_i^\perp+
\Pi_i^\perp Q\Pi_i
=
[\Pi_i,[\Pi_i,Q]],
\]
so the damping acts on the full off-diagonal part of \(Q\) relative to
\(\operatorname{Ran}\Pi_i\oplus\operatorname{Ran}\Pi_i^\perp\). For
\(\beta>0\), it suppresses only the direct coupling between
\(\operatorname{Ran}\Pi_i\) and
\(\operatorname{Ran}\widehat\Pi_i^\perp\); mixing through the buffered modes
is left unrestricted.

This representation interprets, but does not derive, the damping assumption.
Indeed, differentiability of \(W\mapsto Q_\ell(W)\) gives along gradient
flow
\[
\dot Q_\ell
=
-DQ_\ell(W)\bigl[\nabla\mathcal R(W)\bigr],
\]
and a uniform derivative bound may yield
\[
\|\dot Q_\ell\|_F
\lesssim
\|\nabla\mathcal R(W)\|_F
=
\sqrt{\mathscr D}.
\]
This controls the speed of \(Q_\ell\), but supplies neither a sign nor the
negative transverse term in \eqref{eq:lyap-transverse-operator-model}.
Intrinsic damping is therefore a genuine, model-dependent dynamical
coercivity hypothesis.

\begin{theorem}
\label{thm:lyap-intrinsic-damping}
Assume \eqref{eq:lyap-risk-dissipation}, the PL inequality
\eqref{eq:lyap-PL}, and \eqref{eq:lyap-damped-defect} for every
\(i\in\mathcal I\). For \(A>0\) and \(\omega_i>0\), define
\begin{equation}\label{eq:lyap-damped-functional}
\mathscr V_{\mathrm{damp}}(t)
:=
A\mathscr E(t)
+
\sum_{i\in\mathcal I}\omega_iq_i(t).
\end{equation}
If
\begin{equation}\label{eq:lyap-weight-condition-damped}
\delta
:=
A-
\sum_{i\in\mathcal I}
\frac{\omega_ic_i^2}{4\kappa_i}
>0,
\end{equation}
then
\begin{align}
\dot{\mathscr V}_{\mathrm{damp}}
&\leq
-\delta\mathscr D
-
\sum_{i\in\mathcal I}\omega_i\kappa_iq_i,
\label{eq:lyap-damped-dissipation}\\
\mathscr V_{\mathrm{damp}}(t)
&\leq
e^{-\lambda_{\mathrm{damp}}t}
\mathscr V_{\mathrm{damp}}(0),
\qquad
\lambda_{\mathrm{damp}}
:=
\min\left\{
\frac{2\mu\delta}{A},
\min_{i\in\mathcal I}\kappa_i
\right\}.
\label{eq:lyap-damped-exponential}
\end{align}
In particular, \(\mathscr E(t)\to0\) and \(q_i(t)\to0\) exponentially for
all selected modes.
\end{theorem}

\begin{proof}
Young's inequality gives
\[
c_i\sqrt{q_i}\sqrt{\mathscr D}
\leq
\kappa_iq_i
+
\frac{c_i^2}{4\kappa_i}\mathscr D.
\]
Combining this with \eqref{eq:lyap-risk-dissipation} and
\eqref{eq:lyap-damped-defect} yields
\eqref{eq:lyap-damped-dissipation}. By \eqref{eq:lyap-PL},
\[
\dot{\mathscr V}_{\mathrm{damp}}
\leq
-2\mu\delta\mathscr E
-
\sum_{i\in\mathcal I}\omega_i\kappa_iq_i
\leq
-\lambda_{\mathrm{damp}}\mathscr V_{\mathrm{damp}},
\]
and Gronwall's inequality proves \eqref{eq:lyap-damped-exponential}.
\end{proof}

Unlike Theorem~\ref{thm:production-controlled-lyapunov}, this result does not
assume \(q_i\lesssim\mathscr E\). As \(\mathscr D(t)\) becomes small, the
production term weakens while \(-2\kappa_iq_i\) remains effective. The
mechanism can therefore eliminate a persistent late-time sensitivity defect,
provided the model-dependent damping estimate holds.

\subsection{A linear consistency example}
\label{subsec:trainable-lyapunov-model}

We record a simplified, fully trainable linear network showing only that the
assumptions of Theorem~\ref{thm:lyap-intrinsic-damping} are mutually
consistent. This is not an example of nonlinear feature alignment: the target
is zero and the sensitivity operator ultimately collapses. Its role is to
check the conditional mechanism, not to supply evidence that intrinsic
damping occurs generically.

Let \(\xi\in\mathbb R^2\) satisfy
\(\mathbb E[\xi\xi^\top]=I_2\), and consider the two-layer linear network
\[
F_{U,v}(\xi)
=
v^\top U\xi,
\qquad
U\in\mathbb R^{2\times2},
\qquad
v\in\mathbb R^2,
\]
trained against the zero target with population risk
\[
\mathcal R(U,v)
=
\frac12\mathbb E\bigl[|F_{U,v}(\xi)|^2\bigr]
=
\frac12v^\top UU^\top v.
\]
At the hidden interface,
\[
P=UU^\top,
\qquad
Q=vv^\top,
\]
and both operators arise from trainable weights. Gradient flow gives
\[
\dot U=-vv^\top U=-QU,
\qquad
\dot v=-UU^\top v=-Pv,
\]
hence
\[
\dot P=-QP-PQ,
\qquad
\dot Q=-PQ-QP.
\]
The imbalance \(\Delta:=P-Q\) is therefore conserved. Choose
\[
\Delta
=
\begin{pmatrix}
2&0\\
0&1
\end{pmatrix},
\qquad
v(0)
=
\begin{pmatrix}
\varepsilon\\
\varepsilon
\end{pmatrix},
\qquad
\varepsilon=\frac1{10},
\]
and set
\[
U(0)
=
\bigl(\Delta+v(0)v(0)^\top\bigr)^{1/2}.
\]
Then
\[
P(t)
=
\Delta+Q(t)
=
\Delta+v(t)v(t)^\top.
\]
Writing \(v=(x,y)^\top\) and \(r^2=x^2+y^2\), we obtain
\[
\dot x=-(2+r^2)x,
\qquad
\dot y=-(1+r^2)y,
\qquad
r^2(t)\leq r^2(0)=\frac1{50}.
\]

Let \(\Pi(t)\) be the projector onto the leading eigendirection of \(P(t)\)
and define
\[
q(t)
:=
\left\|
\Pi(t)Q(t)(I-\Pi(t))
\right\|_F^2
=
\mathcal G_{Q,0}^{(1,0)}(t).
\]
The spectral gap \(\gamma=\lambda_1(P)-\lambda_2(P)\) satisfies
\[
\gamma(t)
\geq
1-r^2(t)
\geq
\frac{49}{50}.
\]
Since \([P,Q]=[\Delta,Q]\), the two-dimensional commutator identity gives
\[
q(t)
=
\frac{x(t)^2y(t)^2}{\gamma(t)^2}.
\]
In particular, \(q(0)>0\).

Set
\[
\mathscr E(t):=\mathcal R(U(t),v(t)),
\qquad
\mathscr D(t)
:=
\|\nabla_U\mathcal R\|_F^2+
\|\nabla_v\mathcal R\|^2.
\]
Here \(\mathcal R_*=0\), and \(\dot{\mathscr E}=-\mathscr D\). Since
\(P(t)\geq\Delta\geq I_2\),
\[
\mathscr D
\geq
\|Pv\|_2^2
=
v^\top P^2v
\geq
v^\top Pv
=
2\mathscr E.
\]
Thus \eqref{eq:lyap-PL} holds with \(\mu=1\).

It remains to verify transverse damping. In the spectral basis of \(P\),
\[
\Pi\dot Q(I-\Pi)
=
-\bigl(\lambda_1(P)+\lambda_2(P)\bigr)
\Pi Q(I-\Pi).
\]
Differentiating the localized block and using
\[
\|\dot\Pi\|_{\mathrm{op}}
\leq
\frac{\|\dot P\|_{\mathrm{op}}}{\gamma}
\]
therefore gives
\[
\dot q
\leq
-2\operatorname{tr}(P)q
+
4\frac{\|Q\|_F\|\dot P\|_{\mathrm{op}}}{\gamma}
\sqrt q.
\]
Moreover,
\[
\|\dot P\|_{\mathrm{op}}
\leq
2\|U\|_{\mathrm{op}}\|\dot U\|_F
\leq
2\sqrt{\|P\|_{\mathrm{op}}}\sqrt{\mathscr D}.
\]
Using
\[
\operatorname{tr}(P)\geq3,
\qquad
\|Q\|_F=r^2\leq\frac1{50},
\qquad
\|P\|_{\mathrm{op}}\leq\frac{101}{50},
\qquad
\gamma\geq\frac{49}{50},
\]
we conclude that
\[
\dot q
\leq
-6q+c_*\sqrt{\mathscr Dq},
\qquad
c_*
:=
\frac{8(1/50)\sqrt{101/50}}{49/50}
<0.233.
\]
Thus \eqref{eq:lyap-damped-defect} holds with \(\kappa=3\) and \(c=c_*\).
Taking \(A=\omega=1\),
\[
\delta
=
1-\frac{c_*^2}{12}
>0,
\]
and Theorem~\ref{thm:lyap-intrinsic-damping} yields
\[
\mathscr E(t)+q(t)
\leq
e^{-\lambda_{\mathrm{damp}}t}
\bigl(\mathscr E(0)+q(0)\bigr),
\qquad
\lambda_{\mathrm{damp}}
=
\min\{2\delta,3\}
>1.99.
\]
Hence both the population risk and a nonzero initial
covariance-sensitivity defect decay exponentially. This example verifies
compatibility of the Lyapunov assumptions within a trainable architecture; it
does not establish nonlinear feature alignment, and here \(q(t)\to0\) partly
because \(Q(t)=v(t)v(t)^\top\) collapses.

\subsection{Consequences for the other alignment levels}

The internal sensitivity energy is the natural Lyapunov variable because it
incorporates the combined effect of the four source terms and transports
directly to feature space. Proposition~\ref{prop:feature-side-localization}
gives, for \(i=(\ell,k,\beta)\in\mathcal I\) such that
\(1\leq k\leq r_\ell\) and \(k+\beta\leq m_\ell\),
\begin{equation}\label{eq:lyap-feature-transport}
\Pi_{\ell,N}^{(k)}
\mathcal S_\ell
\Pi_{\ell,N}^{(k+\beta),\perp}
=
W_\ell^\top
\Pi_\ell^{(k)}Q_\ell
\Pi_\ell^{(k+\beta),\perp}
W_\ell.
\end{equation}
Here \(\Pi_{\ell,N}^{(k)}\) is the leading spectral projector of
\(\mathrm{NFM}_\ell=W_\ell^\top W_\ell\). Consequently,
\begin{equation}\label{eq:lyap-feature-energy-transport}
\mathcal G_{\mathcal S,\ell}^{(k,\beta)}(t)
\leq
\|W_\ell(t)\|_{\mathrm{op}}^4
\mathcal G_{Q,\ell}^{(k,\beta)}(t)
=
\|P_\ell(t)\|_{\mathrm{op}}^2
\mathcal G_{Q,\ell}^{(k,\beta)}(t).
\end{equation}
Thus, under a uniform bound on \(\|P_\ell(t)\|_{\mathrm{op}}\), internal
sensitivity alignment implies the corresponding feature-side AGOP-NFM
alignment. The converse may fail: small singular values or a nontrivial kernel
of \(W_\ell\) can suppress internal mixing. A direct feature-side evolution
would also involve both \(W_\ell\) and the spectral projectors of
\(\mathrm{NFM}_\ell\). We therefore regard
\(\mathcal G_{Q,\ell}^{(k,\beta)}\) as the primitive Lyapunov quantity and
\(\mathcal G_{\mathcal S,\ell}^{(k,\beta)}\) as its transported consequence.

Gate alignment requires an independent argument. For smooth activations, the
localized gate energies may be added to the Lyapunov functional if they
satisfy an analogous error bound or damping estimate. For ReLU, a global
argument requires a bounded-variation formulation with an explicit switching
remainder. In either case, decay of
\(\mathcal G_{Q,\ell}^{(k,\beta)}\) does not imply decay of
\(\mathcal G_{D,\ell}^{(k,\beta)}\): the gate-covariance interaction is only
one source in the internal recursion and may be offset by the other three.

\section{Analytic examples}\label{Analytic}

The preceding sections distinguish activation, internal sensitivity, and
feature-side alignment.
The examples below isolate phenomena that cannot be inferred from any one of
these commutators alone: gate and sensitivity alignment need not coincide,
the sources of the internal commutator may cancel, feature transport may hide
internal incompatibility, and alignment need not be monotone. In a deep linear
network \(D_\ell=I\), so the activation commutator and the nonlinear gate
source vanish identically; this is the reference configuration against which
the nonlinear examples should be read.

\subsection{A two-neuron hidden layer}

Consider
\begin{equation}\label{eq:two-neuron-model}
F(x)=W_1\sigma(W_0x),
\qquad
W_0=
\begin{pmatrix}
w_1^\top\\ w_2^\top
\end{pmatrix},
\qquad
W_1=\begin{pmatrix}c_1&c_2\end{pmatrix},
\end{equation}
and write
\[
P_0=W_0W_0^\top=
\begin{pmatrix}a&b\\b&c\end{pmatrix},
\qquad
(a,b,c)=(\|w_1\|^2,w_1^\top w_2,\|w_2\|^2),
\qquad
D_0=\operatorname{diag}(d_1,d_2).
\]
Direct multiplication gives
\begin{equation}\label{eq:two-neuron-pointwise-commutator}
[D_0,P_0]
=b(d_1-d_2)
\begin{pmatrix}0&1\\-1&0\end{pmatrix},
\qquad
\mathcal G_0
=2b^2\mathbb E_\mu[(d_1-d_2)^2].
\end{equation}
Thus the gate energy separates covariance coupling from disagreement of the
two activation patterns. If \(\lambda_+\geq\lambda_-\) are the eigenvalues of
\(P_0\), \(U=(u_+,u_-)\), and \(\widetilde D_0=U^\top D_0U\), the same identity
becomes
\begin{equation}\label{eq:two-neuron-spectral-factorization}
\mathcal G_0
=2(\lambda_+-\lambda_-)^2
\mathbb E_\mu|\widetilde D_{0,12}|^2,
\end{equation}
which is the two-dimensional form of the general spectral expansion.

Assume now that \(x\sim N(0,I_d)\), \(w_1,w_2\neq0\), and
\(\sigma(r)=\max\{r,0\}\). Set
\[
\rho=\frac{w_1^\top w_2}{\|w_1\|\|w_2\|},
\qquad
\vartheta=\arccos\rho,
\qquad
\pi_{12}:=\mathbb E_\mu[d_1d_2]
=\frac14+\frac{1}{2\pi}\arcsin\rho.
\]
\begin{theorem}\label{thm:ExactGaussianReLUEnergy}
For the Gaussian ReLU model \eqref{eq:two-neuron-model},
\begin{equation}\label{ExactGaussianReLUEnergy}
\mathcal G_0
=\frac{2}{\pi}(w_1^\top w_2)^2\arccos\rho
=\frac{2}{\pi}\|w_1\|^2\|w_2\|^2
\vartheta\cos^2\vartheta.
\end{equation}
\end{theorem}

\begin{proof}
For binary gates,
\((d_1-d_2)^2=\mathbf 1_{\{d_1\neq d_2\}}\). Substituting \(\mathbb P(d_1\neq d_2)=\vartheta/\pi\) into \eqref{eq:two-neuron-pointwise-commutator} gives the result.
\end{proof}

For fixed neuron norms, the formula has two distinct zero mechanisms:
\(\rho=0\) removes covariance coupling, whereas \(\rho=1\) makes the
oriented half-spaces coincide. By contrast, \(\rho=-1\) gives complementary
gates and \(\mathcal G_0=2\|w_1\|^2\|w_2\|^2\). Hence gate compatibility is
not a monotone function of neuron correlation.

\subsubsection{Internal and feature-side alignment}

Let \(s=(c_1,c_2)^\top\). Since \(S_1=\mathcal S_1=ss^\top\) and
\(Q_1=1\),
\begin{equation}\label{eq:one-hidden-Q0}
Q_0=\mathbb E_\mu[D_0S_1D_0]
=\begin{pmatrix}
c_1^2/2&c_1c_2\pi_{12}\\
c_1c_2\pi_{12}&c_2^2/2
\end{pmatrix}.
\end{equation}
With
\(J:=\left(\begin{smallmatrix}0&1\\-1&0\end{smallmatrix}\right)\), one finds
\begin{equation}\label{eq:one-hidden-internal-commutator}
C_0^{\mathrm{int}}=[Q_0,P_0]=\kappa_0J,
\qquad
\kappa_0
=\frac b2(c_1^2-c_2^2)+(c-a)c_1c_2\pi_{12}.
\end{equation}
Consequently, gate alignment need not imply internal alignment: if
\(w_1\perp w_2\), then \(\mathcal G_0=0\), while
\(C_0^{\mathrm{int}}\neq0\) whenever \(a\neq c\) and \(c_1c_2\neq0\).
The transport identity and
\(\|W_0^\top JW_0\|_F^2=2(ac-b^2)\) give the exact feature-side relation
\begin{equation}\label{eq:one-hidden-feature-energy}
\|C_0^{\mathrm{feat}}\|_F^2
=\det(P_0)\|C_0^{\mathrm{int}}\|_F^2.
\end{equation}
Thus rank collapse can hide a nonzero internal commutator, whereas a
well-conditioned \(W_0\) transmits it quantitatively.

\subsubsection{Exact cancellation of local sources}

At the penultimate level there is no transported or fluctuation component,
so
\(C_0^{\mathrm{int}}=T_0^{\mathrm{imb}}+T_0^{\mathrm{gate}}\). Still for
ReLU, but now for a general input distribution, put
\(p_i=\mathbb E_\mu[d_i]\) and
\(p_{12}=\mathbb E_\mu[d_1d_2]\). Then
\begin{align}
T_0^{\mathrm{imb}}&=\kappa_{\mathrm{imb}}J,
&\kappa_{\mathrm{imb}}
&=p_{12}\bigl[b(c_1^2-c_2^2)+(c-a)c_1c_2\bigr],
\label{eq:one-hidden-imbalance-scalar}\\
T_0^{\mathrm{gate}}&=\kappa_{\mathrm{gate}}J,
&\kappa_{\mathrm{gate}}
&=b\bigl[c_1^2p_1-c_2^2p_2+(c_2^2-c_1^2)p_{12}\bigr].
\label{eq:one-hidden-gate-scalar}
\end{align}
The common term \(b(c_1^2-c_2^2)p_{12}\) cancels in their sum. This
cancellation can be exact while both sources and the gate energy remain
nonzero. Indeed, let \(d=2\), \(x\sim N(0,I_2)\), and
\[
w_1=e_1,
\qquad
w_2=e_1+e_2,
\qquad
c_1=u:=\frac{\sqrt{73}-3}{8},
\qquad
c_2=1.
\]
Here \(p_1=p_2=1/2\), \(p_{12}=3/8\), and \(4u^2+3u-4=0\), whence
\begin{equation}\label{eq:exact-source-cancellation}
T_0^{\mathrm{imb}}=\frac{3u}{32}J,
\qquad
T_0^{\mathrm{gate}}=-\frac{3u}{32}J,
\qquad
C_0^{\mathrm{int}}=C_0^{\mathrm{feat}}=0,
\qquad
\mathcal G_0=\frac12.
\end{equation}
This is exact sensitivity and feature-side alignment produced by cancellation,
not by smallness of the local mechanisms.

\subsubsection{Dependence on a function-preserving parametrization}

The represented function does not determine the three commutator geometries.
For \(r>0\), let
\[
H_r=\operatorname{diag}(r,1),
\qquad
W_0^{(r)}=H_rW_0,
\qquad
W_1^{(r)}=W_1H_r^{-1}.
\]
Positive homogeneity of ReLU gives
\(W_1^{(r)}\sigma(W_0^{(r)}x)=F(x)\), and the gates are unchanged, but
\[
P_0^{(r)}=H_rP_0H_r,
\qquad
Q_0^{(r)}=H_r^{-1}Q_0H_r^{-1}.
\]
Writing \(C_0^{\mathrm{int},(r)}=\kappa_0(r)J\), formula
\eqref{eq:one-hidden-internal-commutator} yields
\begin{equation}\label{eq:gauge-dependent-internal-commutator}
\kappa_0(r)=\frac{A}{r}-Br,
\qquad
A=\frac{bc_1^2}{2}+cc_1c_2\pi_{12},
\qquad
B=\frac{bc_2^2}{2}+ac_1c_2\pi_{12}.
\end{equation}
Moreover,
\begin{equation}\label{eq:gauge-dependent-energies}
\mathcal G_0^{(r)}=r^2\mathcal G_0,
\qquad
\|C_0^{\mathrm{feat},(r)}\|_F^2
=2r^2\det(P_0)\kappa_0(r)^2.
\end{equation}
When \(A,B>0\), the choice \(r^2=A/B\) produces exact internal and
feature-side alignment without changing the network function. This does not
make alignment arbitrary: it shows that these quantities describe parameter
geometry and must therefore be interpreted relative to the chosen
parametrization.

\subsubsection{Transient growth along a constrained gradient flow}

The preceding identities hold along arbitrary parameter paths. To exhibit
nonmonotonicity under a population-gradient flow constrained to a symmetric
manifold, let
\(x\sim N(0,I_d)\), choose orthonormal \(v,e\), and take the teacher
\(f_*(x)=\sigma(v^\top x)\) and the symmetric student
\[
F_\theta(x)=\frac12\bigl(\sigma(w_1(\theta)^\top x)
+\sigma(w_2(\theta)^\top x)\bigr),
\qquad
w_{1,2}(\theta)=\cos\theta\,v\pm\sin\theta\,e,
\quad
0\leq\theta\leq\frac\pi2.
\]
For unit vectors with mutual angle \(\alpha\), the Gaussian ReLU kernel is
\begin{equation}\label{GaussianReLUKernel}
K(\alpha)=\frac{1}{2\pi}
\bigl(\sin\alpha+(\pi-\alpha)\cos\alpha\bigr).
\end{equation}
This is the first-order arc-cosine kernel of \cite{ChoSaul2009}.
The quadratic population risk restricted to this symmetric manifold is
\begin{equation}\label{SymmetricTeacherRisk}
\mathcal R(\theta)=\frac38+\frac14K(2\theta)-K(\theta),
\qquad
\mathcal R'(\theta)
=\frac{\sin\theta}{2\pi}
\bigl[(\pi-\theta)-(\pi-2\theta)\cos\theta\bigr]>0
\end{equation}
for \(0<\theta\leq\pi/2\). The Euclidean metric on \((w_1,w_2)\)
induces \(g_\theta=\|\partial_\theta w_1\|^2+
\|\partial_\theta w_2\|^2=2\), so the constrained gradient flow satisfies
\(\dot\theta=-\mathcal R'(\theta)/2\).

\begin{proposition}\label{prop:symmetric-flow-transient-growth}
If \(0<\theta(0)\leq\pi/2\), then \(\theta(t)\to0\) and
\(w_1(t),w_2(t)\to v\). Along the flow,
\begin{equation}\label{SymmetricGateEnergy}
\mathcal G_0(\theta)
=\frac{4\theta}{\pi}\cos^2(2\theta)\longrightarrow0,
\end{equation}
but \(\mathcal G_0(t)\) need not be monotone. In particular, trajectories
starting below and sufficiently close to \(\pi/4\) first increase and then
decrease in gate energy.
\end{proposition}

\begin{proof}
The bracket in \eqref{SymmetricTeacherRisk} equals
\(\theta+(\pi-2\theta)(1-\cos\theta)\), so its strict positivity gives
\(\dot\theta<0\) for \(\theta>0\) and forces \(\theta(t)\to0\). Since
\(w_1^\top w_2=\cos(2\theta)\), Theorem~\ref{thm:ExactGaussianReLUEnergy}
gives \eqref{SymmetricGateEnergy}. Finally,
\[
\frac{\d}{\d\theta}\mathcal G_0(\theta)
=\frac4\pi\bigl(\cos^2(2\theta)-2\theta\sin(4\theta)\bigr).
\]
The displayed derivative is negative immediately to the left of
\(\pi/4\) and positive near \(0\). Since \(\dot\theta<0\), the time
derivative of the energy has the opposite sign, which proves the asserted
initial growth and subsequent decay.
\end{proof}

The initial growth is caused by increasing covariance coupling as the neurons
leave orthogonality; near the teacher, decreasing gate disagreement dominates.
Thus convergence to exact activation alignment does not make the gate energy
a Lyapunov function.

\subsection{A minimal depth-two hierarchy}

Consider
\begin{equation}\label{eq:depth-two-model}
F(x)=W_2\sigma\bigl(W_1\sigma(W_0x)\bigr),
\qquad
W_0=\begin{pmatrix}w_1^\top\\w_2^\top\end{pmatrix},
\qquad
W_1=\begin{pmatrix}\alpha&-\beta\\-\beta&\alpha\end{pmatrix},
\qquad
W_2=\begin{pmatrix}c_1&c_2\end{pmatrix},
\end{equation}
where \(\sigma\) is ReLU and \(\alpha>\beta>0\). Let
\(D_j=\operatorname{diag}(d_{j,1},d_{j,2})\). Then
\[
P_0=\begin{pmatrix}
\|w_1\|^2&w_1^\top w_2\\w_1^\top w_2&\|w_2\|^2
\end{pmatrix},
\qquad
P_1=\begin{pmatrix}
\alpha^2+\beta^2&-2\alpha\beta\\
-2\alpha\beta&\alpha^2+\beta^2
\end{pmatrix},
\]
and \eqref{eq:two-neuron-pointwise-commutator} gives
\begin{align}
\mathcal G_0
&=2(w_1^\top w_2)^2
\mathbb P_\mu(d_{0,1}\neq d_{0,2}),
\label{TwoLayerG0}\\
\mathcal G_1
&=8\alpha^2\beta^2
\mathbb P_\mu(d_{1,1}\neq d_{1,2}).
\label{TwoLayerG1}
\end{align}
If both coordinates of \(a_1=\sigma(W_0x)\) are positive and
\(R=a_{1,1}/a_{1,2}\), the second-layer gates disagree precisely when
\begin{equation}\label{SecondLayerRatioCondition}
R<\frac\beta\alpha
\qquad\text{or}\qquad
R>\frac\alpha\beta.
\end{equation}

\begin{proposition}\label{prop:ExplicitSecondLayerEnergy}
For the depth-two ReLU model \eqref{eq:depth-two-model}, assume
\(x\sim N(0,I_d)\) and \(w_1,w_2\) are orthonormal, and put
\(\tau=\arctan(\beta/\alpha)\). Then
\begin{equation}\label{ExplicitSecondLayerEnergy}
\mathcal G_0=0,
\qquad
\mathcal G_1
=8\alpha^2\beta^2\left(\frac12+\frac\tau\pi\right)>0.
\end{equation}
Moreover, with
\(p_i=\mathbb E_\mu[d_{1,i}]\) and
\(p_{12}=\mathbb E_\mu[d_{1,1}d_{1,2}]\),
\begin{equation}\label{eq:orthogonal-second-layer-gate-moments}
p_1=p_2=\frac12-\frac{\tau}{2\pi},
\qquad
p_{12}=\frac14-\frac\tau\pi.
\end{equation}
\end{proposition}

\begin{proof}
The second and fourth Gaussian quadrants, of total angular measure \(\pi\),
give gate disagreement. In the first quadrant, the two exterior sectors in
\eqref{SecondLayerRatioCondition} have total angle \(2\tau\), while the third
quadrant activates neither neuron. Division by \(2\pi\) gives
\eqref{ExplicitSecondLayerEnergy}; counting the individual active sectors and
their intersection gives \eqref{eq:orthogonal-second-layer-gate-moments}.
\end{proof}

Thus exact gate alignment at the first layer does not prevent a new nonlinear
incompatibility from being created at the second. At that layer,
\begin{equation}\label{eq:depth-two-Q1}
Q_1=\mathbb E_\mu[D_1S_2D_1]
=\begin{pmatrix}
c_1^2p_1&c_1c_2p_{12}\\
c_1c_2p_{12}&c_2^2p_2
\end{pmatrix},
\qquad
S_2=W_2^\top W_2,
\end{equation}
and hence
\begin{equation}\label{SecondLayerSensitivityCommutator}
C_1^{\mathrm{int}}=\kappa_1J,
\qquad
\kappa_1=-2\alpha\beta(c_1^2p_1-c_2^2p_2).
\end{equation}
The gate and internal energies therefore share the covariance factor
\(8\alpha^2\beta^2\), but contain different geometric information: gate
disagreement in the first case and sensitivity imbalance in the second. In
the orthonormal Gaussian configuration, \(p_1=p_2>0\), so
\(C_1^{\mathrm{int}}=0\) when \(|c_1|=|c_2|\), although
\(\mathcal G_1>0\).

Since \(A^\top JA=(\det A)J\) for \(2\times2\) matrices,
\begin{equation}\label{eq:depth-two-feature-energy}
C_1^{\mathrm{feat}}
=(\alpha^2-\beta^2)C_1^{\mathrm{int}},
\qquad
\|C_1^{\mathrm{feat}}\|_F^2
=(\alpha^2-\beta^2)^2\|C_1^{\mathrm{int}}\|_F^2.
\end{equation}
This is the singular-value weighting of feature transport in exact
two-dimensional form. At the preceding layer,
\begin{equation}\label{eq:depth-two-Q0}
Q_0
=\mathbb E_\mu\bigl[
D_0W_1^\top D_1S_2D_1W_1D_0
\bigr],
\qquad
C_0^{\mathrm{feat}}=W_0^\top C_0^{\mathrm{int}}W_0.
\end{equation}
The compact formula already displays the composed activation geometry of both
hidden layers. In particular, if \(P_0=\lambda I\), then
\(C_0^{\mathrm{int}}=C_0^{\mathrm{feat}}=0\) irrespective of the deeper
gates: isotropy can mask, at the earlier layer, incompatibility created deeper
in the network.

At \(\ell=1=L-2\), only the imbalance and gate sources remain. With
\(h=-2\alpha\beta\),
\begin{align}
T_1^{\mathrm{imb}}
&=h(c_1^2-c_2^2)p_{12}J,
\label{eq:depth-two-imbalance-source}\\
T_1^{\mathrm{gate}}
&=h\bigl[c_1^2p_1-c_2^2p_2
+(c_2^2-c_1^2)p_{12}\bigr]J,
\label{eq:depth-two-gate-source}
\end{align}
so their sum is \(C_1^{\mathrm{int}}=h(c_1^2p_1-c_2^2p_2)J\). If
\(p_1,p_2,p_{12}>0\), \(p_1\neq p_2\), and
\(c_1^2p_1=c_2^2p_2\) with \(c_1^2\neq c_2^2\), the two nonzero sources
cancel exactly. The depth-two model therefore reproduces the cancellation
mechanism while also showing that nonlinear incompatibility may be created
anew from layer to layer.

These examples give no universal direction of evolution. A vanishing gate
commutator may coexist with a nonzero internal commutator; a vanishing internal
commutator may result from cancellation of nonzero sources; and a small
feature-side commutator may result from singular-value suppression. Moreover,
the same represented function can carry different commutator geometries, and
even a convergent constrained gradient-flow trajectory can exhibit transient
growth. The three levels should therefore be read as complementary diagnostics
of a layer-dependent mechanism, not as quantities whose decay is built into
the framework.

\section{Numerical experiments}\label{Num}

The purpose of this section is to determine how the mechanisms isolated in
the preceding analysis are realized along finite-width training trajectories.
The exact transport identity and internal commutator recursion require no
numerical confirmation as mathematical statements.  Numerically, the relevant
questions are instead whether the gate, internal-sensitivity, and feature-side
geometries evolve in the same manner; whether a small internal commutator is
produced by small source terms or by cancellation among large terms; and
whether the resulting picture persists across depth, width, initialization,
and data.  We therefore distinguish throughout between exact identities,
ensemble observations, and conditional interpretations.  In particular, a
decrease of the risk is not assumed to force decay of any commutator.

The study has three components. First, a controlled synthetic experiment
compares two, four, and six hidden layers at matched relative-risk levels and
adds a width comparison at fixed depth. Second, the resulting conclusions are
tested on the Concrete Compressive Strength and Energy Efficiency regression
datasets.\footnote{Both datasets are taken from the UCI Machine Learning
Repository. The Concrete Compressive Strength dataset contains \(1030\)
observations with eight quantitative input variables describing the concrete
composition and age, and compressive strength as the scalar response. The
Energy Efficiency dataset contains \(768\) simulated building configurations
with eight input variables and two responses, heating and cooling load; in the
present experiments, heating load is used as the scalar response.} Third, a
paired positive-homogeneous rescaling preserves the initialized network
function while changing the adjacent-layer imbalance. The first two components
test robustness of the geometric picture, whereas the third probes how the
imbalance and parametrization class influence the internal hierarchy while
holding the initialized function and gates fixed.


\subsection{Experimental setting and observables}
\label{sec:numerical-setting}

For the synthetic experiments, the input distribution is the standard
Gaussian measure on \(\mathbb R^2\), and the scalar target is
\begin{equation}
f_\ast(x_1,x_2)
=
0.8(x_1+0.6x_2)_+
-0.6(-0.7x_1+x_2)_+
+0.25x_1.
\label{eq:numerical-target}
\end{equation}
A fixed sample of \(N=256\) points is used for training, and an independent
sample of size \(M=1024\) is used for geometric diagnostics. All networks are
bias-free, fully connected ReLU networks with scalar linear output. If \(H\)
denotes the number of hidden layers, the depth experiment uses
\[
H\in\{2,4,6\},
\qquad
m_1=\cdots=m_H=32,
\]
whereas the capacity experiment fixes \(H=4\) and uses widths
\(16,32,64\). The weight entries are initialized independently with variance
\(2/m_{\mathrm{in}}\), and the empirical risk
\[
\mathscr E_N(W)
=
\frac1{2N}\sum_{i=1}^N
\bigl(F_W(x_i)-f_\ast(x_i)\bigr)^2
\]
is minimized by full-batch gradient descent in double precision, with step
size \(\eta=10^{-2}\), for at most \(10^4\) iterations. We identify iteration
\(n\) with the scaled gradient-descent time \(t_n=n\eta\); this is a discrete
time variable, not an observed continuous gradient-flow trajectory. Thus the
final checkpoint corresponds to \(T=100\). The geometry is evaluated at
logarithmically spaced checkpoints and at the first crossings of
\begin{equation}
r(t)
:=
\frac{\mathscr E_N(t)}{\mathscr E_N(0)}
\in
\{10^{-1},10^{-2},10^{-3},10^{-4}\}.
\label{eq:numerical-relative-risk}
\end{equation}
Architecture comparisons are therefore made at matched optimization progress
rather than at a common iteration number. For the squared loss used here,
\(\mathscr E_N=\mathcal R_N\) in the notation of
\eqref{Pre:empirical-risk}.

The Concrete Compressive Strength and Energy Efficiency datasets
\cite{Yeh1998,TsanasXifara2012} contain, respectively, \(1030\) and \(768\)
observations with eight input variables.
For Energy, the heating load is used as the scalar response. Each dataset is
split once into \(80\%\) training data and \(20\%\) test data. Inputs and
targets are standardized using statistics computed only from the training
split. Four-hidden-layer width-\(32\) networks are trained for \(10^4\)
iterations over twenty independent initializations, with the same
training-test split retained across all seeds.

When no ambiguity arises, the sample subscript is suppressed in tables and
figures. Synthetic global profiles use the independent diagnostic sample,
benchmark global profiles use the test sample, and all source decompositions
use the corresponding training sample.

For a diagnostic sample \(\{x_a\}_{a=1}^M\), set
\[
Q_{\ell,M}
=\frac1M\sum_{a=1}^M
D_\ell(x_a)S_{\ell+1}(x_a)D_\ell(x_a),
\qquad
\mathcal S_{\ell,M}
=\frac1M\sum_{a=1}^M S_\ell(x_a).
\]
The three global defects are normalized by the operator scale of the first
factor and the covariance scale of the second:
\begin{align}
\widehat{\mathcal G}_{D,\ell,M}
&:=
\frac{
M^{-1}\sum_{a=1}^M
\|[D_\ell(x_a),P_\ell]\|_F^2
}{
\left(
M^{-1}\sum_{a=1}^M
\|D_\ell(x_a)\|_{\mathrm{op}}^2
\right)
\|P_\ell\|_F^2
},
\label{eq:numerical-normalized-gate}
\\
\widehat{\mathcal G}_{Q,\ell,M}
&:=
\frac{
\|[Q_{\ell,M},P_\ell]\|_F^2
}{
\|Q_{\ell,M}\|_{\mathrm{op}}^2
\|P_\ell\|_F^2
},
\label{eq:numerical-normalized-internal}
\\
\widehat{\mathcal G}_{\mathcal S,\ell,M}
&:=
\frac{
\|[\mathcal S_{\ell,M},\mathrm{NFM}_\ell]\|_F^2
}{
\|\mathcal S_{\ell,M}\|_{\mathrm{op}}^2
\|\mathrm{NFM}_\ell\|_F^2
}.
\label{eq:numerical-normalized-feature}
\end{align}
Here $Q_{\ell,M}$ and $\mathcal S_{\ell,M}$ denote the empirical counterparts
of the corresponding population operators. The gate defect averages the
pointwise commutator energy, whereas the sensitivity and feature-side defects
are formed from the empirically averaged operators. Thus
\eqref{eq:numerical-normalized-internal} and
\eqref{eq:numerical-normalized-feature} correspond to the internal and
feature-side objects studied in the theory, rather than to stronger pointwise
variants. These dimensionless quantities measure relative incompatibility
and are invariant under separate scalar rescalings of the two operators,
whenever the normalizing denominators are nonzero; otherwise they are left
undefined according to the convention specified below.

For localized quantities, let \(\Pi_{\ell,P}^{(k)}\) denote the leading
rank-\(k\) spectral projector of \(P_\ell\), and set
\[
\mathsf Q_\ell(x)
:=
D_\ell(x)S_{\ell+1}(x)D_\ell(x).
\]
For a sample of size \(M\), define
\begin{align}
q_{D,\ell,M}^{(k,\beta)}
&:=
\frac1M\sum_{a=1}^M
\left\|
\Pi_{\ell,P}^{(k)}D_\ell(x_a)
\Pi_{\ell,P}^{(k+\beta),\perp}
\right\|_F^2,
\label{eq:numerical-localized-D}
\\
q_{Q,\ell,M}^{(k,\beta),\mathrm{av}}
&:=
\left\|
\Pi_{\ell,P}^{(k)}Q_{\ell,M}
\Pi_{\ell,P}^{(k+\beta),\perp}
\right\|_F^2,
\label{eq:numerical-localized-Q-averaged}
\\
q_{Q,\ell,M}^{(k,\beta),\mathrm{pt}}
&:=
\frac1M\sum_{a=1}^M
\left\|
\Pi_{\ell,P}^{(k)}\mathsf Q_\ell(x_a)
\Pi_{\ell,P}^{(k+\beta),\perp}
\right\|_F^2.
\label{eq:numerical-localized-Q-pointwise}
\end{align}
The averaged quantity in
\eqref{eq:numerical-localized-Q-averaged} is the empirical counterpart of
the localized internal defect used in the Lyapunov theory, whereas the
pointwise quantity in \eqref{eq:numerical-localized-Q-pointwise} is a
stronger inputwise diagnostic. Jensen's inequality gives
\[
q_{Q,\ell,M}^{(k,\beta),\mathrm{av}}
\leq
q_{Q,\ell,M}^{(k,\beta),\mathrm{pt}}.
\]

Let
\[
r_{\ell,P}^{(k,\beta)}
:=
\min\{k,m_{\ell+1}-k-\beta\}.
\]
The corresponding normalized quantities are
\begin{align}
\widehat q_{D,\ell,M}^{(k,\beta)}
&:=
\frac{
q_{D,\ell,M}^{(k,\beta)}
}{
r_{\ell,P}^{(k,\beta)}
M^{-1}\sum_{a=1}^M
\|D_\ell(x_a)\|_{\mathrm{op}}^2
},
\label{eq:numerical-localized-D-normalized}
\\
\widehat q_{Q,\ell,M}^{(k,\beta),\mathrm{av}}
&:=
\frac{
q_{Q,\ell,M}^{(k,\beta),\mathrm{av}}
}{
r_{\ell,P}^{(k,\beta)}
\|Q_{\ell,M}\|_{\mathrm{op}}^2
},
\label{eq:numerical-localized-Q-averaged-normalized}
\\
\widehat q_{Q,\ell,M}^{(k,\beta),\mathrm{pt}}
&:=
\frac{
q_{Q,\ell,M}^{(k,\beta),\mathrm{pt}}
}{
r_{\ell,P}^{(k,\beta)}
M^{-1}\sum_{a=1}^M
\|\mathsf Q_\ell(x_a)\|_{\mathrm{op}}^2
}.
\label{eq:numerical-localized-Q-pointwise-normalized}
\end{align}

Writing
\(\lambda_{\ell,1}\geq\cdots\geq\lambda_{\ell,m_{\ell+1}}\)
for the eigenvalues of \(P_\ell\), define the normalized boundary gap by
\[
\widehat\gamma_{\ell,j}
:=
\frac{\lambda_{\ell,j}-\lambda_{\ell,j+1}}
{\lambda_{\ell,1}}
=
\frac{\lambda_{\ell,j}-\lambda_{\ell,j+1}}
{\|P_\ell\|_{\mathrm{op}}}.
\]
We use
\[
(k,\beta)\in\{(1,0),(1,1),(2,0),(2,2)\}
\]
and retain a projector-dependent quantity only when
\[
\min\left\{
\widehat\gamma_{\ell,k},
\widehat\gamma_{\ell,k+\beta}
\right\}
\geq10^{-6}.
\]
All cutoffs reported below satisfy this criterion for every seed.
Ensemble curves display means over the defined seed set and pointwise \(95\%\)
bootstrap intervals; matrices and projectors are never averaged across
independently initialized networks. The threshold is only a numerical
nondegeneracy filter. It does not by itself establish statistical stability of
the selected eigenspaces, and no robustness claim with respect to alternative
gap thresholds is made here.

The source-level diagnostics are computed on the corresponding training
sample. Writing
\[
Q_{\ell,N}
=
\frac1N\sum_{a=1}^N
D_\ell(x_a)S_{\ell+1}(x_a)D_\ell(x_a),
\]
and replacing every population expectation in the four source terms by an
average over the same sample, the exact finite-sample identity is
\begin{equation}
[Q_{\ell,N},P_\ell]
=
T_\ell^{\mathrm{tr}}
+
T_\ell^{\mathrm{imb}}
+
T_\ell^{\mathrm{fluc}}
+
T_\ell^{\mathrm{gate}}.
\label{eq:numerical-source-decomposition}
\end{equation}
Let
\[
\mathcal V
=
\{\mathrm{tr},\mathrm{imb},\mathrm{fluc},\mathrm{gate}\}.
\]
For \(\nu\in\mathcal V\), define the source-energy fraction and aggregate
cancellation ratio by
\begin{equation}
e_\ell^\nu
:=
\frac{\|T_\ell^\nu\|_F^2}
{\displaystyle\sum_{\zeta\in\mathcal V}
\|T_\ell^\zeta\|_F^2},
\qquad
\chi_\ell
:=
\frac{
\bigl\|\displaystyle\sum_{\nu\in\mathcal V}T_\ell^\nu\bigr\|_F
}{
\displaystyle\sum_{\nu\in\mathcal V}\|T_\ell^\nu\|_F
}.
\label{eq:numerical-source-fractions}
\end{equation}
Whenever they are defined, the fractions \(e_\ell^\nu\) sum to one and
describe the relative squared magnitudes of the four sources. They do not
include the pairwise interaction terms in the squared norm of the total
commutator. The ratio \(\chi_\ell\in[0,1]\) compares the norm of the total
commutator with the sum of the source norms; a small value therefore indicates
strong matrix-valued cancellation. We further measure the principal
directional interaction by
\begin{equation}
\rho_\ell^{\mathrm{tr},\mathrm{imb}}
:=
\frac{
\langle
T_\ell^{\mathrm{tr}},
T_\ell^{\mathrm{imb}}
\rangle_F
}{
\|T_\ell^{\mathrm{tr}}\|_F
\|T_\ell^{\mathrm{imb}}\|_F
}.
\label{eq:numerical-source-cosine}
\end{equation}
Whenever it is defined,
\(\rho_\ell^{\mathrm{tr},\mathrm{imb}}\in[-1,1]\), with negative values
indicating that the transport and imbalance sources point in opposing
directions in Frobenius space.

All ensemble summaries retain every seed for which the displayed observable
is defined. A normalized global or localized defect is left undefined when
its normalizing denominator vanishes. The quantities \(e_\ell^\nu\) and
\(\chi_\ell\) are defined only when at least one source term is nonzero, and
\(\rho_\ell^{\mathrm{tr},\mathrm{imb}}\) only when both source norms are
positive. Undefined values are omitted from the corresponding layerwise mean
and bootstrap interval and are never replaced by zero. In particular, seed
\(19\), which produces a dead layer at one recorded checkpoint, remains in
the risk summaries and in every well-defined global statistic but is omitted
from the affected sensitivity and source summaries. We call a summary formed
after such an omission an \emph{active-seed summary}; all other summaries use
the full stated ensemble. Thus a stored zero placeholder for a dead layer is
not interpreted as perfect alignment or perfect cancellation.

Table~\ref{tab:numerical-success} records the number of seeds reaching each
relative-risk level. All synthetic seeds reach \(10^{-3}\), which is
therefore the principal matched-risk level. The benchmark ensembles do not
reach \(10^{-3}\), and conditioning the comparison on the seven Concrete or
fifteen Energy seeds that reach \(10^{-2}\) would change the initialization
ensemble. Their principal comparison is consequently made at the common final
time \(t=100\).

\begin{table}[t]
\centering
\small
\begin{tabular}{c|c|c|c|ccc}
experiment & \(H\) & width & seeds
& \(r\leq10^{-2}\) & \(r\leq10^{-3}\) & \(r\leq10^{-4}\) \\
\hline
synthetic depth & 2 & 32 & 20 & 20 & 20 & 13 \\
synthetic depth & 4 & 32 & 20 & 20 & 20 & 20 \\
synthetic depth & 6 & 32 & 20 & 20 & 20 & 19 \\
synthetic capacity & 4 & 16 & 10 & 10 & 10 & 4 \\
synthetic capacity & 4 & 64 & 10 & 10 & 10 & 10 \\
Concrete & 4 & 32 & 20 & 7 & 0 & 0 \\
Energy & 4 & 32 & 20 & 15 & 0 & 0
\end{tabular}
\caption{Numbers of seeds reaching the indicated relative-risk levels. No
architecture comparison is reported at a risk level not attained by every
seed in the corresponding ensemble.}
\label{tab:numerical-success}
\end{table}

\subsection{Risk reduction does not force global alignment}
\label{sec:numerical-depth-alignment}

Figure~\ref{fig:numerical-depth-alignment} compares the three normalized
defects at \(r=10^{-3}\). The gate defect is substantially larger than the
internal and feature defects at every depth. At the input layer its mean is
\(0.4486\), \(0.4486\), and \(0.4487\) for \(H=2,4,6\), respectively. The
corresponding internal means are \(0.0646\), \(0.0805\), and \(0.0793\),
whereas the feature means are \(0.0211\), \(0.0199\), and \(0.0201\). Thus
the input-layer gate geometry is almost depth-independent, while the smaller
feature-side values are consistent with attenuation of internal
incompatibility under congruence transport. In the interior of the deeper
networks, the internal and feature defects are both of order \(10^{-2}\) but
remain nonzero.

\begin{figure}[t]
\centering
\includegraphics[width=0.94\textwidth]
{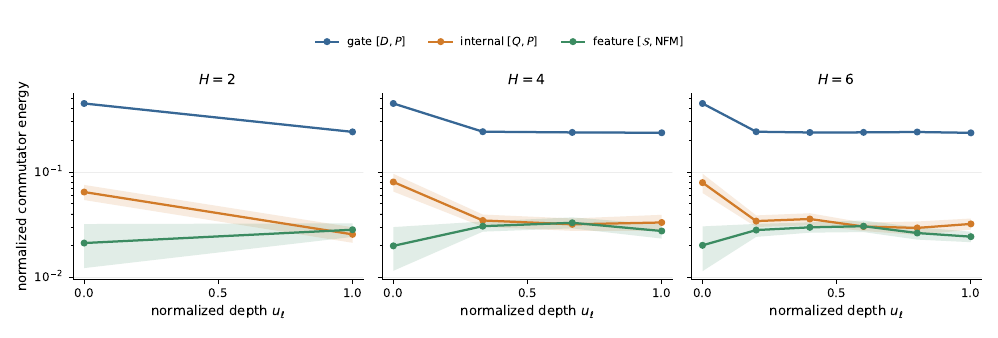}
\caption{The three normalized commutator energies at matched relative risk
\(r=10^{-3}\), plotted against normalized depth
\(u_\ell=\ell/(H-1)\). Curves are means over the twenty-seed ensembles,
subject to the active-seed convention for an undefined normalized entry, and
shaded regions are pointwise \(95\%\) bootstrap intervals. The logarithmic
vertical scale makes clear that the gate, averaged internal, and aggregated
feature geometries remain quantitatively distinct.}
\label{fig:numerical-depth-alignment}
\end{figure}

The separation is not merely an endpoint effect. Table
\ref{tab:numerical-risk-evolution} follows the four-hidden-layer network while
the relative risk decreases from \(10^{-1}\) to \(10^{-3}\). The gate defect
is essentially constant in every layer. At the input, the internal defect
increases from \(0.0446\) to \(0.0805\), and the feature defect increases from
\(0.0134\) to \(0.0199\). The three interior profiles change only slightly.
At \(r=10^{-4}\), the corresponding input values are \(0.4487\), \(0.0791\),
and \(0.0211\), so the profile has stabilized without approaching global
commutation.

\begin{table}[t]
\centering
\small
\begin{tabular}{c|cc|cc|cc}
& \multicolumn{2}{c|}{\(\widehat{\mathcal G}_{D,\ell}\)}
& \multicolumn{2}{c|}{\(\widehat{\mathcal G}_{Q,\ell}\)}
& \multicolumn{2}{c}{\(\widehat{\mathcal G}_{\mathcal S,\ell}\)} \\
\(\ell\) & \(10^{-1}\) & \(10^{-3}\)
& \(10^{-1}\) & \(10^{-3}\)
& \(10^{-1}\) & \(10^{-3}\) \\
\hline
0 & 0.4486 & 0.4486 & 0.0446 & 0.0805 & 0.0134 & 0.0199 \\
1 & 0.2418 & 0.2418 & 0.0362 & 0.0347 & 0.0306 & 0.0306 \\
2 & 0.2378 & 0.2381 & 0.0327 & 0.0318 & 0.0329 & 0.0329 \\
3 & 0.2368 & 0.2362 & 0.0332 & 0.0331 & 0.0278 & 0.0275
\end{tabular}
\caption{Means for the four-hidden-layer width-\(32\) synthetic network at two
matched relative-risk levels. All twenty seeds enter except for an undefined
normalized entry covered by the active-seed convention. A hundredfold
decrease of the risk does not produce global commutator decay.}
\label{tab:numerical-risk-evolution}
\end{table}

This experiment therefore gives no evidence for universal alignment generated
solely by risk minimization. It instead exhibits three different responses:
the gate geometry is nearly frozen, the input internal and feature defects
increase before stabilizing, and the interior defects remain approximately
constant. The numerical conclusion is not that alignment is absent, but that
it is selective and cannot be inferred from the risk alone.

\subsection{Depth propagation and systematic source cancellation}
\label{sec:numerical-source-cancellation}

We now apply the finite-sample decomposition and source diagnostics defined in
\eqref{eq:numerical-source-decomposition}-\eqref{eq:numerical-source-cosine}.
Figure~\ref{fig:numerical-depth-sources} shows a stable backward-propagation
pattern. In the interior layers of the \(H=4\) network,
\(\chi_1=0.296\) and \(\chi_2=0.290\); for \(H=6\), the corresponding
interior values range from \(0.298\) to \(0.321\). Hence the net internal
commutator is only about thirty percent of the sum of the source magnitudes.
The main cancellation is between transport and imbalance:
\[
\rho_\ell^{\mathrm{tr},\mathrm{imb}}
\in[-0.712,-0.703]
\quad\text{for the two interior layers when }H=4,
\]
and it lies between \(-0.716\) and \(-0.649\) across the four interior layers
when \(H=6\). This persistent negative interaction is not visible from
separate norm estimates for the sources.

\begin{figure}[t]
\centering
\includegraphics[width=0.98\textwidth]
{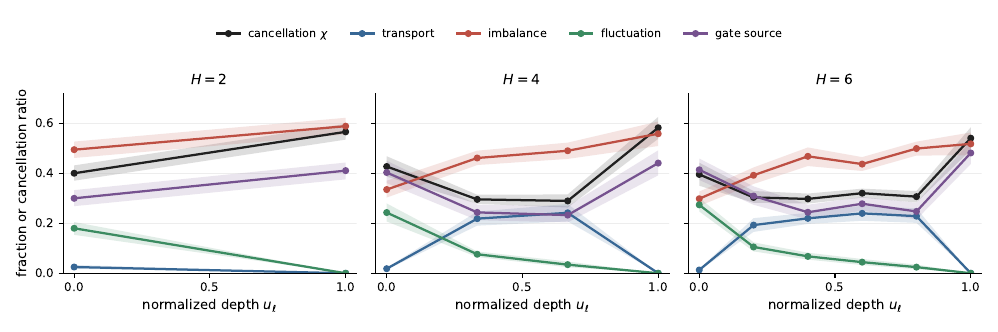}
\caption{Source-energy fractions and aggregate cancellation ratio at
\(r=10^{-3}\). Curves are active-seed means and shaded regions are
pointwise \(95\%\) bootstrap intervals; all twenty seeds enter whenever the
displayed ratio is defined. Imbalance is the largest interior source,
transport remains substantial, and the fluctuation fraction decreases toward
the output. At the terminal hidden layer, transport and fluctuation vanish
by the boundary structure of the recursion, so that layer is not directly
comparable with the interior ones.}
\label{fig:numerical-depth-sources}
\end{figure}

The source fractions add a second distinction between layers. Near the input,
the fluctuation and nonlinear gate terms carry a substantial part of the
source energy, while transport is small. Moving toward the interior,
imbalance becomes dominant, transport rises to approximately \(20\%\)-\(25\%\),
and the fluctuation contribution decreases. These profiles change little
between the last matched-risk checkpoints. The observed small internal
commutator is therefore not generally the result of uniformly small
production; it is maintained by a reproducible cancellation structure.

\subsection{Width dependence and the localized Lyapunov question}
\label{sec:numerical-width-lyapunov}

The capacity experiment compares the width dependence of the three defects.
Figure
\ref{fig:numerical-width-alignment} and Table
\ref{tab:numerical-width} show that, over the tested width range, increasing
width is associated with smaller internal and feature defects, especially in
the interior, but not with a smaller gate defect.  At layer \(1\), the
internal mean decreases from \(0.0635\) at
width \(16\) to \(0.0185\) at width \(64\), and the feature mean decreases
from \(0.0555\) to \(0.0170\).  In contrast, the input gate defect increases
from \(0.3610\) to \(0.4685\).  Three widths do not determine an asymptotic
scaling law, but over the tested range they do not support the interpretation
that the three defects are interchangeable finite-width measurements of one
common alignment quantity.

\begin{figure}[t]
\centering
\includegraphics[width=0.98\textwidth]
{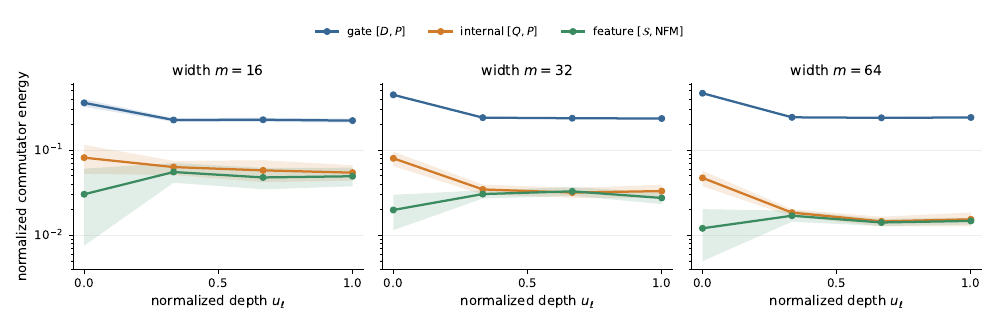}
\caption{Width comparison for four hidden layers at \(r=10^{-3}\).  The
width-\(32\) curves use twenty seeds and the width-\(16\) and width-\(64\)
curves use ten seeds, subject to the active-seed convention for an undefined
normalized entry; bands are pointwise \(95\%\) bootstrap intervals.  The
internal and feature defects decrease over the tested widths, whereas the
gate defect does not.}
\label{fig:numerical-width-alignment}
\end{figure}

\begin{table}[t]
\centering
\small
\begin{tabular}{c|ccc|ccc}
width
& \(\widehat{\mathcal G}_{D,0}\)
& \(\widehat{\mathcal G}_{Q,1}\)
& \(\widehat{\mathcal G}_{\mathcal S,1}\)
& \(\chi_1\) & \(\chi_2\)
& \(\rho_1^{\mathrm{tr},\mathrm{imb}}\) \\
\hline
16 & 0.3610 & 0.0635 & 0.0555 & 0.352 & 0.348 & -0.663 \\
32 & 0.4486 & 0.0347 & 0.0306 & 0.296 & 0.290 & -0.712 \\
64 & 0.4685 & 0.0185 & 0.0170 & 0.314 & 0.270 & -0.649
\end{tabular}
\caption{Seed means at \(r=10^{-3}\) in the four-hidden-layer synthetic
capacity experiment.  The separation of the three alignment levels and the
negative transport-imbalance interaction both persist across the tested
widths.  Source ratios are active-seed means whenever a denominator
vanishes.}
\label{tab:numerical-width}
\end{table}

The localized data give a direct numerical caution for the Lyapunov
principles of Section~\ref{sec:lyapunov-alignment}. Define the unnormalized
leading internal defect on the training sample by
\[
q_{Q,0,N}^{(1,0),\mathrm{av}}
:=
\left\|
\Pi_{0,P}^{(1)}Q_{0,N}
\Pi_{0,P}^{(1),\perp}
\right\|_F^2.
\]
 Table
\ref{tab:numerical-localized-risk} shows that this defect increases while the
risk decreases by two orders of magnitude.  The same conclusion holds for
the buffered cutoff \((k,\beta)=(1,1)\), for which the median paired growth
factors are \(2.15\), \(1.96\), and \(1.79\) at widths \(16,32,64\).

\begin{table}[t]
\centering
\small
\begin{tabular}{c|cc|c|cc}
width
& \multicolumn{2}{c|}{mean \(q_{Q,0,N}^{(1,0),\mathrm{av}}\)}
& median paired factor
& \multicolumn{2}{c}{mean \(\widehat q_{D,0,N}^{(1,0)}\)} \\
& \(r=10^{-1}\) & \(r=10^{-3}\)
& \(q(10^{-3})/q(10^{-1})\)
& \(r=10^{-1}\) & \(r=10^{-3}\) \\
\hline
16 & \(4.61\times10^{-3}\) & \(9.15\times10^{-3}\) & 2.09 & 0.1949 & 0.1960 \\
32 & \(1.88\times10^{-3}\) & \(3.41\times10^{-3}\) & 2.03 & 0.2287 & 0.2287 \\
64 & \(7.63\times10^{-4}\) & \(1.24\times10^{-3}\) & 1.79 & 0.2395 & 0.2395
\end{tabular}
\caption{Localized input-layer behavior in the four-hidden-layer synthetic
experiment.  The raw averaged-\(Q\) defect grows across the matched-risk
interval, while the normalized localized gate defect is stationary.  Every
displayed projector boundary passes the spectral reliability test.}
\label{tab:numerical-localized-risk}
\end{table}

For width \(32\), all twenty seeds also reach \(r=10^{-4}\).  At that level,
\[
\bigl\langle q_{Q,0,N}^{(1,0),\mathrm{av}}\bigr\rangle_{\mathrm{seed}}
=3.24\times10^{-3},
\qquad
\bigl\langle
\frac{q_{Q,0,N}^{(1,0),\mathrm{av}}}{\mathscr E_N}
\bigr\rangle_{\mathrm{seed}}
=37.4,
\]
whereas the latter mean is \(0.025\) at \(r=10^{-1}\).  Hence the data do not
support a risk-transfer estimate with a constant comparable to its early-time
value.  More importantly, the localized defect does not track the risk at all:
it grows and then stabilizes while the risk continues to decrease.  This does
not contradict the conditional Lyapunov results, whose intrinsic damping and
risk-transfer hypotheses are additional assumptions.  It shows that risk
dissipation alone cannot supply those hypotheses in the present regime.

\subsection{Replication on scalar-regression benchmarks}
\label{sec:numerical-benchmarks}

The benchmark trajectories provide a test of whether the synthetic source
geometry is tied to the specially constructed target \eqref{eq:numerical-target}.
At the common final time, the twenty-seed training and test risks are
\begin{align*}
\text{Concrete:}\quad
&\mathscr E_{\mathrm{train}}=0.0183\pm0.0018,
&&\mathscr E_{\mathrm{test}}=0.0780\pm0.0117,\\
\text{Energy:}\quad
&\mathscr E_{\mathrm{train}}=0.0061\pm0.0024,
&&\mathscr E_{\mathrm{test}}=0.0181\pm0.0049.
\end{align*}
These values are half mean-squared errors in standardized target units.
Here \(\pm\) denotes one standard deviation across the twenty seeds.

Figure~\ref{fig:numerical-benchmark-alignment} shows that the three-level
separation is reproduced on the test data.  For Concrete, the layer-\(1\)
gate, internal, and feature means are \(0.2365\), \(0.0320\), and \(0.0281\);
for Energy they are \(0.2397\), \(0.0327\), and \(0.0302\).  The agreement is
not imposed by a shared training sample or target: the datasets have different
observations, responses, and train-test splits.

\begin{figure}[t]
\centering
\includegraphics[width=0.78\textwidth]
{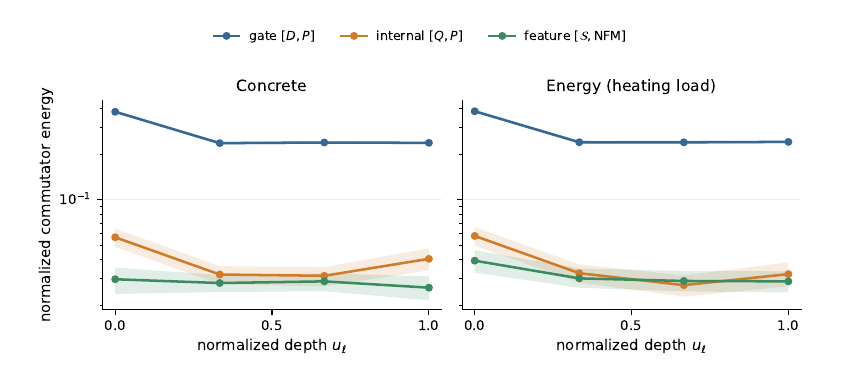}
\caption{Test-sample alignment profiles at the common final training time for
Concrete Compressive Strength and Energy Efficiency with heating load as the
response.  Curves are means over the twenty-seed ensembles, subject to the
active-seed convention for an undefined normalized entry, and bands are
pointwise \(95\%\) bootstrap intervals.  Both datasets reproduce the
quantitative separation of the gate geometry from the averaged internal and
aggregated feature geometries.}
\label{fig:numerical-benchmark-alignment}
\end{figure}

The source decomposition also replicates.  Table
\ref{tab:numerical-benchmark-sources} shows strong cancellation in both
interior layers, again dominated by a negative transport-imbalance
interaction.  The imbalance term carries approximately one half of the source
energy.  The fluctuation fraction is larger near the input and falls from
\(0.0345\) to \(0.0116\) between layers \(1\) and \(2\) for Concrete, and
from \(0.0599\) to \(0.0227\) for Energy.  Thus both the cancellation and the
backward decrease of the fluctuation contribution persist beyond the
synthetic model.

\begin{table}[t]
\centering
\small
\begin{tabular}{c|cc|cc|cc}
dataset
& \(\chi_1\) & \(\rho_1^{\mathrm{tr},\mathrm{imb}}\)
& \(\chi_2\) & \(\rho_2^{\mathrm{tr},\mathrm{imb}}\)
& \(e_1^{\mathrm{imb}}\) & \(e_2^{\mathrm{imb}}\) \\
\hline
Concrete & 0.327 & -0.722 & 0.347 & -0.689 & 0.517 & 0.539 \\
Energy   & 0.310 & -0.678 & 0.289 & -0.713 & 0.480 & 0.516
\end{tabular}
\caption{Twenty-seed means of the final-time source geometry on the two
benchmarks, with undefined ratios omitted as above.  The interior cancellation and the negative
transport-imbalance interaction reproduce the synthetic pattern.}
\label{tab:numerical-benchmark-sources}
\end{table}


\subsection{A function-preserving imbalance intervention}
\label{sec:numerical-rescaling}

This paired experiment uses the same synthetic target
\eqref{eq:numerical-target}, but a separate three-hidden-layer width-\(32\)
network with \(N=160\) training points, \(M=1024\) diagnostic points, and ten
paired initializations.  The learning rate is \(10^{-2}\), and the final time
is \(T=100\).

To vary the adjacent-layer imbalance while holding the realized initial
function and activation gates fixed, we apply the positive rescaling
\[
W_1^{(c)}(0)=cW_1(0),
\qquad
W_2^{(c)}(0)=c^{-1}W_2(0),
\qquad
c\in\left\{\frac12,1,2\right\}.
\]
By the positive homogeneity of the ReLU activation, this transformation
preserves both the initial network function and the activation gates. It does,
however, change the adjacent-layer imbalance according to
\begin{equation}
\Delta_1^{(c)}(0)
=
c^2W_1(0)W_1(0)^\top
-
c^{-2}W_2(0)^\top W_2(0).
\label{eq:numerical-rescaled-imbalance}
\end{equation}
We measure its relative magnitude at initialization by
\begin{equation}
\widehat\Delta_1^{(c)}(0)
:=
\frac{
\|\Delta_1^{(c)}(0)\|_F
}{
\|c^2W_1(0)W_1(0)^\top\|_F
+
\|c^{-2}W_2(0)^\top W_2(0)\|_F
}.
\label{eq:numerical-normalized-imbalance}
\end{equation}
Across the ten paired initializations, the mean and standard deviation of
\(\widehat\Delta_1^{(c)}(0)\) are
\[
0.914\pm0.007,\qquad
0.507\pm0.009,\qquad
0.914\pm0.007
\]
for \(c=1/2\), \(c=1\), and \(c=2\), respectively. Thus the two nontrivial
rescalings have nearly equal normalized imbalance magnitudes, although they
weight the two matrix contributions to \(\Delta_1^{(c)}(0)\) in opposite
ways.

This intervention also has a neuronwise balancedness interpretation. Indeed,
\begin{align}
\operatorname{Diag}\Delta_1^{(c)}(0)
&=
c^2\operatorname{Diag}
\bigl(W_1(0)W_1(0)^\top\bigr)
-
c^{-2}\operatorname{Diag}
\bigl(W_2(0)^\top W_2(0)\bigr).
\label{eq:numerical-rescaled-diagonal-imbalance}
\end{align}
Consequently, although the rescaling preserves the realized network function
and the activation gates at initialization, different values of \(c\)
generally place the trajectories in distinct neuronwise imbalance classes. For
bias-free ReLU gradient flow, these diagonal imbalances are conserved by
Remark~\ref{rem:neuronwise-balancedness}. The subsequent differences between
the rescaled trajectories should therefore not be interpreted as random
numerical variation. Rather, they reflect training from functionally
equivalent initializations with systematically different diagonal imbalances,
which may
alter the imbalance source and its interaction with the remaining terms in the
commutator recursion. Under the finite-step optimization used in the
experiments, the corresponding conservation law holds only approximately.

The intervention should nevertheless not be interpreted as changing
\(\Delta_1\) alone. The rescaling also changes the scales of the adjacent
covariance operators and the parametrization of the gradient dynamics.
Accordingly, the experiment separates the imbalance and parametrization class
from the realized initial function, but does not attribute the subsequent
differences exclusively to \(\Delta_1\).

For two source terms, also define the magnitude-weighted interaction
\[
I_\ell^{\nu,\eta}
=
\frac{2\langle T_\ell^\nu,T_\ell^\eta\rangle_F}
{\sum_\zeta\|T_\ell^\zeta\|_F^2}.
\]
The ten-seed endpoint means in Table~\ref{tab:numerical-rescaling-sources}
show that the rescaling changes the source geometry sharply, even though the
initialized function and activation gates are unchanged.  For \(c=1/2\), transport and imbalance
each carry approximately one half of the source energy and cancel almost
exactly.  For \(c=2\), transport carries only \(0.21\%\), while imbalance and
the nonlinear gate source dominate.  Since \(c=1/2\) and \(c=2\) have nearly
the same normalized imbalance magnitude, \(\|\Delta_1\|_F\) alone cannot
predict the interaction.

\begin{table}[t]
\centering
\scriptsize
\begin{tabular}{c|ccc|cccc}
\(c\)
& \(\rho_1^{\mathrm{tr},\mathrm{imb}}\)
& \(I_1^{\mathrm{tr},\mathrm{imb}}\)
& \(\chi_1\)
& \(e_1^{\mathrm{tr}}\)
& \(e_1^{\mathrm{imb}}\)
& \(e_1^{\mathrm{fluc}}\)
& \(e_1^{\mathrm{gate}}\) \\
\hline
\(1/2\) & -0.997 & -0.994 & 0.037 & 0.4970 & 0.5000 & 0.0007 & 0.0026 \\
\(1\)   & -0.704 & -0.492 & 0.296 & 0.2410 & 0.4920 & 0.0580 & 0.2090 \\
\(2\)   & -0.091 & -0.005 & 0.477 & 0.0021 & 0.5010 & 0.1070 & 0.3890
\end{tabular}
\caption{Ten-seed means of the middle-layer source geometry at \(t=100\) in
the function-preserving rescaling experiment.  The intervention redistributes
the source budget and changes its directional cancellation.}
\label{tab:numerical-rescaling-sources}
\end{table}

The separation persists along the recorded trajectories, as shown in
Figure~\ref{fig:numerical-source-time}. It also affects the subsequent
localized sensitivity geometry. Define
\[
R_Q^{(c)}(T)
:=
\frac{
\widehat q_{Q,1,M}^{(1,0),\mathrm{pt},(c)}(T)
}{
\widehat q_{Q,1,M}^{(1,0),\mathrm{pt},(c)}(0)
}.
\]
For paired seeds \(s=1,\ldots,10\), the reported fixed-time effect of \(c=2\)
relative to \(c=1\) is
\[
\mathfrak E_Q^{(2:1)}(T)
:=
\exp\left\{
\frac1{10}\sum_{s=1}^{10}
\log
\frac{R_{Q,s}^{(2)}(T)}{R_{Q,s}^{(1)}(T)}
\right\},
\]
with the analogous definition at matched empirical-risk crossings.
Relative to \(c=1\), the geometric-mean paired effect of \(c=2\) on
\(R_Q^{(c)}(T)\) is \(1.219\), with paired \(95\%\) bootstrap interval
\([1.027,1.523]\), at fixed time. At the matched empirical-risk level
\(10^{-3}\), the corresponding effect is \(1.235\), with interval
\([1.079,1.441]\). Thus the difference persists after matching the empirical-risk level and cannot
be explained solely by the time required to reach the prescribed risk. The
intervention has a particularly stable effect on the source decomposition,
whereas the endpoint defect remains more seed-dependent. Source production,
source cancellation, and observed alignment must therefore be distinguished.
Moreover, \(\widehat q_{Q,1,M}^{(1,0),\mathrm{pt}}\) is the pointwise
diagnostic in \eqref{eq:numerical-localized-Q-pointwise-normalized}, not the
averaged-\(Q\) variable in the Lyapunov theory.  The paired effect therefore
shows parametrization sensitivity of the stronger inputwise geometry, but is
not direct evidence for decay or growth of the Lyapunov defect.

\begin{figure}[t]
\centering
\includegraphics[width=0.92\textwidth]
{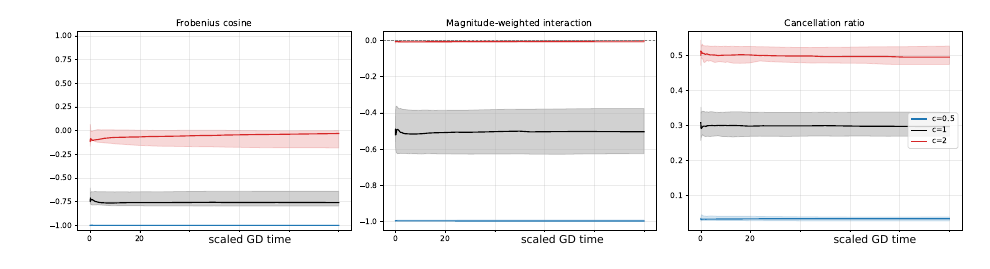}
\caption{Evolution of the middle-layer transport-imbalance interaction and
aggregate cancellation ratio in the paired rescaling experiment.  Curves
summarize ten paired seeds.  The same realized initial function can lead to
nearly exact source cancellation or to a regime in which transport is almost
absent.}
\label{fig:numerical-source-time}
\end{figure}

\subsection{Validation and interpretation}
\label{sec:numerical-validation}

For two matrices \(A\) and \(B\), define the relative residual
\[
\operatorname{Res}(A,B)
:=
\frac{\|A-B\|_F}
{\|A\|_F+\|B\|_F+\varepsilon_{\mathrm{num}}},
\]
where \(\varepsilon_{\mathrm{num}}>0\) prevents division by zero. The
computations use \(\varepsilon_{\mathrm{num}}=10^{-30}\). The four-source
residual is obtained by taking
\[
A=[Q_{\ell,N},P_\ell],
\qquad
B=
T_\ell^{\mathrm{tr}}
+
T_\ell^{\mathrm{imb}}
+
T_\ell^{\mathrm{fluc}}
+
T_\ell^{\mathrm{gate}}.
\]
The feature-transport residual is evaluated on each training,
diagnostic, or test sample using the corresponding empirical operators. Thus,
for a sample of size \(n\), it uses
\[
A=[\mathcal S_{\ell,n},\mathrm{NFM}_\ell],
\qquad
B=W_\ell^\top[Q_{\ell,n},P_\ell]W_\ell.
\]

The exact finite-sample identities provide stringent implementation checks.
Across \(12\,452\) evaluations, the normalized residual in the four-source
recursion never exceeds \(1.25\times10^{-15}\).  Across \(24\,904\)
evaluations, the feature-transport residual never exceeds
\(6.54\times10^{-13}\).  Symmetry, Jensen, spectral-gap, and localized-gap
checks are also satisfied, and all projector boundaries used in the main
tables and figures are spectrally resolved.  These residuals validate the
implementation; they are not treated as empirical evidence for alignment.

No active run reaches exact interpolation.  Accordingly, every alignment and
source statement in this section is a finite-time ensemble observation,
dependent on the sampled initializations and restricted to the architectures,
widths, datasets, and optimization horizon tested here. The experiments use a
single learning rate and the normalizations in
\eqref{eq:numerical-normalized-gate}-
\eqref{eq:numerical-normalized-feature}; neither a step-size refinement toward
continuous gradient flow nor a systematic comparison of alternative global
normalizations is included. The width and cancellation conclusions should be
read with this scope in mind.

Within the tested regimes, the three commutators remain quantitatively
distinct, and decreasing risk does not force their global or leading localized
defects to decay. Width reduces the internal and feature-side defects, but not
the gate defect, over the tested range. The persistent negative
transport--imbalance interaction, including under the function-preserving
rescaling, supports the interpretation that small observed commutators may
reflect structured cancellation rather than universal Lyapunov decay.

\section{Concluding remarks}
\label{sec:Conclusion}

This paper develops the three-level hierarchy
\[
[D_\ell,P_\ell]
\quad\longrightarrow\quad
[Q_\ell,P_\ell]
\quad\longrightarrow\quad
[\mathcal S_\ell,\mathrm{NFM}_\ell],
\]
in finite-width nonlinear networks, where the arrows indicate structural
coupling rather than logical implication. The exact recursion for the middle
commutator identifies
downstream transport, adjacent-layer imbalance, pointwise sensitivity
fluctuations, and nonlinear gate-covariance interactions as its four sources.
Buffered localization distinguishes coupling across separated spectral blocks
from mixing within a boundary cluster, while the transport identity carries
internal compatibility to feature space with explicit singular-value weights.

Risk reduction alone does not impose a universal monotone alignment law:
stabilization and decay are distinct, and a small net commutator may result
from cancellation among large sources or from singular-value filtering. The
decay results instead isolate two additional sufficient mechanisms, geometric
error bounds and intrinsic transverse damping. The analytic examples and
finite-time experiments exhibit the corresponding separation among gate,
internal, and feature-side geometry, together with exact cancellation,
transient growth, and parametrization dependence.

The most important next problem is to derive verifiable, model-dependent
conditions that produce genuine localized damping with a non-collapsing
sensitivity operator. Other natural extensions are to stochastic training,
where switching and mini-batch noise introduce additional variation; to
wide-network limits in feature-learning parametrizations; and to implicit
selection among the many internal geometries compatible with interpolation.
The present hierarchy provides concrete observables and exact identities with
which to formulate these questions, but their resolution lies beyond the
scope of this paper.
\par\medskip
\noindent
\textbf{Declaration on the use of generative AI.}
The main ideas and an initial version of the manuscript were developed by the
author before ChatGPT and Codex, developed by OpenAI, were used in its subsequent
preparation and revision.  These tools were used as an interactive aid for language
editing, organization, LaTeX formatting, checking the clarity and internal
consistency of the mathematical arguments,  and in particular in the development and debugging of the numerical experiments.  All AI-assisted outputs were
critically assessed by the author, and all mathematical statements,
proofs, references, and conclusions are independently verified by the author,
who takes full responsibility for the content of the paper.

\end{document}